\documentclass[nohyperref]{article}

\usepackage{microtype}
\usepackage{graphicx}
\usepackage{subfigure}
\usepackage{booktabs} % for professional tables

\usepackage{hyperref}

\usepackage[accepted]{icml2022}
\usepackage{amsmath}
\usepackage{amssymb}
\usepackage{mathtools}
\usepackage{amsthm}

\usepackage[capitalize,noabbrev]{cleveref}

\theoremstyle{plain}
\newtheorem{theorem}{Theorem}[section]
\newtheorem{proposition}[theorem]{Proposition}
\newtheorem{lemma}[theorem]{Lemma}

\theoremstyle{definition}

\theoremstyle{remark}

\usepackage[textsize=tiny]{todonotes}

\usepackage{multirow}
\usepackage{multicol}
\usepackage{array}

\usepackage{import}
\usepackage{xifthen}
\usepackage{pdfpages}
\usepackage{transparent}
\usepackage{bm}
\usepackage{color}
\newcommand{\ca}[1]{\bm{\mathcal{#1}}}

\icmltitlerunning{A Unified Weight Initialization Paradigm for Tensorial Convolutional Neural Networks}

\begin{document}

\twocolumn[
\icmltitle{A Unified Weight Initialization Paradigm \\
for Tensorial Convolutional Neural Networks}



\icmlsetsymbol{equal}{*}

\begin{icmlauthorlist}
\icmlauthor{Yu Pan}{1}
\icmlauthor{Zeyong Su}{3}
\icmlauthor{Ao Liu}{4}
\icmlauthor{Jingquan Wang}{1}
\icmlauthor{Nannan Li}{5,6}
\icmlauthor{Zenglin Xu}{1,2}
\end{icmlauthorlist}

\icmlaffiliation{1}{Harbin Institute of Technology Shenzhen, Shenzhen, China}
\icmlaffiliation{2}{Pengcheng Laboratory, Shenzhen, China}
\icmlaffiliation{3}{University of Electronic Science and Technology of China, Chengdu, China}
\icmlaffiliation{4}{Tokyo Institute of Technology, Tokyo, Japan}
\icmlaffiliation{5}{State Key Laboratory for Management and Control of Complex Systems, Institute of Automation, Chinese Academy of Sciences, Beijing, China}
\icmlaffiliation{6}{School of Artificial Intelligence, University of Chinese Academy of Sciences, Beijing, China}

\icmlcorrespondingauthor{Yu Pan}{iperryuu@gmail.com}
\icmlcorrespondingauthor{Zenglin Xu}{zenglin@gmail.com}

\icmlkeywords{Machine Learning, ICML}

\vskip 0.3in
]



\printAffiliationsAndNotice{}  % leave blank if no need to mention equal contribution

\begin{abstract}

Tensorial Convolutional Neural Networks (TCNNs) have attracted much research attention for their power in reducing model parameters or enhancing the generalization ability. However, exploration of TCNNs is hindered even from weight initialization methods. To be specific, general initialization methods, such as Xavier or Kaiming initialization, usually fail to generate appropriate weights for TCNNs. Meanwhile, although there are ad-hoc approaches for specific architectures (e.g., Tensor Ring Nets), they are not applicable to  TCNNs with other tensor decomposition methods (e.g., CP or Tucker decomposition). To address this problem, we propose a universal weight initialization paradigm, which generalizes Xavier and Kaiming methods and can be widely applicable to arbitrary TCNNs. Specifically, we first present the Reproducing Transformation to convert the backward process in TCNNs to an equivalent convolution process. Then, based on the convolution operators in the forward and backward processes, we build a unified paradigm to control the variance of features and gradients in TCNNs. Thus, we can derive fan-in and fan-out initialization for various TCNNs. We demonstrate that our paradigm can stabilize the training of TCNNs, leading to faster convergence and better results.

\end{abstract}

\section{Introduction}
\label{sec:intro}

% \panyu{Advantage} Tensorial Convolutional Neural Networks (TCNNs) are important variants of Convolutional Neural Networks (CNNs). TCNNs usually adopt various tensor decomposition techniques to factorize large convolutional kernels into lower-rank tensor nodes, significantly reducing the number of parameters. Various tensor decomposition methods have been successfully utilized for TCNNs, including  Tensor Train (TT)~\citep{DBLP:conf/nips/NovikovPOV15, DBLP:journals/corr/abs-1904-06194, DBLP:journals/corr/GaripovPNV16}, Tensor Ring (TR)~\citep{DBLP:conf/cvpr/WangSEWA18}, CANDECAMP/PARAFAC (CP)~\citep{DBLP:conf/cvpr/KossaifiTBPHP20, DBLP:journals/corr/LebedevGROL14}, Tucker~\citep{DBLP:journals/corr/KimPYCYS15, DBLP:conf/cvpr/KossaifiTBPHP20,  DBLP:journals/corr/abs-1909-05675}, etc.
Tensorial Convolutional Neural Networks (TCNNs) are important variants of Convolutional Neural Networks (CNNs). TCNNs usually adopt various tensor decomposition techniques to factorize large convolutional kernels into lower-rank tensor nodes, aiming to reduce the number of parameters. For example, Tensor Ring (TR) is utilized to decompose CNNs~\citep{DBLP:conf/cvpr/WangSEWA18}, leading to a high compression rate while maintaining comparably good performance. Tensor Train (TT) was used to improve performance of CNNs for image classification with parameter reduction~\cite{DBLP:conf/cvpr/YinSL021}. The CP-Higher-Order convolution (CP-HOConv) was proposed to factorize higher- order convolutional neural networks and has achieved the state-of-the-art results in spatio-temporal facial emotion analysis~\cite{DBLP:conf/cvpr/KossaifiTBPHP20}.

% \panyu{many formats} TCNNs can be regarded as a family of convolutional neural networks if the corresponding architectures can be represented with hypergraphs. A hypergraph is a tensor diagram with dummy tensors (Figure~\ref{fig:hyper-dummy}) and hyperedges (Figure~\ref{fig:hyper-edge}) denoting TCNNs. For example, through hypergraph representation, the bottleneck convolution of CNNs is equivalent to Tucker-2 convolution. Similarly, traditional CNN variants (e.g., low-rank convolution~\citep{DBLP:conf/cvpr/RigamontiSLF13, DBLP:conf/cvpr/IdelbayevC20}, factoring convolution~\citep{DBLP:conf/cvpr/SzegedyVISW16}, and even vanilla convolution) can also be considered as a special case of TCNNs since they can be represented as a  hypergraph. Therefore, TCNNs can incorporate a large family of convolutional neural networks with flexible implementations. 

In addition to the advantages in reducing model parameters, TCNNs are promising to be explored as a more general family of CNNs if the corresponding structures can be represented with hypergraphs. A hypergraph is a tensor diagram with a dummy tensor (as illustrated in Figure~\ref{fig:hyper-dummy}) and a hyperedge (as illustrated in Figure~\ref{fig:hyper-edge}). Equipped with the hypergraph representation, TCNNs can include not only factorized CNNs based on tensor decomposition methods (e.g., Tensor Ring (TR) decomposition~\citep{DBLP:conf/cvpr/WangSEWA18}, Tensor Train (TT) 
decomposition~\citep{DBLP:conf/nips/NovikovPOV15, DBLP:journals/corr/abs-1904-06194, DBLP:journals/corr/GaripovPNV16}, CANDECAMP/PARAFAC (CP) decomposition~\citep{DBLP:journals/corr/LebedevGROL14,DBLP:journals/ijon/PanWX22}, Tucker decomposition~\citep{DBLP:journals/corr/KimPYCYS15,DBLP:journals/corr/abs-1909-05675}, Block-Term Tucker decomposition~\citep{DBLP:conf/cvpr/YeWLCZCX18,DBLP:journals/nn/YeLCYZX20}),
% Interestingly, through hypergraph representation, the bottleneck convolution of CNNs is equivalent to Tucker-2 convolution. 
but also traditional CNN variants (e.g., low-rank convolution~\citep{DBLP:conf/cvpr/RigamontiSLF13, DBLP:conf/cvpr/IdelbayevC20}, factoring convolution~\citep{DBLP:conf/cvpr/SzegedyVISW16}, and even the vanilla convolution),
% as a special case of TCNNs, 
since each of them can be represented as a hypergraph. 
% Therefore, TCNNs can incorporate a large family of convolutional neural networks with flexible implementations. 

% \panyu{Big family, hard to explore, since lacking weight initialization} Despite the flexibility of TCNNs, an important issue underlying such a big family of architectures is the unstable training due to inappropriate weight initialization methods~\citep{DBLP:conf/cvpr/WangSEWA18, DBLP:journals/corr/abs-1909-05675}. A common and direct initialization method generates weights by sampling from a probability distribution~\citep{DBLP:conf/aaai/PanXWYWBX19,li2021heuristic}. Unfortunately, this initialization is sensitive to the distribution variance selection. Additionally, the distribution parameters can only be tuned manually, which causes an inefficient operation for this initialization in practice.
Despite these merits,  TCNNs suffer from unstable training due to inappropriate weight initialization~\citep{DBLP:conf/cvpr/WangSEWA18, DBLP:journals/corr/abs-1909-05675}. A common and direct initialization method generates weights by sampling from a probability distribution~\citep{DBLP:conf/aaai/PanXWYWBX19,li2021heuristic}. Unfortunately, this initialization method is sensitive to the choice of distribution variance; the distribution parameters are usually tuned manually, which is inefficient in practice.
%Additionally, the distribution parameters can only be tuned manually, which causes inefficiency  in practice.
Another straightforward method is to borrow some adaptive weight initialization methods widely used in CNNs, such as the Xavier initialization~\citep{DBLP:journals/jmlr/GlorotB10} and the Kaiming initialization~\citep{DBLP:conf/iccv/HeZRS15}, however,
they usually fail to initialize weights at a correct scale for TCNNs~\citep{DBLP:conf/iconip/WangSL0ZX20,DBLP:conf/iclr/ChangFL20}.
% Some works generalize Xavier and Kaiming initialization to specific TCNNs.
% by regulating the weight initial variance to stabilize data-flow of the network.
% ~\citep{DBLP:conf/cvpr/WangSEWA18, DBLP:conf/iclr/ChangFL20, DBLP:journals/corr/abs-1709-02956}. 
% However, these methods are adjusted to some specific TCNNs (e.g., TR) and cannot be generalized to other TCNN variants (e.g., CP and Tucker).
In addition, ad-hoc initialization methods, such as  modified Xavier methods proposed in \cite{DBLP:conf/cvpr/WangSEWA18,DBLP:conf/iclr/ChangFL20} or the method of decomposing corresponding CNN weights~\cite{DBLP:journals/corr/abs-1909-05675}, are either designed for specific TCNNs, or dependent on special tensor decomposition methods. 
%%In addition, there are some initialization methods designed for specific TCNN variants. 
%%Targeting at TR-based CNNs, \citet{DBLP:conf/cvpr/WangSEWA18} propose a initialization similar to Xavier to control variance change. However, this method is hard to adapt to other TCNNs (e.g., CP and Tucker) for different tensor formats.  
% weights in a smaller scale than that of Xavier, which may make data-flow vanish and lead to poor results.
%%\citet{DBLP:journals/corr/abs-1909-05675} propose to initialize a TCNN by decomposing the corresponding CNN weights.
% a method that first decomposes the weight of a well-trained CNN and then uses the decomposition results to initialize the weights of TCNN.
%%However, this method requires special tensor decomposition algorithms for different tensor formats and it is not guaranteed that every tensor format has such an algorithm.
% Moreover, it is especially hard to train a higher-order convolutional kernel~\citep{DBLP:conf/bmvc/JaderbergVZ14, DBLP:journals/corr/LebedevGROL14, DBLP:journals/corr/TaiXWE15}.
%%Moreover, HyperNetworks can also be regarded as a special case of TCNN under hypergraph representation as shown in Figure~\ref{fig:hypernetfb} of the Appendix. \citet{DBLP:conf/iclr/ChangFL20} utilize the idea of Xavier to limit weights of HyperNetworks to an appropriate scale, which makes the training process stable. Nevertheless, this initialization still cannot generalize to other TCNNs (e.g., TR).

Therefore, it is necessary to design a universal initialization scheme for various TCNN variants. 
To this end, we propose a unified paradigm that studies TCNNs from their topology. In detail, by extracting a backbone graph (BG) from a convolution hypergraph, we can encode an arbitrary TCNN into an adjacency matrix with the backbone graph and then calculate a suitable initial variance through the adjacency matrix, which can simply initialize any TCNN. 
%In this paper, the method of obtaining a suitable variance for dummy tensor based TCNNs is called Unified Paradigm.
% Worth to mention, our unified method can be regarded as generalization of \citet{DBLP:conf/iclr/ChangFL20} through depicting the HyperNetwork with hypergraph, which indicates that besides TCNNs, the unified initialization is potentially applicable for more CNN variants.

\begin{figure}[t]
	\centering
	\subfigure[]{
		\includegraphics[height=0.054\textheight]{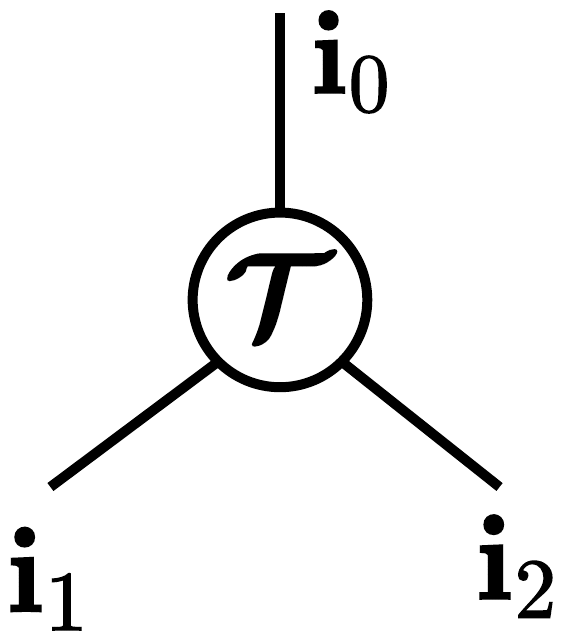}
		\label{fig:classical-tensor}
	}~
	\subfigure[]{
		\includegraphics[height=0.05\textheight]{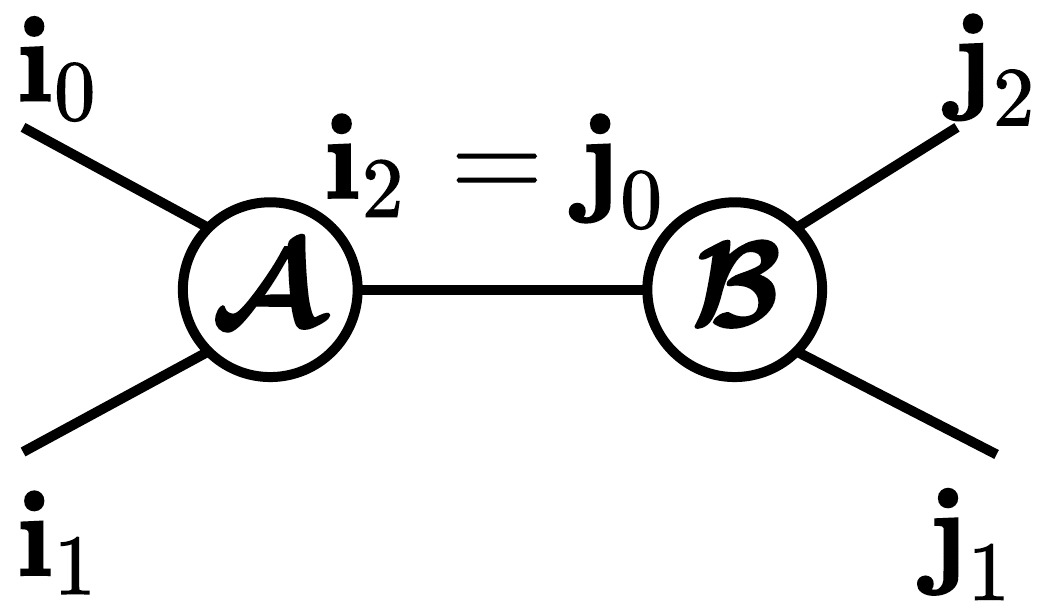}
		\label{fig:classical-contraction}
	}~
	\subfigure[]{
		\includegraphics[height=0.054\textheight]{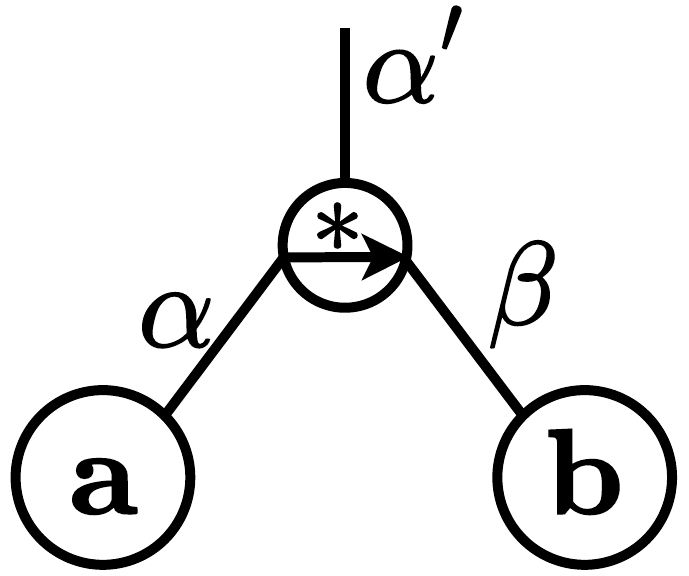}
		\label{fig:hyper-dummy}
	}~
	\subfigure[]{
		\includegraphics[height=0.05\textheight]{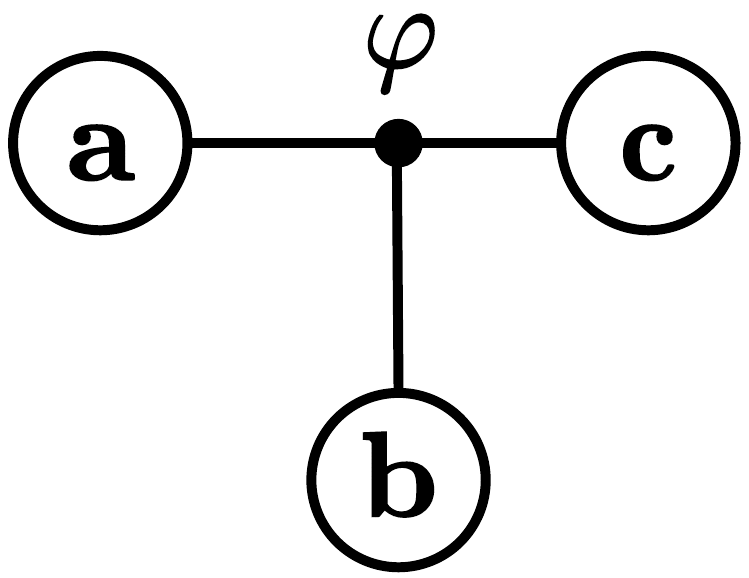}
		\label{fig:hyper-edge}
	}
	\caption{Tensor graphical instances. \subref{fig:classical-tensor} A tensor $\ca{T} \in \mathbb{R}^{\mathbf{i}_0\times \mathbf{i}_1  \times \mathbf{i}_2}$; \subref{fig:classical-contraction} Tensor Contraction; \subref{fig:hyper-dummy} Dummy Tensor; \subref{fig:hyper-edge} Hyperedge.}
	\label{fig:hypers}
\end{figure}

\begin{figure}[t]
	\centering
	\subfigure[$\mathbf{G_f}$]{
		\includegraphics[height=0.092\textheight]{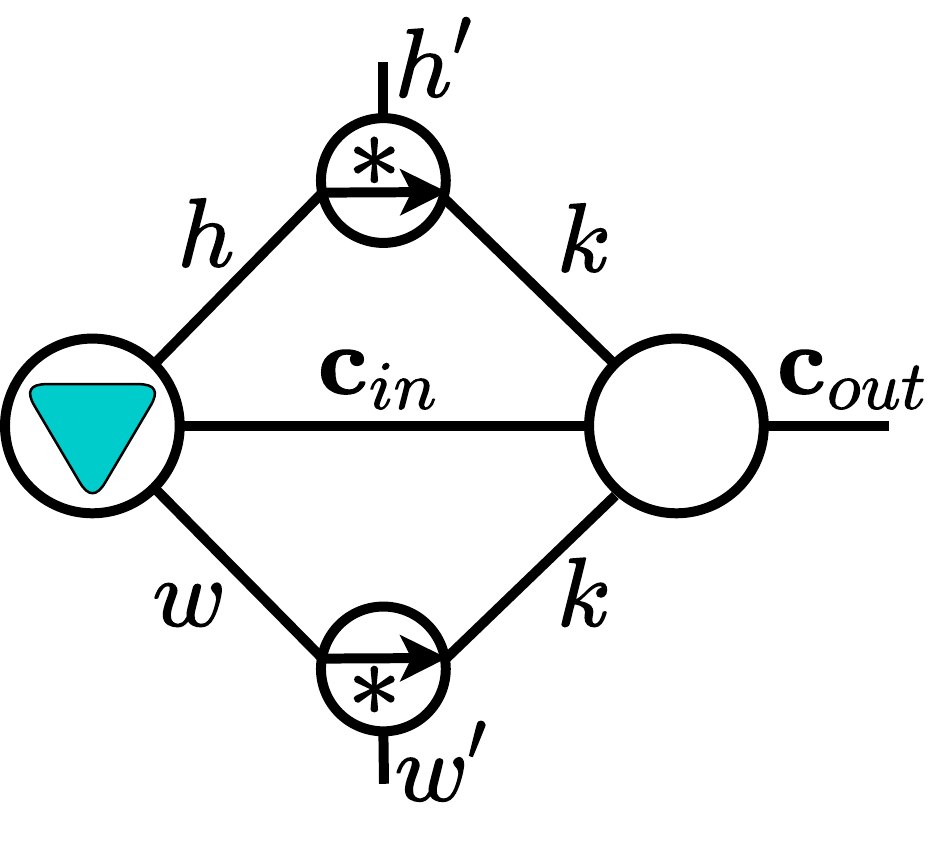}
		\label{fig:cnn-hyper-for}
	}
	\subfigure[$\mathbf{G_{b}}$]{
		\includegraphics[height=0.092\textheight]{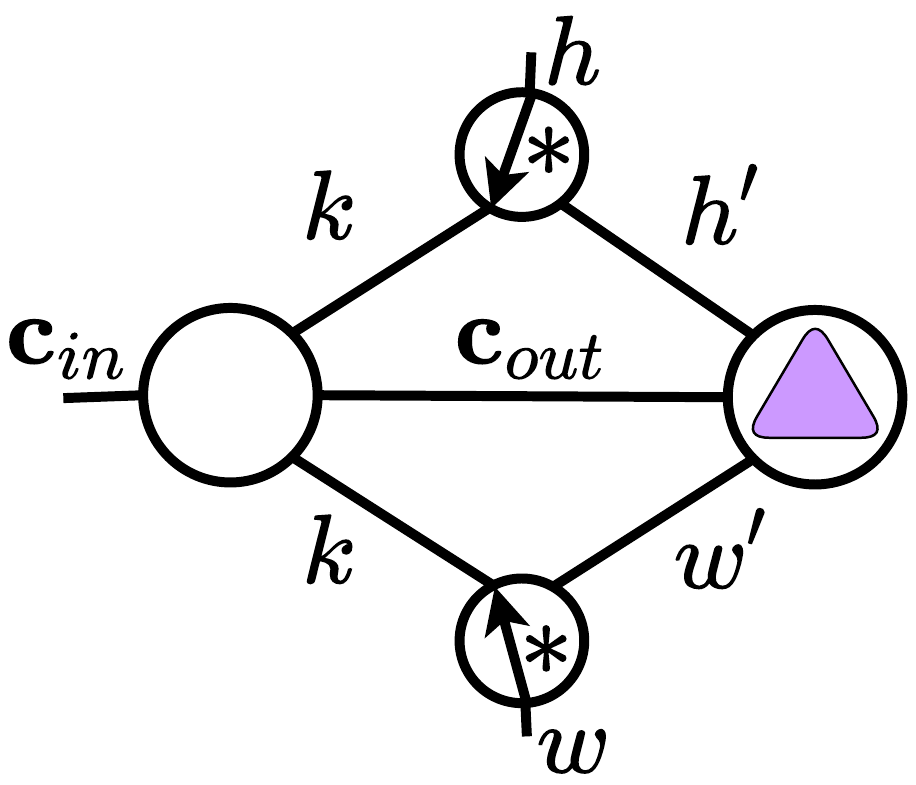}
		\label{fig:cnn-hyper-back}
	}
	\subfigure[$\mathbf{G_{bt}}$ (ours)]{
		\includegraphics[height=0.092\textheight]{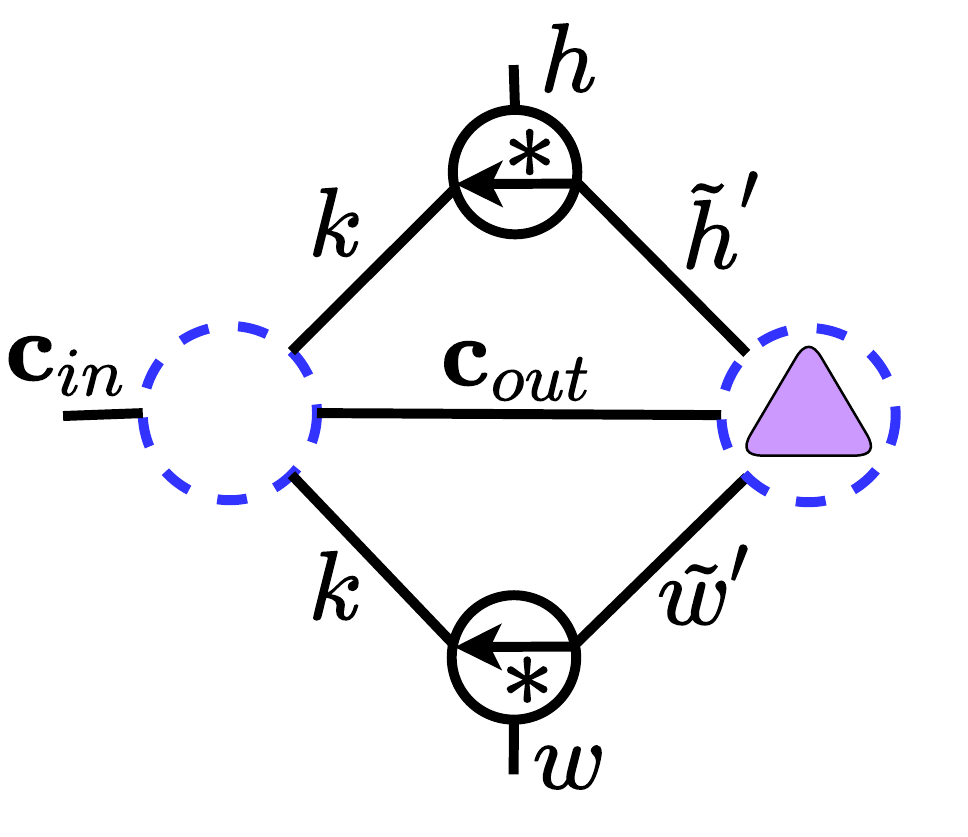}
		\label{fig:cnn-super-back}
	}
	\caption{Illustration for vanilla CNN. \subref{fig:cnn-hyper-for} $\mathbf{G_f}$: A hypergraph forward process formulated as $\ca{Y}=\ca{X}\circledast\ca{C}+\mathbf{b}$, where $\ca{C} \in \mathbb{R}^{\mathbf{c}_{in}\times \mathbf{c}_{out} \times k \times k}$ denotes a convolutional kernel, $\ca{X}\in \mathbb{R}^{\mathbf{c}_{in}\times h\times w}$ denotes the input feature (Cyan Inverted Triangle), $\ca{Y}\in \mathbb{R}^{\mathbf{c}_{out}\times h'\times w'}$ denotes the output feature, $\mathbf{b}\in \mathbb{R}^{\mathbf{c}_{out}}$ represents the bias, and $\circledast$ denotes the convolutional operator. $k$ represents kernel window size, $\mathbf{c}_{in}$ is the input channel, $h$ and $w$ denote height and width of $\ca{X}$, $\mathbf{c}_{out}$ is the output channel, $h'$ and $w'$ denote height and width of $\ca{Y}$; \subref{fig:cnn-hyper-back} $\mathbf{G_b}$: A hypergraph backward process derived directly from $\mathbf{G_f}$; \subref{fig:cnn-super-back} $\mathbf{G_{bt}}$: A hypergraph backward process  equivalently transformed from $\mathbf{G_{b}}$ with Reproducing Transformation. $\mathbf{G_{bt}}$ is completely the same as $\mathbf{G_{b}}$. In \subref{fig:cnn-hyper-back} and \subref{fig:cnn-super-back}, Purple Triangle means input gradient $\Delta \ca{Y} \in \mathbb{R}^{\mathbf{c}_{out}\times h'\times w'}$. $\tilde{h}'$ and $\tilde{w}'$ denote height and width of transformed $\Delta \ca{Y}$. In dummy tensor represented graphs, convolutional kernel vertex should connect arrow head of the dummy tensor and data-flow should connect the arrow tail. Thus, $\mathbf{G_{f}}$ and $\mathbf{G_{bt}}$ are convolutions, while $\mathbf{G_{b}}$ is not. 
	}
	\label{fig:hypergraphs}
\end{figure}

The unified paradigm can be applicable in controlling variance of two data-flow types, i.e., features in the forward process (fan-in mode) and gradients in the backward process (fan-out mode). For the fan-in mode, since the forward hypergraph (namely $\mathbf{G_f}$ in Figure~\ref{fig:cnn-hyper-for}) is a dummy based convolution, the unified paradigm can inherently be applied.
However, in the fan-out mode, the backward hypergraph (namely $\mathbf{G_b}$ in Figure~\ref{fig:cnn-hyper-back}) cannot represent a convolution process due to the conflict with the dummy tensor definition in Section~\ref{sec:hypergraph}. %, meaning that it violates the paradigm's condition, i.e., being applied to a convolutional process. 
To solve this problem, we originally propose the Reproducing Transformation to reproduce $\mathbf{G_{b}}$ as a convolution hypergraph $\mathbf{G_{bt}}$ shown in Figure~\ref{fig:cnn-super-back}. Through the Reproducing Transformation, the unified paradigm can be applicable to the backward process. The 
overall working flow is illustrated in Figure~\ref{fig:total-frame}. In brief, our principal initialization can unify a variety of tensor formats, and meanwhile, fit both forward and backward propagation.

Through extensive experiments on various image classification benchmarks, we demonstrate that our method can produce appropriate initial weights for complicated TCNNs compared with classical initialization methods. Last but not least, we show that our paradigm is intrinsically a generalization of Xavier and relevant methods~\citep{DBLP:conf/cvpr/WangSEWA18, DBLP:conf/iccv/HeZRS15, DBLP:conf/iclr/ChangFL20}, while working more effectively for arbitrary TCNNs.

\begin{figure}[t]
	\centering
	\includegraphics[width=0.3\textwidth]{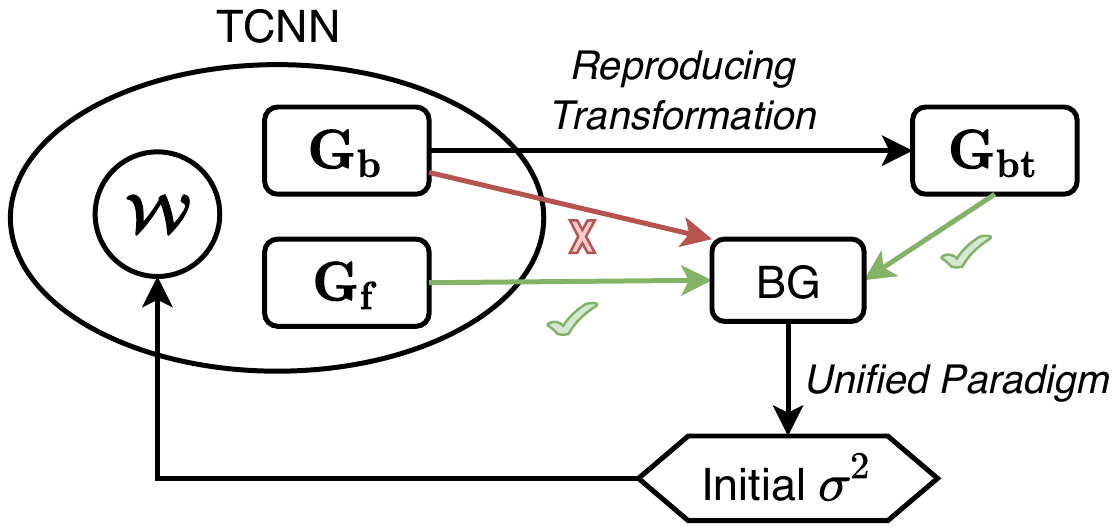}
	\caption{The overall workflow of the proposed unified initialization. A TCNN contains a forward hypergraph $\mathbf{G_f}$ and a backward hypergraph $\mathbf{G_b}$, besides the network weights $\ca{W}$. The objective is to achieve an acceptable variance $\sigma^2$ for $\ca{W}$ in order to keep the magnitude of data-flow stable across layers. To reach the goal, we derive a unified paradigm to calculate the desired $\sigma^2$.
	Note that the paradigm is applicable only to a backbone graph (BG) derived from a convolutional hypergraph. As $\mathbf{G_b}$ cannot be converted into the BG representation, we propose a reproducing transformation to transfer $\mathbf{G_b}$ in a convolutional representation $\mathbf{G_{bt}}$. With $\mathbf{G_f}$ and $\mathbf{G_{bt}}$, we can initialize TCNNs by regulating data-flow variance.
	}
	%The paradigm is applied only to a backbone graphs (BG) derived from a convolutional hypergraph. As a result, $\mathbf{G_f}$ can be converted into the BG representation directly, however, $\mathbf{G_b}$ cannot since it is unable to denote a convolution. Towards this problem, we propose a reproducing transformation to transfer $\mathbf{G_b}$ in a convolutional representation $\mathbf{G_{bt}}$. With $\mathbf{G_f}$ and $\mathbf{G_{bt}}$, it is feasible to successfully initialize arbitrary TCNN by regulating data-flow variance.}
	\label{fig:total-frame}
\end{figure}

% \begin{figure*}[t]
% 	\centering
% 	\includegraphics[width=0.95\textwidth]{figures/cnn-replace.pdf}
% 	\caption{Reproducing Transformation, an equivalent transformation from $\mathbf{G_{b}}$ to $\mathbf{G_{bt}}$. Notations follow Figure~\ref{fig:hypergraphs}. According to Section~\ref{sec:hypergraph}, the purple gradient vertex should connect to the arrow tails of dummy tensors, therefore $\mathbf{G_{b}}$ cannot denote a convolution. To overcome the problem, we design Reproducing Transformation to bond the gradient vertex to arrow tails for the convolution representation. Taking vanilla convolution as an example, Reproducing Transformation utilizes an equivalent replacement with reversal matrix $\mathbf{R}$ and transformation matrix $\mathbf{T}$ (details in Section~\ref{sec:supergraph}) to exchange the arrow tail entry and the extra entry. Then $\mathbf{G_{bt}}$ can be  derived through contracting $\mathbf{R}$ and $\mathbf{T}$ with the weight vertex and the gradient vertex, respectively.}
% 	\label{fig:replace}
% \end{figure*}

\begin{figure*}[t]
	\centering
	\includegraphics[width=0.95\textwidth]{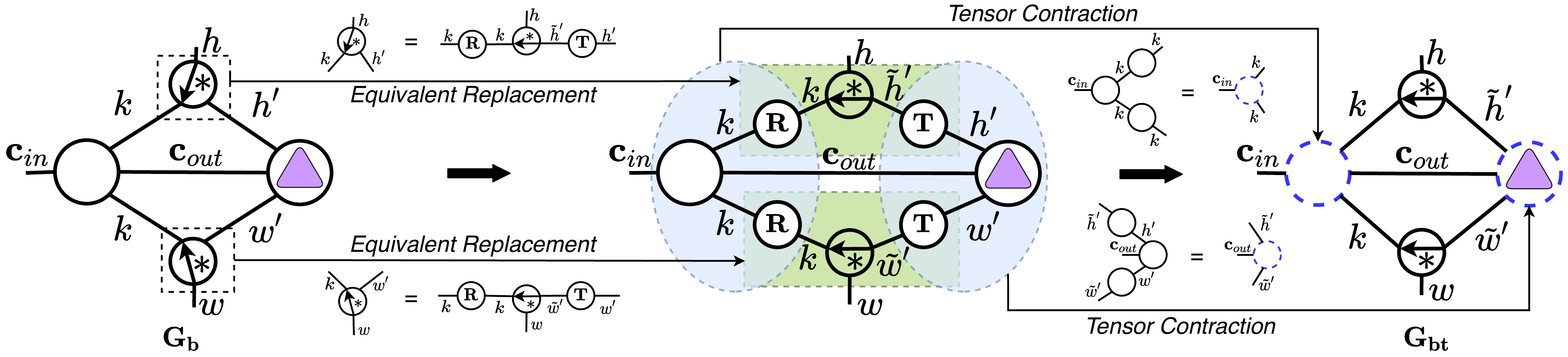}
	\caption{Reproducing Transformation, an equivalent transformation from $\mathbf{G_{b}}$ to $\mathbf{G_{bt}}$. Notations follow Figure~\ref{fig:hypergraphs}. According to Section~\ref{sec:hypergraph}, the purple gradient vertex should connect to the arrow tails of dummy tensors, therefore $\mathbf{G_{b}}$ cannot denote a convolution. To overcome the problem, we design Reproducing Transformation to bond the gradient vertex to arrow tails for the convolution representation. Taking vanilla convolution as an example, Reproducing Transformation utilizes an equivalent replacement with reversal matrix $\mathbf{R}$ and transformation matrix $\mathbf{T}$ (details in Section~\ref{sec:supergraph}) to exchange the arrow tail entry and the extra entry. Then $\mathbf{G_{bt}}$ can be  derived through contracting $\mathbf{R}$ and $\mathbf{T}$ with the weight vertex and the gradient vertex, respectively.}
	\label{fig:replace}
\end{figure*}

\section{Preliminaries}
\label{sec:pre}
In this section, we introduce the necessary preliminaries about tensors, and Xavier/Kaiming initialization.

\subsection{Tensor Diagram}

A tensor diagram mainly consists of two components, a tensor vertex and tensor contraction.

\textbf{Tensor Vertex.}~~A tensor is denoted as a vertex whose order is given by the number of edges connected to it. The integer assigned to each edge denotes the dimension of the corresponding mode. For example,
Figure~\ref{fig:classical-tensor} shows a 3rd-order tensor $\ca{T} \in \mathbb{R}^{\mathbf{i}_0\times \mathbf{i}_1  \times \mathbf{i}_2}$.

\textbf{Tensor Contraction.}~~The inner-product of two tensors on matching modes denotes tensor contraction. As illustrated in Figure~\ref{fig:classical-contraction}, a tensor $\ca{A}\in\mathbb{R}^{\mathbf{i}_0\times \mathbf{i}_1  \times \mathbf{i}_2} $ and a tensor $\ca{B}\in\mathbb{R}^{\mathbf{j}_0\times \mathbf{j}_1 \times \mathbf{j}_2}$, can contract in the corresponding position, forming a new tensor of $\mathbb{R}^{\mathbf{i}_0\times \mathbf{i}_1 \times \mathbf{j}_2 \times \mathbf{j}_3}$, when they have equal dimensions: $\mathbf{i}_2 = \mathbf{j}_0 \triangleq
e_0$. The contraction operation can be formulated as
\begin{equation}
(\ca{A}\times_{2}^{0}\ca{B})_{i_0, i_1, j_2, j_3} 
= \sum_{m=0}^{e_0 - 1} \ca{A}_{i_0, i_1, m}\ca{B}_{m, j_2, j_3}.
\end{equation}

\subsection{Hypergraph}
\label{sec:hypergraph}

To enhance expressive ability of the tensor diagram in deep models, \citet{DBLP:conf/nips/HayashiYSM19} proposed hypergraph to represent forward process of TCNNs through the dummy tensor and the hyperedge.

\textbf{Dummy Tensor.}~~A vertex with an arrow symbol denotes a dummy tensor which is able to represent a convolutional operation. As depicted in Figure~\ref{fig:hyper-dummy}, for a dummy tensor $\ca{P} \in \left\{ 0, 1 \right\}^{\alpha \times {\alpha '} \times \beta}$, $\alpha$ is the arrow tail entry, $\beta$ is the arrow head entry and $\alpha'$ is the extra entry. Relation among the three entries is formulated as $\ca{P}_{j,j^{'},k}=1$ if $j = sj' + k - p$ and $0$ otherwise. Here, $s$ represents the stride size; $p$ denotes the padding size. A vector convolution in Figure~\ref{fig:hyper-dummy} can be formulated as $\mathbf{c} = \mathbf{a} \times_0^0 \ca{P} \times_1^0 \mathbf{b} = \mathbf{a}\circledast \mathbf{b} \in \mathbb{R}^{\alpha'}$, in which $\circledast$ is the convolutional operator. This formulation represents a convolution in which $\mathbf{a}$ means a data-flow and $\mathbf{b}$ denotes a convolutional kernel. For a dummy tensor, convolutional kernel vertex should connect arrow head of the dummy tensor and data-flow should connect the arrow tail.

\textbf{Hyperedge.}~~A hyperedge $\varphi$ can connect to more than two tensor vertices. As shown in Figure~\ref{fig:hyper-edge}, an output of a special case, connecting three vectors through a hyperedge, can be calculated as $y = \sum_{k=0}^{\varphi-1} \mathbf{a}_{k}\mathbf{b}_{k} \mathbf{c}_{k}$. There is usually at most one hyperedge in a hypergraph layer, connecting to all weight vertices~\citep{DBLP:conf/nips/HayashiYSM19}.
% require at most one hyperedge in a hypergraph layer, and the hyperedge is supposed to connect all weight tensor nodes. 
% Obviously, the hyperedge $\varphi$ is simply equivalent to a tensor whose diagonal elements are 1s, which means adding operation over several sub-structures (e.g., TCNNs).
A hyperedge $\varphi$ of a hypergraph represents summation over sub-structures, the parts without the hyperedge.
For such an adding composite structure, we can derive the whole architecture initialization by processing each sub-structure.

\begin{figure*}[t]
	\centering
	\includegraphics[width=0.98\textwidth]{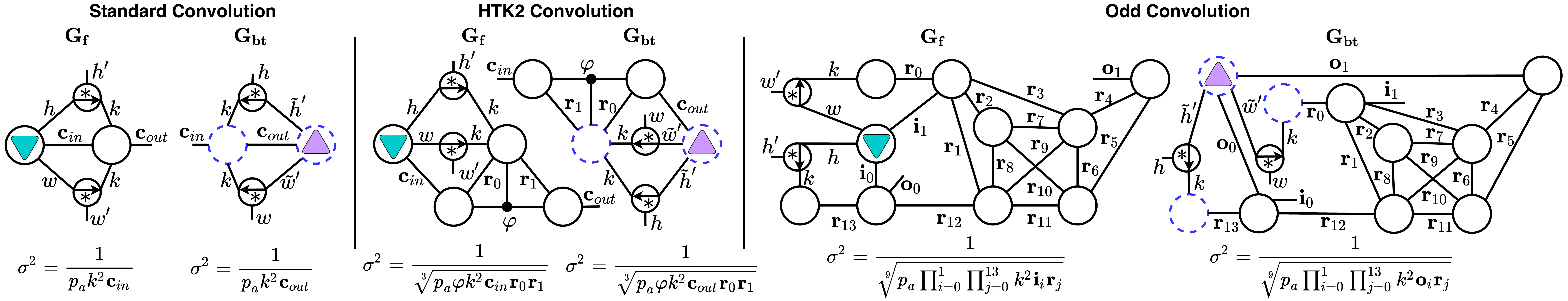}
	\caption{Reproducing Transformation Cases.
	$\sigma^2$ denotes initial variance of each wight vertex. 
(i)~Standard Convolution; It is the most common convolution in CNNs. We observe that the graphical initialization will degenerate to Xavier/Kaiming initialization on the standard convolution, as they have the same weight variance formulation;
(ii)~Hyper Tucker-2 (HTK2) Convolution; Tucker-2 (TK2) is a classical tensor decomposition, known as the bottleneck structure in ResNet~\citep{DBLP:conf/cvpr/HeZRS16}. We apply hyperedge to its weight vertices to form the HTK2;
(iii)~Odd Convolution; We introduce a particularly complicated tensor format (named Odd Tensor here) originally proposed by \citet{DBLP:conf/icml/LiS20}. Odd Tensor contains 9 vertices and 14 edges. The connection among these vertices is irregular, making weight initialization a complex problem. By connecting all weight vertices with a hyperedge $\varphi$, it is flexible to construct HOdd (Graph-in: $\frac{1}{\sqrt[9]{\bm{p}_{\bm{a}}\varphi \prod^1_{i=0}\prod^{13}_{j=0} k^2\mathbf{i}_i\mathbf{r}_j}}$; Graph-out: $\frac{1}{ \sqrt[9]{\bm{p}_{\bm{a}}\varphi\prod^1_{i=0}\prod^{13}_{j=0}k^2\mathbf{o}_i\mathbf{r}_j}}$).
By successfully training Hyper Odd (HOdd) based networks, we can better demonstrate the potential adaptability of our method to diverse TCNNs.}
	\label{fig:super-figures}
\end{figure*}

\subsection{Xavier and Kaiming Initialization}
\label{sec:xavier}
Xavier initialization~\citep{DBLP:journals/jmlr/GlorotB10} and Kaiming initialization~\citep{DBLP:conf/iccv/HeZRS15} are widely used in CNNs. They aim to control the variance of features and gradients for stable training. We will introduce them through a vanilla CNN (Figure~\ref{fig:cnn-hyper-for}), formulated as $\ca{Y}=\ca{X}\circledast\ca{C}+\mathbf{b}$, where $\ca{C} \in \mathbb{R}^{\mathbf{c}_{in}\times \mathbf{c}_{out} \times k \times k}$ denotes a convolutional kernel, $\ca{X}\in \mathbb{R}^{\mathbf{c}_{in}\times h\times w}$ denotes the input, $\ca{Y}\in \mathbb{R}^{\mathbf{c}_{out}\times h'\times w'}$ denotes the output, $\mathbf{b}\in \mathbb{R}^{\mathbf{c}_{out}}$ represents the bias, and $\circledast$ denotes the convolutional operator. $k$ represents kernel window size, $\mathbf{c}_{in}$ is the input channel, $h$ and $w$ denote height and width of $\ca{X}$, $\mathbf{c}_{out}$ is the output channel, and $h'$ and $w'$ denote height and width of $\ca{Y}$.

Xavier initialization makes the following assumptions: (1) Elements of $\ca{C}$, $\ca{X}$ and $\mathbf{b}$ all satisfy the i.i.d. condition; (2) $\mathbb{E}(\ca{C})=0$; (3)~$\mathbb{E}(\ca{X})=0$; and (4) $\mathbf{b}=\mathbf{0}$. There are two modes of Xavier initialization: (1) maintaining the variance of feature $\ca{X}$ which is referred to as the fan-in mode: $\sigma^2(\ca{C}) = \frac{1}{k^2\mathbf{c}_{in}}$; (2) maintaining the variance of gradients as the fan-out mode: $\sigma^2(\ca{C}) = \frac{1}{k^2\mathbf{c}_{out}}$.
In practice, the harmonic form is preferred: $\sigma^2(\ca{C}) = \frac{2}{k^2(\mathbf{c}_{in} + \mathbf{c}_{out})}$.

Kaiming initialization extends the Xavier initialization to incorporate ReLU activation function. In accordance with Assumption (3) of Xavier initialization, Kaiming initialization requires the distribution of $\ca{C}$ to be symmetric. Similarly, Kaiming initialization also contains two modes: (1) the fan-in mode: $\sigma^2(\ca{C}) = \frac{2}{k^2\mathbf{c}_{in}}$; (2) the fan-out mode: $\sigma^2(\ca{C}) = \frac{2}{k^2\mathbf{c}_{out}}$.

% \section{Review of Previous Work}
% \label{sec:relate}

\section{Unified Initialization}
\label{sec:method}

% We propose the unified initialization to adapt various tensor formats based on variance analysis. Specifically, we first design Reproducing Transformation to convert original backward process $\mathbf{G_b}$ to a convolution graph $\mathbf{G_{bt}}$. Then we build a unified paradigm on a convolution, to fit arbitrary tensor formats. At last, we give a simple initialization exemplar by deriving Graph-in mode and Graph-out mode through applying the unified paradigm on forward graph $\mathbf{G_{f}}$ and the transformed backward graph $\mathbf{G_{bt}}$, respectively.
In this section, we introduce our proposed unified initialization paradigm designed for various TCNNs. 
We first introduce our Reproducing Transformation, then we demonstrate the derivation of our unified paradigm, and finally we provide a simple exemplar initialization method that can be directly obtained based on the paradigm.

\subsection{Reproducing Transformation}
\label{sec:supergraph}

% Controlling data-flow can be implemented from both forward feature (i.e., fan-in mode) and backward gradient (i.e., fan-out mode). However, as shown in Figure~\ref{fig:replace}, $\mathbf{G_b}$ directly derived from $\mathbf{G_f}$ cannot represent convolution operations, which makes it impossible to derive the fan-out mode initialization through hypergraph. Motivated by that, we build Reproducing Transformation
% to represent backward propagation as a convolution (named $\mathbf{G_{bt}}$)
% by performing an equivalent replacement operation on $\mathbf{G_b}$. With $\mathbf{G_f}$ and $\mathbf{G_{bt}}$, we can unify forward and backward propagation in the identical graph representation. Prior to the transformation, we first formulate the forward process.

% Controlling variance of data-flow can be implemented in fan-in mode and fan-out mode. However, as shown in Figure~\ref{fig:replace}, $\mathbf{G_b}$ directly derived from $\mathbf{G_f}$ cannot represent convolution operations, which makes it impossible to derive the fan-out mode initialization through hypergraph. Motivated by that, we build Reproducing Transformation
% to represent backward propagation as a convolution (named $\mathbf{G_{bt}}$)
% by performing an equivalent replacement operation on $\mathbf{G_b}$. With $\mathbf{G_f}$ and $\mathbf{G_{bt}}$, we can unify forward and backward propagation in the identical graph representation. Prior to the transformation, we first formulate the forward process.

We build our unified initialization through derivation on a convolution hypergraph, whereby we can directly achieve the fan-in mode initialization from the forward hypergraph $\mathbf{G_f}$ since it is a natural convolution. However, the backward hypergraph $\mathbf{G_b}$ directly derived from $\mathbf{G_f}$ cannot represent a convolution as elaborated in Figure~\ref{fig:hypergraphs}, which hinders the derivation of the fan-out mode. To solve this problem, we build Reproducing Transformation to convert $\mathbf{G_b}$ to a convolution hypergraph $\mathbf{G_{bt}}$. Before presenting the transformation, we first formulate the forward process.

In the forward process of a convolutional layer, we denote the output tensor by $\ca{Y}$ and the input tensor by $\ca{X}$. Then we have $\ca{Y}=\bm{a}(\bm{f}(\ca{X}, {\bm \theta}))\triangleq \bm{g}(\ca{X})$, where $\bm{f}(\cdot)$ means a linear mapping function, ${\bm \theta}$ denotes parameters of $\bm{f}(\cdot)$, and $\bm{a}(\cdot)$ denotes an activation function (usually a ReLU function).

For the backward propagation, $\bm{\mathfrak{L}}$ denotes the Loss. In this process, we utilize a reversal matrix and a transformation matrix to achieve the equivalent transformation. These two auxiliary matrices will only change element position when they contract with another tensor, which helps calculate the variance of data-flow and weight vertices.

\textbf{Reversal Matrix.}
~A reversal matrix $\mathbf{R} \in \mathbb{R}^{r \times r}$ is an anti-diagonal matrix, where $\mathbf{R}_{ij}=1$ when $i+j =r-1$, $\mathbf{R}_{ij}=0$ otherwise.

\textbf{Transformation Matrix.}
~A transformation matrix $\mathbf{T} \in \mathbb{R}^{t \times \tilde{t}}$ is an identity-like matrix, where $\tilde{t} = \varepsilon(t-1) + 1$ and $\varepsilon\in \mathbb{R}^{N}$ is a coefficient. $\mathbf{T}_{ij}=1$ when $i=\frac{j}{\varepsilon}$, $\mathbf{T}_{ij}=0$ otherwise. 

With these two matrices, we can derive Theorem~\ref{thm:backdummy}. 

\begin{theorem}
\label{thm:backdummy}
Given a vector $\mathbf{a} \in \mathbb{R}^{\alpha}$ and a vector $\mathbf{b} \in \mathbb{R}^{\beta}$, let $\mathbf{y}=\mathbf{a}\circledast \mathbf{b} \in \mathbb{R}^{\alpha '}$, then $\Delta \mathbf{a}=\Delta \mathbf{y}\mathbf{T}\circledast \mathbf{R}\mathbf{b}$, where $\mathbf{R} \in \mathbb{R}^{\beta \times \beta}$ denotes a reversal matrix, $\mathbf{T} \in \mathbb{R}^{\alpha' \times \tilde{\alpha}'}$ represents a transformation matrix, $\circledast$ means convolution operation and $\Delta \bullet \triangleq \frac{\partial \bm{\mathfrak{L}}}{\partial \bullet}$ denotes the gradient.
\end{theorem}

The proof of Theorem~\ref{thm:backdummy} is provided in Appendix~\ref{sec:profthmdummy}.
Theorem~\ref{thm:backdummy} is corresponding to the equivalent replacement in Figure~\ref{fig:replace}. We implement the Reproducing Transformation by applying the equivalent replacements to the original backward hypergraph $\mathbf{G_b}$, then, we contract $\mathbf{R}$ and $\mathbf{T}$ with the weight vertex and the gradient vertex, respectively. Finally, we can obtain the transformed backward hypergraph $\mathbf{G_{bt}}$ which denotes the backward convolution.
% Similar to the forward propagation, the backward propagation denotes a convolutional layer formulated as ${\Delta \ca{X}}={{\bm{g}}}'({\Delta \ca{Y}})$, where ${\bm{g}}'$ denotes the corresponding backward functions of ${\bm{g}}$, and $\Delta \ca{Y}$ means $\frac{\partial \mathfrak{L}}{\partial \ca{Y}}$, the gradient.
We show some Reproducing Transformation cases in Figure~\ref{fig:super-figures}.

\begin{figure*}[t]
	\centering
	\includegraphics[width=0.98\textwidth]{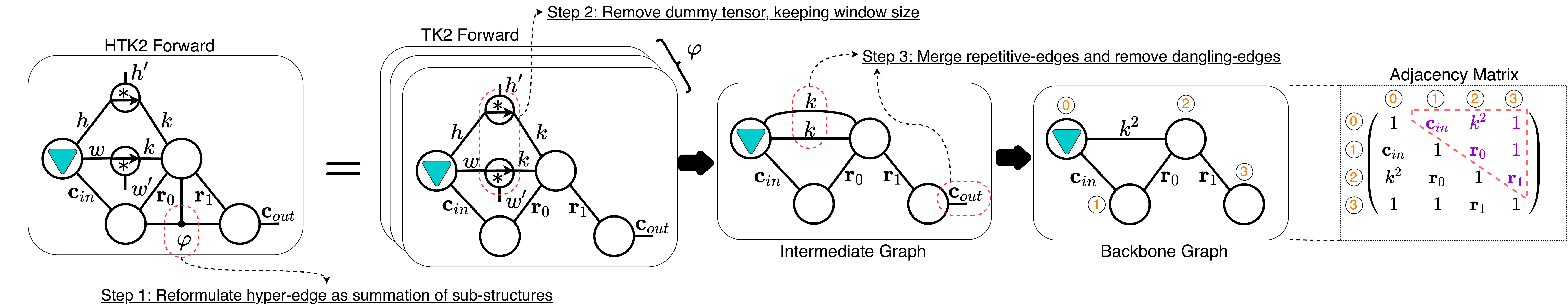}
	\caption{An example of deriving Graph-in mode for Hyper Tucker-2 (HTK2) convolution. Step 1: Since a hyperedge $\varphi$ indicates adding operation over $\varphi$ sub-structures (Tucker-2 here), we can derive the whole architecture initialization by processing each sub-structure. Step 2: Since a convolution only calculates on the kernel window, we can remove the dummy tensors by leaving kernel $k$ to derive Intermediate Graph (IG). Step 3: Since elements of IG have same variance, we can further diminish $\mathbf{c}_{out}$ edge while merging repetitive-edges to derive Backbone Graph (BG). Then the initial variance of convolutional weights can be derived as $\frac{1}{\sqrt[3]{\bm{p}_{\bm{a}}\varphi k^2 {\mathbf{c}}_{in}{\mathbf{r}}_{0}{\mathbf{r}}_{1}}}$ in terms to the adjacent matrix of BG, where $\bm{p}_{\bm{a}}$ denotes the scale of activation function. Graph-out case is shown in Figure~\ref{fig:backward-init} of Appendix.}
	\label{fig:forward-init}
\end{figure*}

\subsection{Unified Paradigm}
\label{sec:unified-paradigm}

% To analyse variance change through tensor formats, we derive Proposition~\ref{pps:tensorsum} and Proposition~\ref{pps:tensorprod} to describe the relationship between variance and tensor calculation.

Here, we will derive a unified paradigm through variance analysis. We first give  Proposition~\ref{pps:tensorsum} and Proposition~\ref{pps:tensorprod} to describe the relationship between variance and tensor calculation. Then we introduce Backbone Graph (BG) to illustrate inner production in a hypergraph. At last, we obtain the paradigm in terms of BG and these two propositions.

\begin{proposition}
    \label{pps:tensorsum}
    Given tensors $\ca{X} \in \mathbb{R}^{{\mathbf i}_0 \times {\mathbf i}_1 \times \dots \times {\mathbf i}_{m-1}}$ and $\ca{Y} \in \mathbb{R}^{{\mathbf i}_0 \times {\mathbf i}_1 \times \dots \times {\mathbf i}_{m-1}}$, where elements of $\ca{X}$ and $\ca{Y}$ are independent with each other, the variance of their element-wise sum $\ca{Z} = \ca{X} + \ca{Y}$ is
    \begin{equation}
    {\sigma^2}{(\ca{Z})} = {\sigma^2}{(\ca{X})} + {\sigma^2}{(\ca{Y})}.
    \end{equation}
\end{proposition}
\begin{proposition}
    \label{pps:tensorprod}
    A tensor $ \ca{X} \in \mathbb{R}^{{\mathbf i}_0 \times {\mathbf i}_1 \times \dots \times {\mathbf i}_{m-1}}$ (i.i.d.) and a tensor $\ca{Y}$ contract $d$ dimensions ($d \leq \min{(m, n)}$), where $\ca{Y} \in \mathbb{R}^{{\mathbf j}_0 \times {\mathbf j}_1 \times \dots \times {\mathbf j}_{n-1}}$ is i.i.d. and follows a zero-mean symmetrical distribution.  The ${\mathbf x}_t$-th dimension of $\ca{X}$ corresponds to the ${\mathbf y}_t$-th dimension of $\ca{Y}$, where ${\mathbf x}_t\neq {\mathbf x}_u$ and ${\mathbf y}_t\neq {\mathbf y}_u$ if $t\neq u$, ${\mathbf x}_t\leq m-1$, and ${\mathbf y}_t\leq n-1$. Without loss of generality, let ${\mathbf i}_{{\mathbf x}_t}={\mathbf j}_{{\mathbf y}_t}={\mathbf v}_t$, for $t \in \{0, 1, \dots, d-1\}$. The variance of contracted tensor  $\ca{Z} = \ca{X} \times_{{\mathbf x}_0, {\mathbf x}_1, \dots, {\mathbf x}_{d-1}}^{{\mathbf y}_0, {\mathbf y}_1, \dots, {\mathbf y}_{d-1}} \ca{Y}$ is calculated by
    \begin{equation}
    {\sigma^2}{(\ca{Z})} = {\sigma^2}{(\ca{X})}  {\sigma^2}{(\ca{Y})} \prod_{t=0}^{d-1}{{\mathbf v}_t}.
    \end{equation}
\end{proposition}

The proofs of the two propositions are given in Appendix~\ref{sec:pro-corosum} and \ref{sec:pro-coroprod}. It is worth mentioning that $\ca{X}$ in Proposition~\ref{pps:tensorprod} is hard to satisfy i.i.d, but assuming $\ca{X}$ non-i.i.d is still applicable in practice as the empirical elaboration in Appendix~\ref{sec:pro-validation}.

\subsubsection{Backbone Graph}

According to Proposition~\ref{pps:tensorprod}, variance change depends not only on weight and input, but also on contracted dimension ${\mathbf v}_t$.
Therefore, we introduce Backbone Graph (BG) that only contains contracting edges (i.e., contracted dimensions). Figure~\ref{fig:forward-init} shows a process to derive BG from a dummy tensor based convolution. An adjacency matrix of $\tau$-vertex BG is defined as $\mathbf{E} \in \mathbb{R}^{\tau\times \tau}$, whose element $e_{ij}$ satisfying $e_{ij} = e_{ji}$ and diagonal element $e_{ii}=1$, where $i, j\in \{0, 1,\dots, \tau-1\}$.
% In addition, we set $e_{ij}=1$ if there is no edge between vertex $i$ and vertex $j$. 
As shown in Figure~\ref{fig:forward-init}, the adjacency matrix in the tensor diagram is specially designed to fit the calculation of the variance where each element denotes the contraction between two nodes.
% Thus, $e_{ij}=1$ means the dimension is equal to 1 after contracting node $i$ and node $j$, suggesting that there is no edge between node $i$ and node $j$.
Thus, $e_{ij}=1$ means the contracting dimension between node $i$ and node $j$ is equal to 1, suggesting that there is no edge between node $i$ and node $j$. $\mathbf{E}$ is symmetric and each vertex does not connect to itself. A supergraph denotes an output tensor $\ca{Y}$.
We use $BG(\mathbf{E})$ to denote the Backbone Graph that comes from $\ca{Y}$. $BG(\mathbf{E})$ can be regarded as an element $\ca{Y}_{\ast}$ of $\ca{Y}$.

\subsubsection{Derivation for Unified Paradigm}
\label{sec:derive-unify}

Since $\mathbf{E} \in \mathbb{R}^{\tau \times \tau}$ is symmetric, we consider  edges $e_{ij}$ satisfying $ i < j$ only. Then based on Proposition~\ref{pps:tensorprod}, we present Theorem~\ref{thm:total} to reveal the scale after the input through a TCNN. The proof of Theorem~\ref{thm:total} is in Appendix~\ref{sec:prototal}.
\begin{theorem}
	\label{thm:total}
	Assume the input $\ca{X}$ contracts with $n$ weight vertices $\{\ca{W}^{(i)}\}_{i=0}^{n-1}$. Meanwhile, input variance is ${\sigma^2}(\ca{X})$ and output variance is ${\sigma^2}(\ca{Y})$, then
	\begin{align}
	\label{eq:varequal}
	    {\sigma^2}(\ca{Y}) = {\sigma^2}(\ca{X})\prod_{k=0}^{n-1}{{\sigma^2}(\ca{W}^{(k)})}\prod_{i=0}^{n-1}\prod_{j=i+1}^{\tau-1}{{e}_{ij}}.
	\end{align}
\end{theorem}

\begin{figure*}[t]
	\centering
	\subfigure[HOdd-5 ($\varphi$=1)  Layer1]{
		\includegraphics[width=0.22\textwidth]{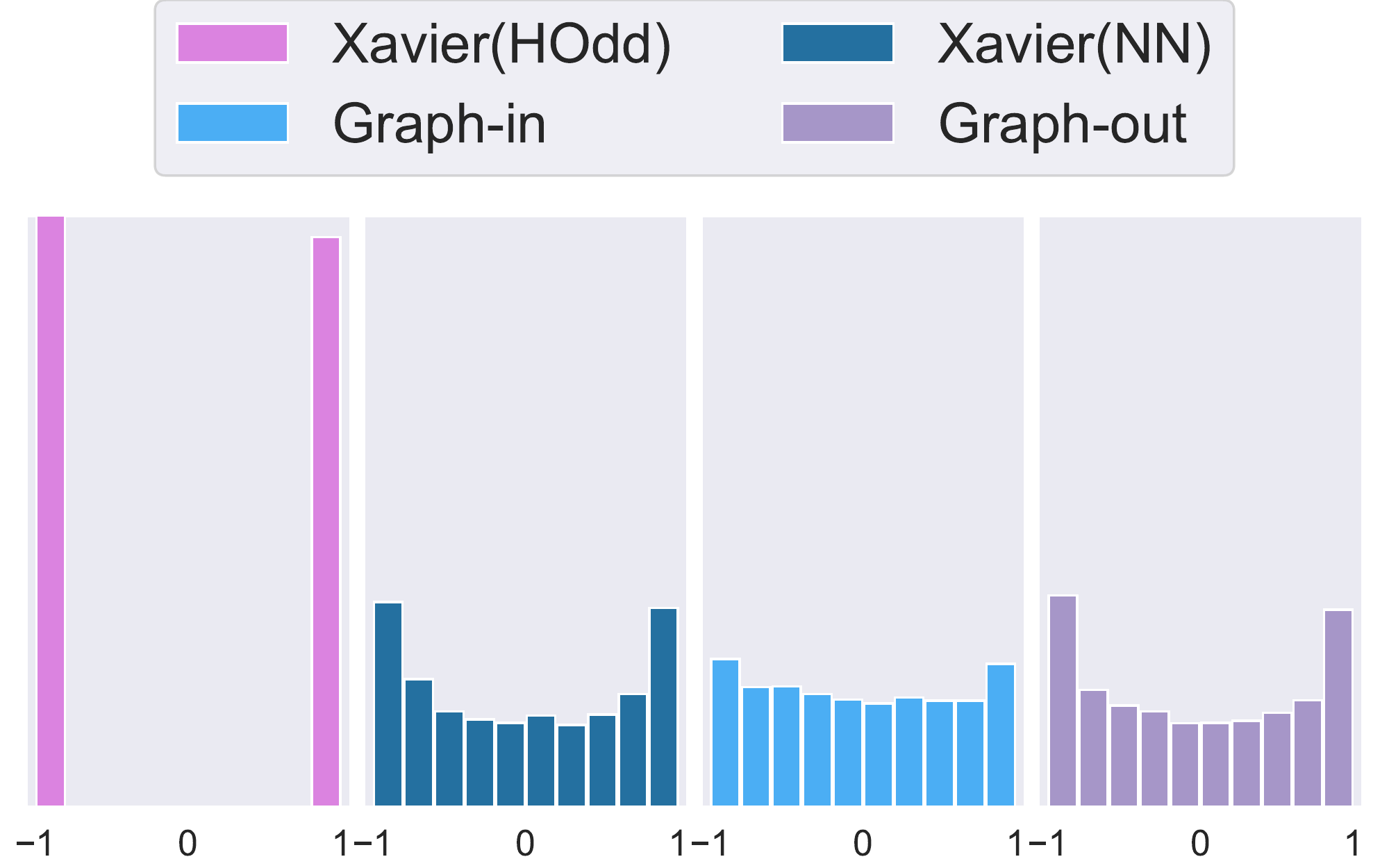}
		\label{fig:hoddc1-layer1}
	}
	\subfigure[HOdd-5 ($\varphi$=1)  Layer4]{
		\includegraphics[width=0.22\textwidth]{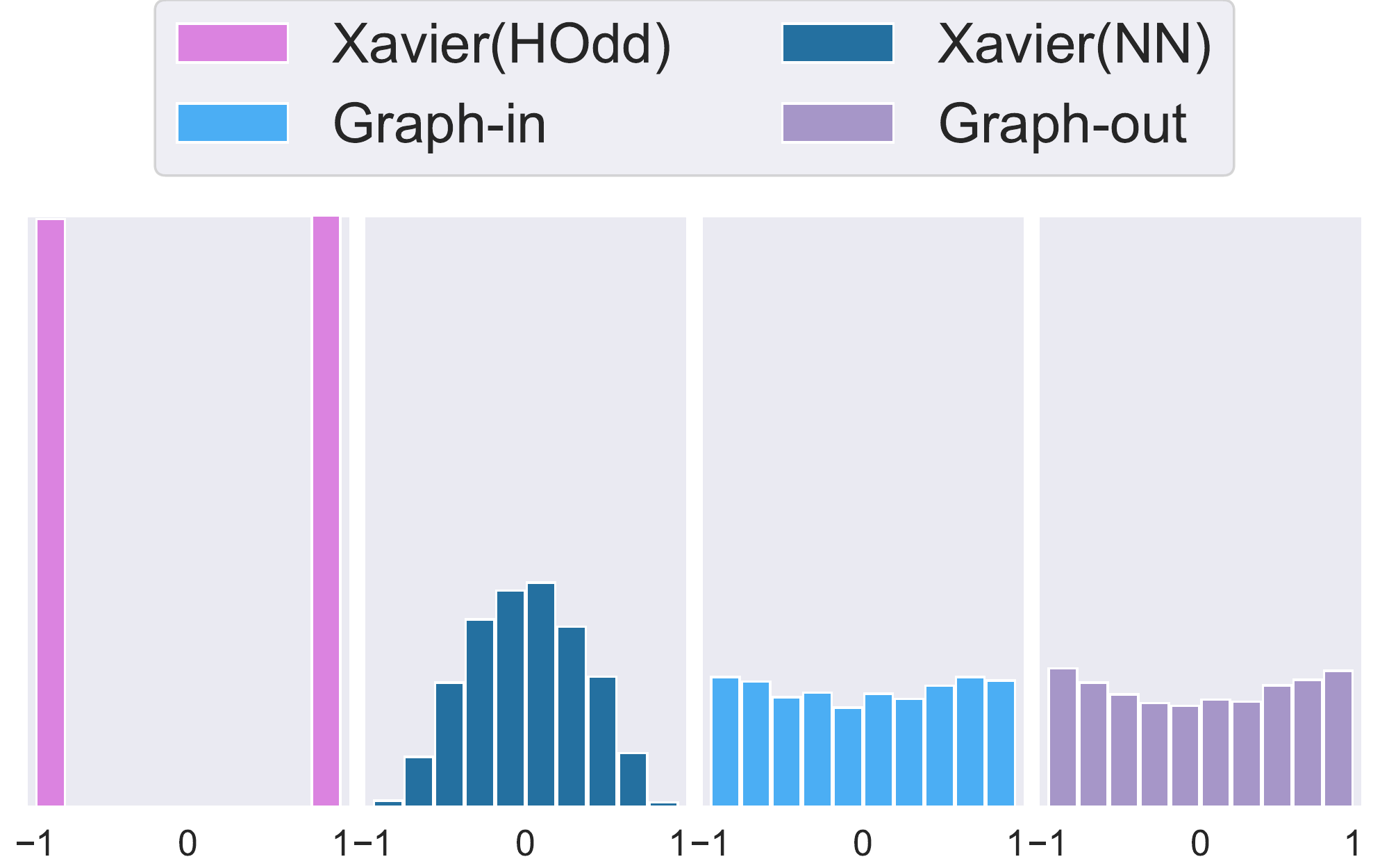}
		\label{fig:hoddc1-layer4}
	}
	\subfigure[HOdd-5 ($\varphi$=4)  Layer1]{
		\includegraphics[width=0.22\textwidth]{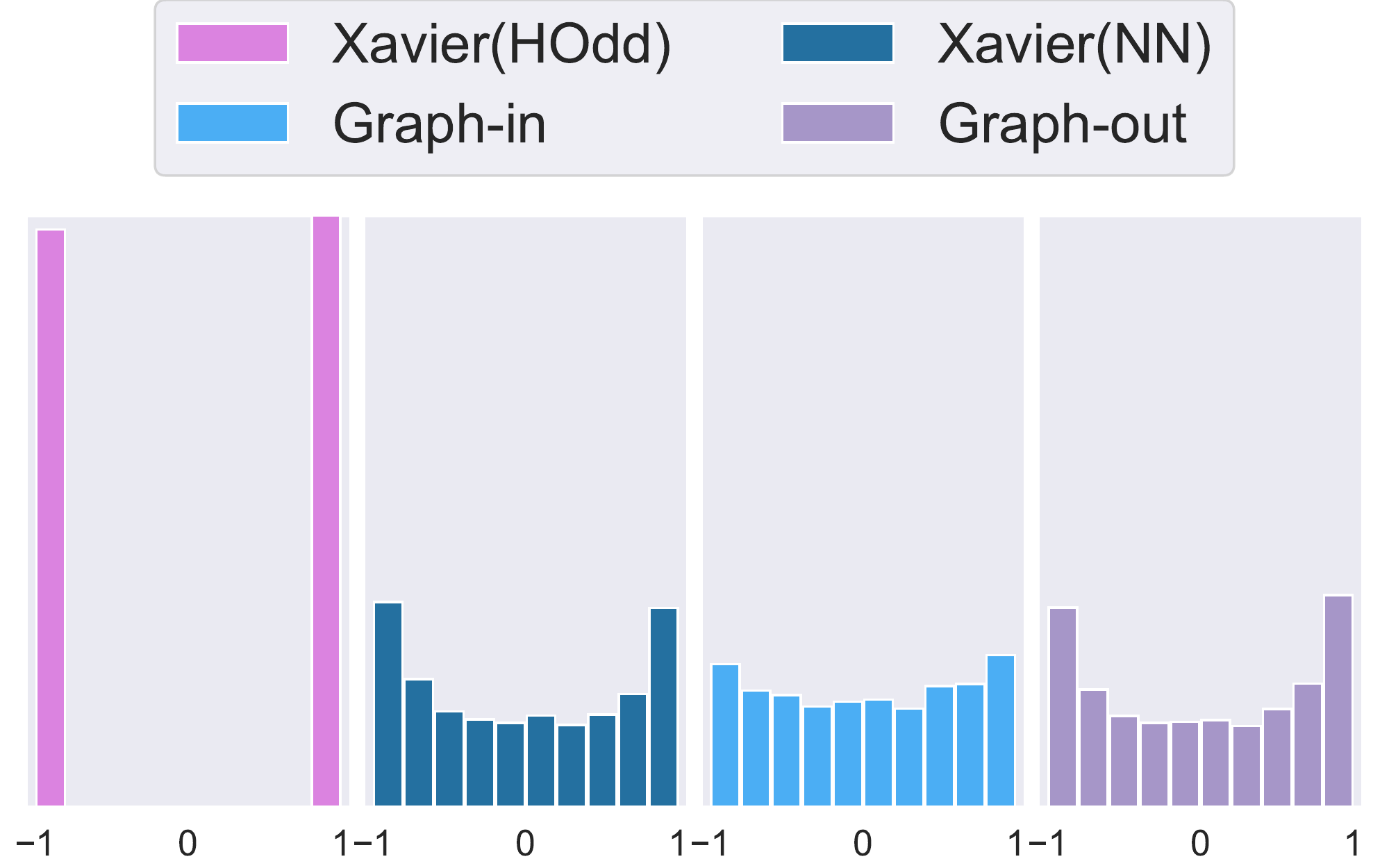}
		\label{fig:hoddc4-layer1}
	}
	\subfigure[HOdd-5 ($\varphi$=4) Layer4]{
		\includegraphics[width=0.22\textwidth]{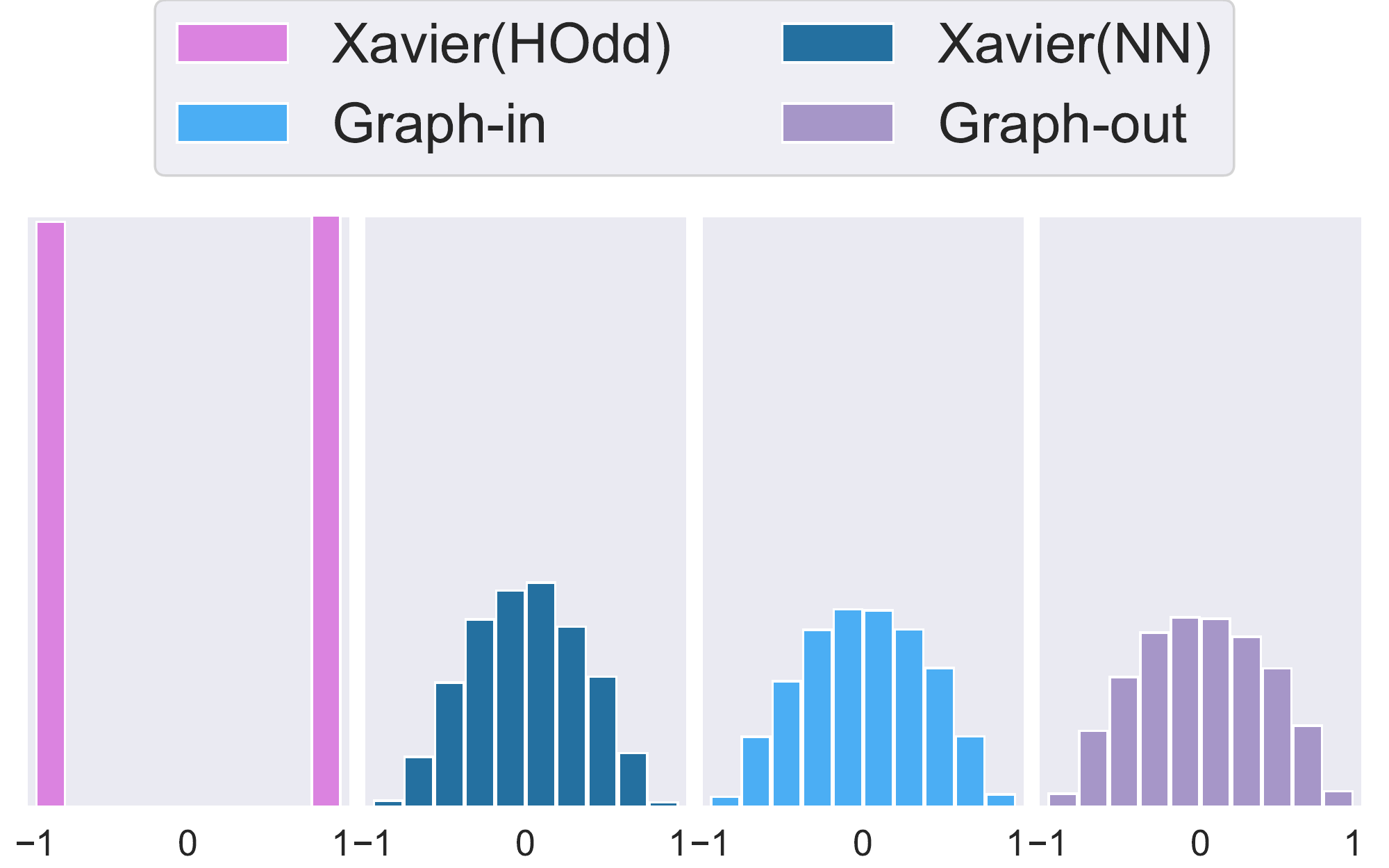}
		\label{fig:hoddc4-layer4}
	}
	
	\subfigure[HOdd-5 ($\varphi$=1)  Loss]{
		\includegraphics[width=0.22\textwidth]{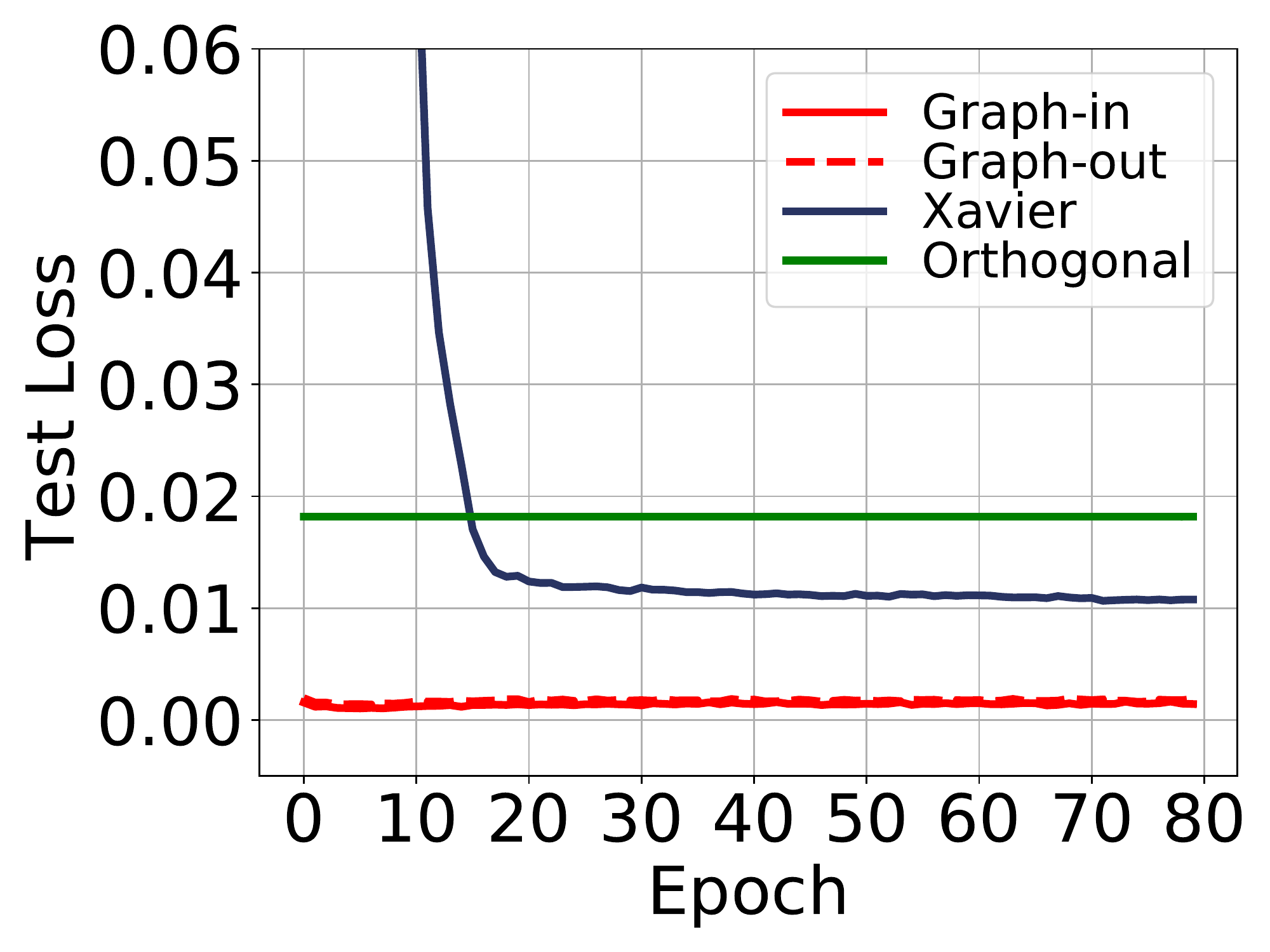}
		\label{fig:hoddc1-loss}
	}
	\subfigure[HOdd-5 ($\varphi$=1) Accuracy]{
		\includegraphics[width=0.22\textwidth]{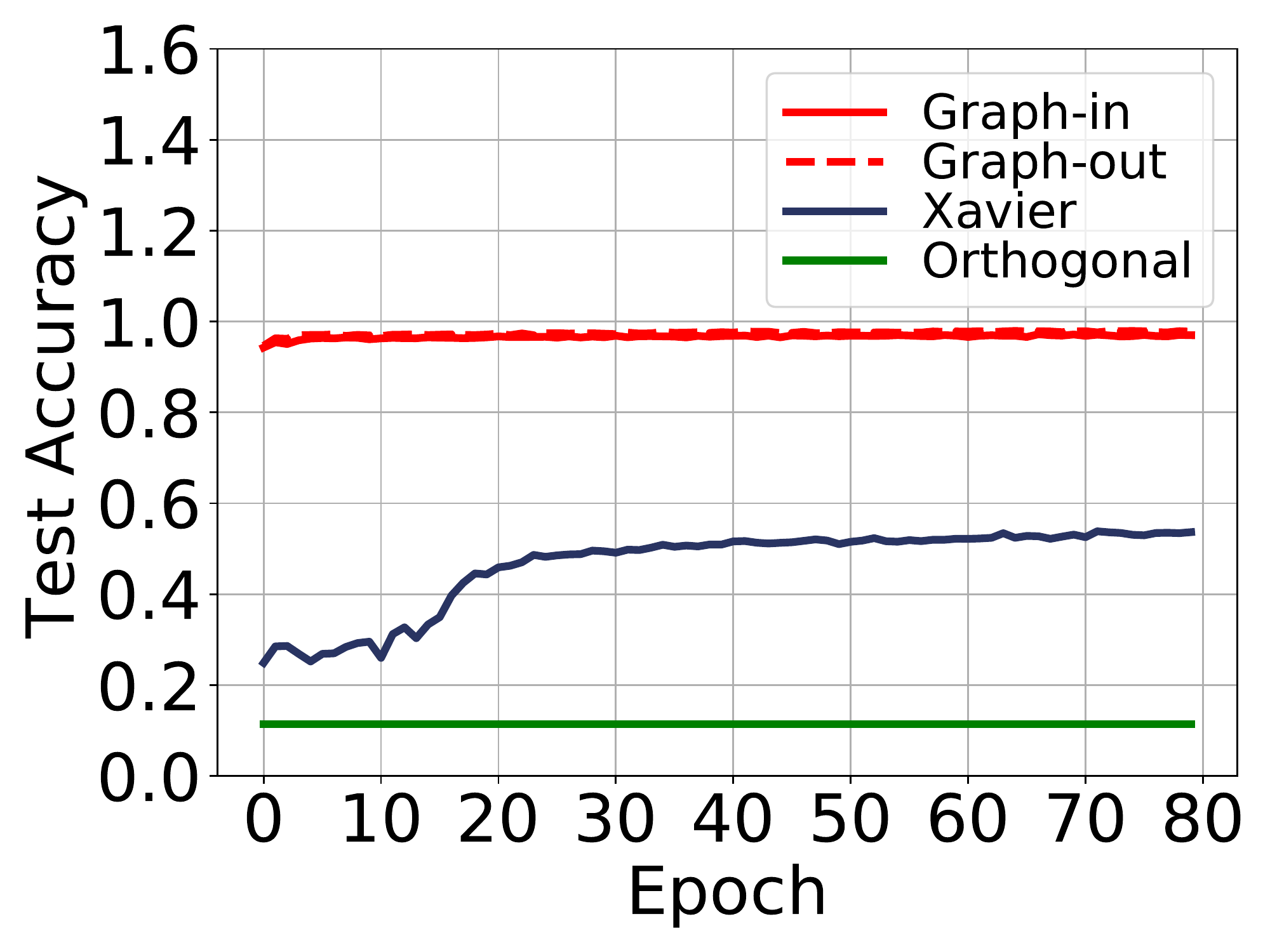}
		\label{fig:hoddc1-acc}
	}
	\subfigure[HOdd-5 ($\varphi$=4) Loss]{
		\includegraphics[width=0.22\textwidth]{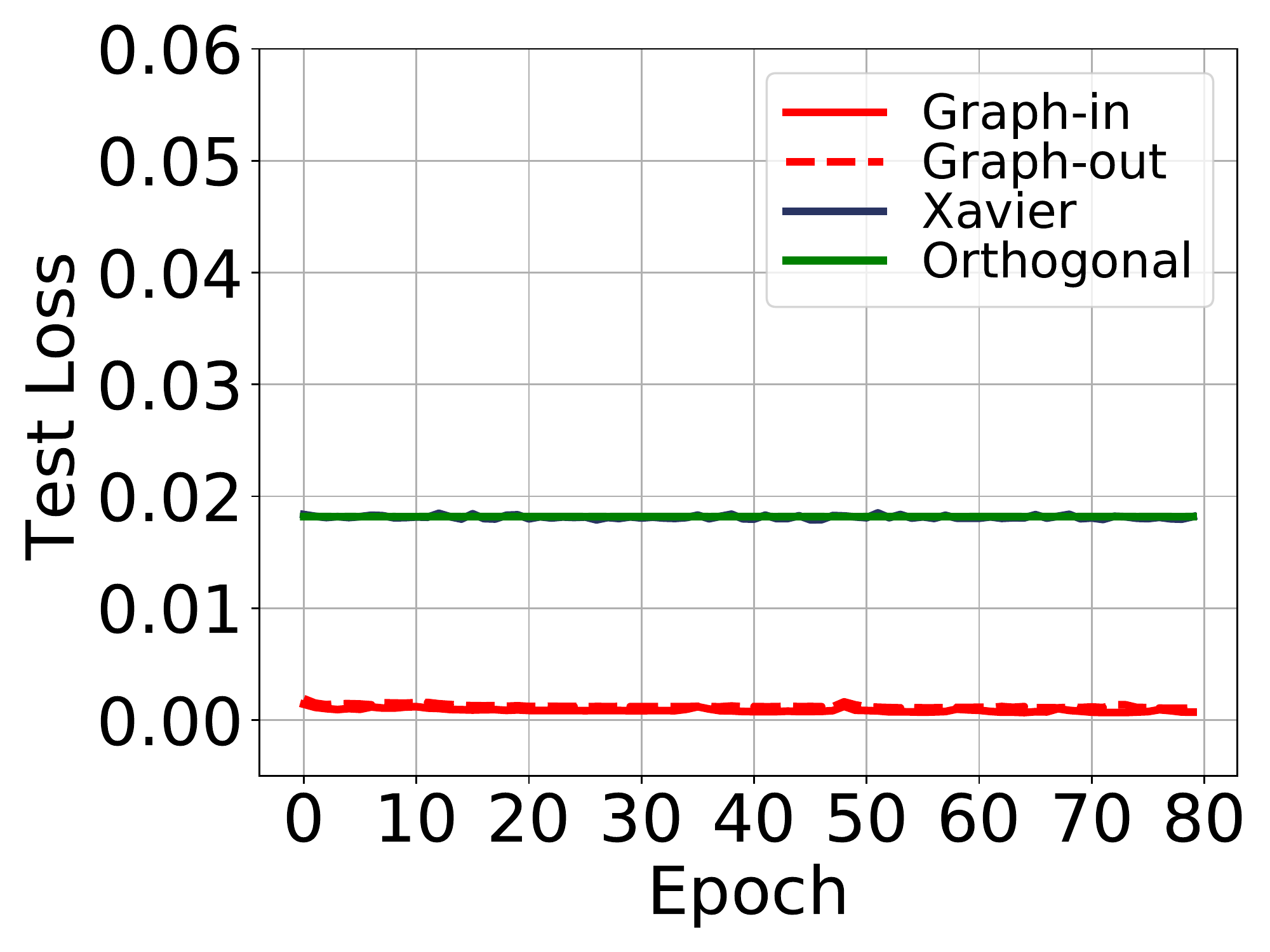}
		\label{fig:hoddc4-testloss}
	}
	\subfigure[HOdd-5 ($\varphi$=4) Accuracy]{
		\includegraphics[width=0.22\textwidth]{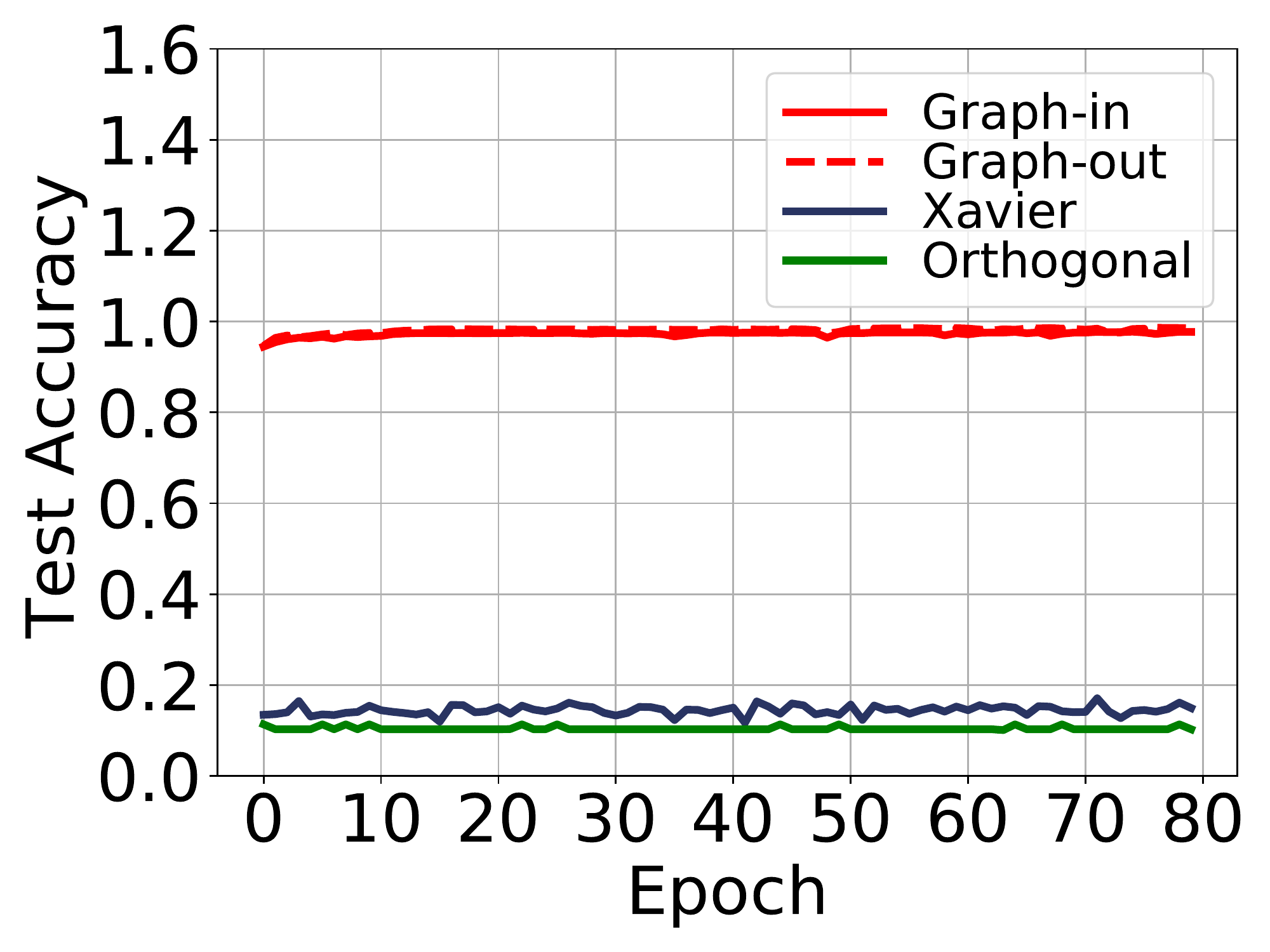}
		\label{fig:hoddc4-testacc}
	}
	\caption{Activation distribution before training and results of the activation propagation analysis. Xavier (NN) and Xavier (HOdd) represent the case of applying Xavier initialization to Linear-5 and HOdd-5 respectively. Graph-in and Graph-out represent the case of applying the proposed initialization to HOdd-5. $\varphi$ means a hyperedge. Xavier (NN) works since maintaining activation in the unsaturated region of activation function tanh, and Graph(-in/-out) also benefits from this. Under $\varphi$=1 and $\varphi$=4, activation of Graph(-in/-out) distributes in the unsaturated region, which indicates that Graph(-in/-out) can fit the sophisticated HOdd format and integrate with a hyperedge. Nevertheless, Xavier (HOdd) suffers from activation explosion in the saturated region and fails to train HOdd-5. Orthogonal initialization cannot train HOdd-5 either. By contrast, Graph(-in/out) successfully trains the model and derives relatively good results.}
	\label{fig:linear5distribution}
\end{figure*}

Next, considering activation function and a hyperedge $\varphi$, variance of final output $\ca{Y}_o$ is ${\sigma^2}(\ca{Y}_o) = \bm{p}_{\bm{a}}\varphi{\sigma^2}(\ca{Y})$ according to Proposition~\ref{pps:tensorsum}, where $\bm{a}$ is an activation map, $\varphi$ means a hyperedge value, and $\bm{p}_{\bm{a}}$ denotes scale caused by activation function. For example, $\bm{p}_{\text{ReLU}}=\frac{1}{2}$ and $\bm{p}_{\text{tanh}}=1$. We set ${\sigma^2}(\ca{Y}_o) = {\sigma^2}(\ca{X})$ to maintain the data-flow variance equal. 
Thus, We can re-formulate Eq.~\eqref{eq:varequal} as 
\begin{align}
    \label{eq:equal-relation} \frac{{\sigma^2}(\ca{X})}{\bm{p}_{\bm{a}}\varphi} = {\sigma^2}(\ca{X})\prod_{k=0}^{n-1}{{\sigma^2}(\ca{W}^{(k)})}\prod_{i=0}^{n-1}\prod_{j=i+1}^{\tau-1}{{e}_{ij}}.
\end{align}
From Eq.~\eqref{eq:equal-relation}, We find that ${\sigma^2}(\ca{Y}_o)$ is highly related to $\varphi$ and edges of BG, and will change exponentially when edge number increases. Notably, Xavier and Kaiming fail since they only consider channel edges and convolutional window size edges, namely, part of edges. An expository example is in Appendix~\ref{sec:variance-example}.
As a result, we can derive
\begin{align}
\label{eq:result}
\prod_{k=0}^{n-1}{\sigma^2}(\ca{W}^{(k)}) = \frac{1}{\bm{p}_{\bm{a}}\varphi\prod_{i=0}^{n-1}\prod_{j=i+1}^{\tau-1}{e_{ij}}}.
\end{align}
If the initialized weight satisfies Eq.~\eqref{eq:result}, then we can attain the same effects as what Xavier and Kaiming achieve, even on multi-vertex tensor graphs. Thus, Eq.~\eqref{eq:result} can serve as a \textbf{unified paradigm} to ensure the effectiveness of weight initialization methods on TCNNs.

\subsection{A Simple Initialization Exemplar}
\label{sec:simple-exemplar}

To ensure that Eq.~\eqref{eq:result} holds, there are plenty of choices to set the variance of weight vertices, which indicates potentially numerous weight initialization schemes. To verify the feasibility of our paradigm, we propose an exemplar choice by setting all the variance of weight vertices the same through
\begin{align}
    \label{eq:avg-set}
    {\sigma^2}(\ca{W}^{(\ast)}) = \frac{1}{\sqrt[n]{\bm{p}_{\bm{a}}\varphi\prod_{i=0}^{n-1}\prod_{j=i+1}^{\tau-1}{e_{ij}}}}.
\end{align}

In this way, we can determine a specific weight initialization method, to which we refer as Graph Initialization. It has two modes, \textbf{Graph-in} and \textbf{Graph-out}, similar to fan-in and fan-out modes of Kaiming initialization.

% \textbf{Graph-in}~~We take HTK2 as an instance to derive a Graph-in mode in Figure~\ref{fig:forward-init}. By connecting input $\ca{X}$ with the convolutional kernel, the output $\ca{Y}$ is derived. Then, the Graph-in variance of each weight vertex $\ca{W}^{(\ast)}$ can be calculated as described in Eq.~\eqref{eq:avg-set}.

% \textbf{Graph-out}~~Since that the backward graph is similar to the forward graph, we illustrate an HTK2 backward instance in Figure~\ref{fig:backward-init} at Appendix~\ref{sec:backinit}. The backward graph is convenient to be illustrated by matching the $\Delta \ca{Y}$ (i.e., $\frac{\partial \bm{\mathfrak{L}}}{\partial \ca{Y}}$) and the convolutional kernel. Similar to forward graph, initial variance for Graph-out mode can also be derived by Eq.~\eqref{eq:avg-set}.
Graph-in and Graph-out are constructed by applying Eq.~\eqref{eq:avg-set} on a TCNN's $\mathbf{G_f}$ and $\mathbf{G_{bt}}$, respectively. We take derivation of Graph-in for HTK2 convolution as an instance in Figure~\ref{fig:forward-init}. After extracting BG from the HTK2 $\mathbf{G_f}$, we can calculate suitable initial variance for weights of HTK2 convolution, by applying Eq.~\eqref{eq:avg-set} on the BG's adjacency matrix. Graph-out is derived from $\mathbf{G_{bt}}$ exactly the same as Graph-in. We show some Graph Initialization demos in Figure~\ref{fig:super-figures}.

% Through the representation of $\mathbf{G_f}$ and $\mathbf{G_{bt}}$, weights of TCNNs can be initialized in terms to Eq.~\eqref{eq:avg-set}, and some demos are shown in Figure~\ref{fig:super-figures}.

\section{Experiment}
\label{sec:exp}

In this section, we first give an illustrative example on linear layers (i.e., 1$\times1$ convolution) to show the damage of activation amplification.
Then, we use randomly generated tensor formats to show statistic results of our Graph(-in/-out) initialization. Next, by experiments on complex HOdd networks, we verify the adaption of the proposed method in complicated situations.
Finally, via experiments on ImageNet with random networks, we show that our initialization is suitable for arbitrary TCNNs. We compare it with Xavier or Kaiming when the activation function is tanh or ReLU, respectively.
We put details of all experiments in Appendix~\ref{sec:exp-detail}, including the batch size, the learning rate, the network architectures and the training machine.

\begin{figure*}[t]
	\centering
	\subfigure[Sam. 1]{
	\includegraphics[width=0.105\textwidth]{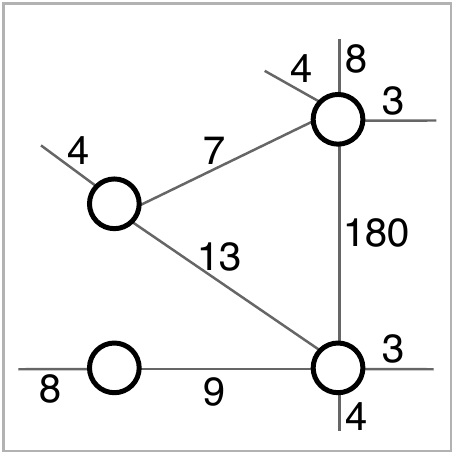}
	\label{fig:rand-sample1}
	}
	\subfigure[Sam. 2]{
	\includegraphics[width=0.105\textwidth]{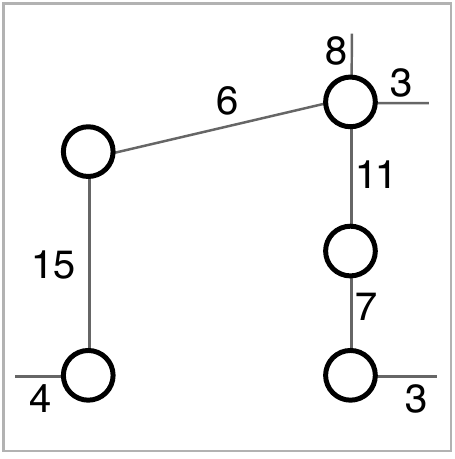}
	\label{fig:rand-sample2}
	}
	\subfigure[Sam. 3]{
	\includegraphics[width=0.105\textwidth]{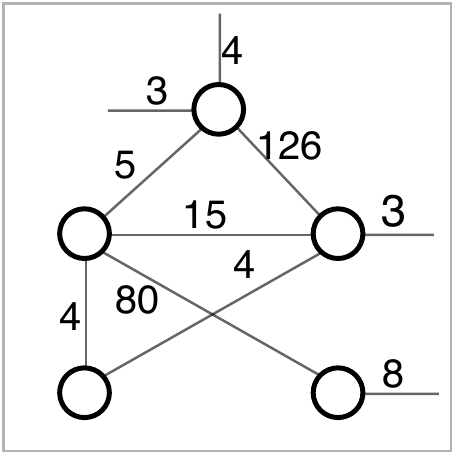}
	\label{fig:rand-sample3}
	}
	\subfigure[Sam. 4]{
	\includegraphics[width=0.105\textwidth]{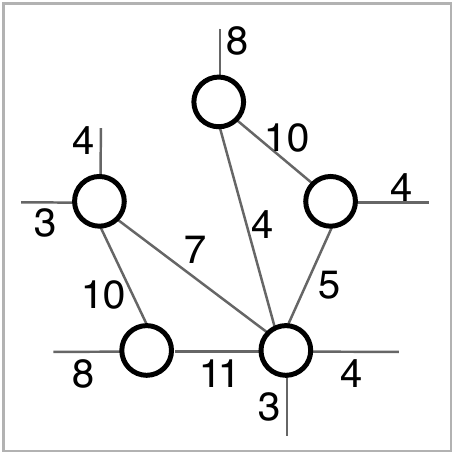}
	\label{fig:rand-sample4}
	}
	\subfigure[Sam. 5]{
	\includegraphics[width=0.105\textwidth]{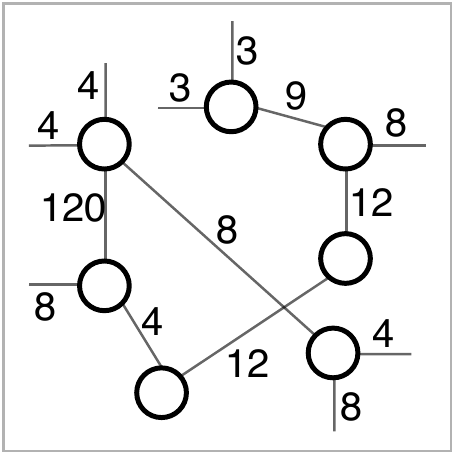}
	\label{fig:rand-sample5}
	}
	\subfigure[Sam. 6]{
	\includegraphics[width=0.105\textwidth]{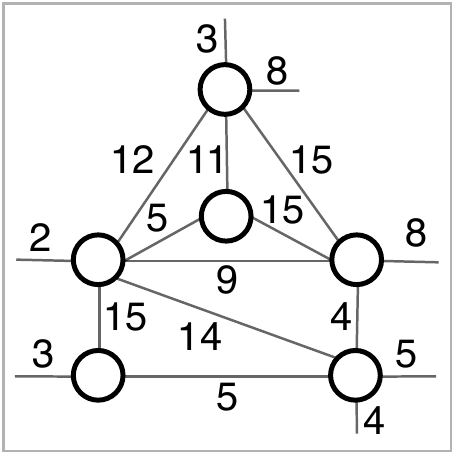}
	\label{fig:rand-sample6}
	}
	\subfigure[Sam. 7]{
	\includegraphics[width=0.105\textwidth]{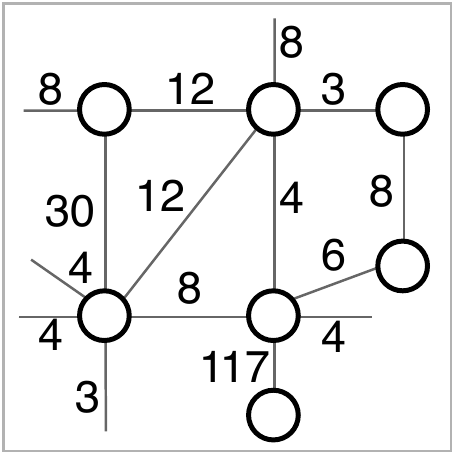}
	\label{fig:rand-sample7}
	}
	\subfigure[Sam. 8]{
	\includegraphics[width=0.105\textwidth]{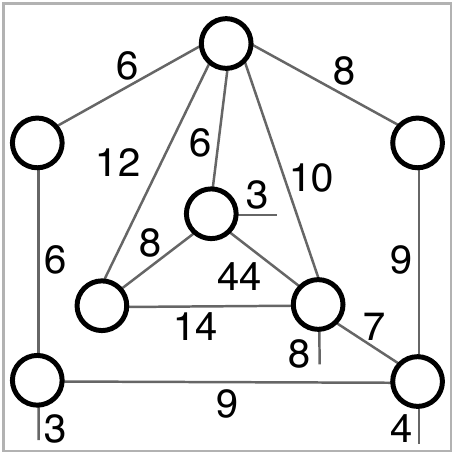}
	\label{fig:rand-sample8}
	}
	\caption{Random Tensor Formats. For more clear observation, we draw
% 	\subref{fig:conv4trainingl1}, \subref{fig:conv4trainingl2}, \subref{fig:conv4trainingl3} and \subref{fig:conv4trainingl4} 
\subref{fig:rand-sample1}~-~\subref{fig:rand-sample8}, in total eight random sub-structure TCNN samples without hyperedges ("Sam." is the abbreviation of "sample"). Notably, tensor nodes are largely different in size. For example, in \subref{fig:rand-sample7}, a large node can be of size $117\times8\times4\times6\times4=89856$ and a small one is $3\times8=24$, which indicates the random tensor formats are sophisticated sufficiently.}
	\label{fig:random-layer}
\end{figure*}

\begin{figure}[t]
	\centering
	\subfigure[MNIST Loss]{
		\includegraphics[width=0.22\textwidth]{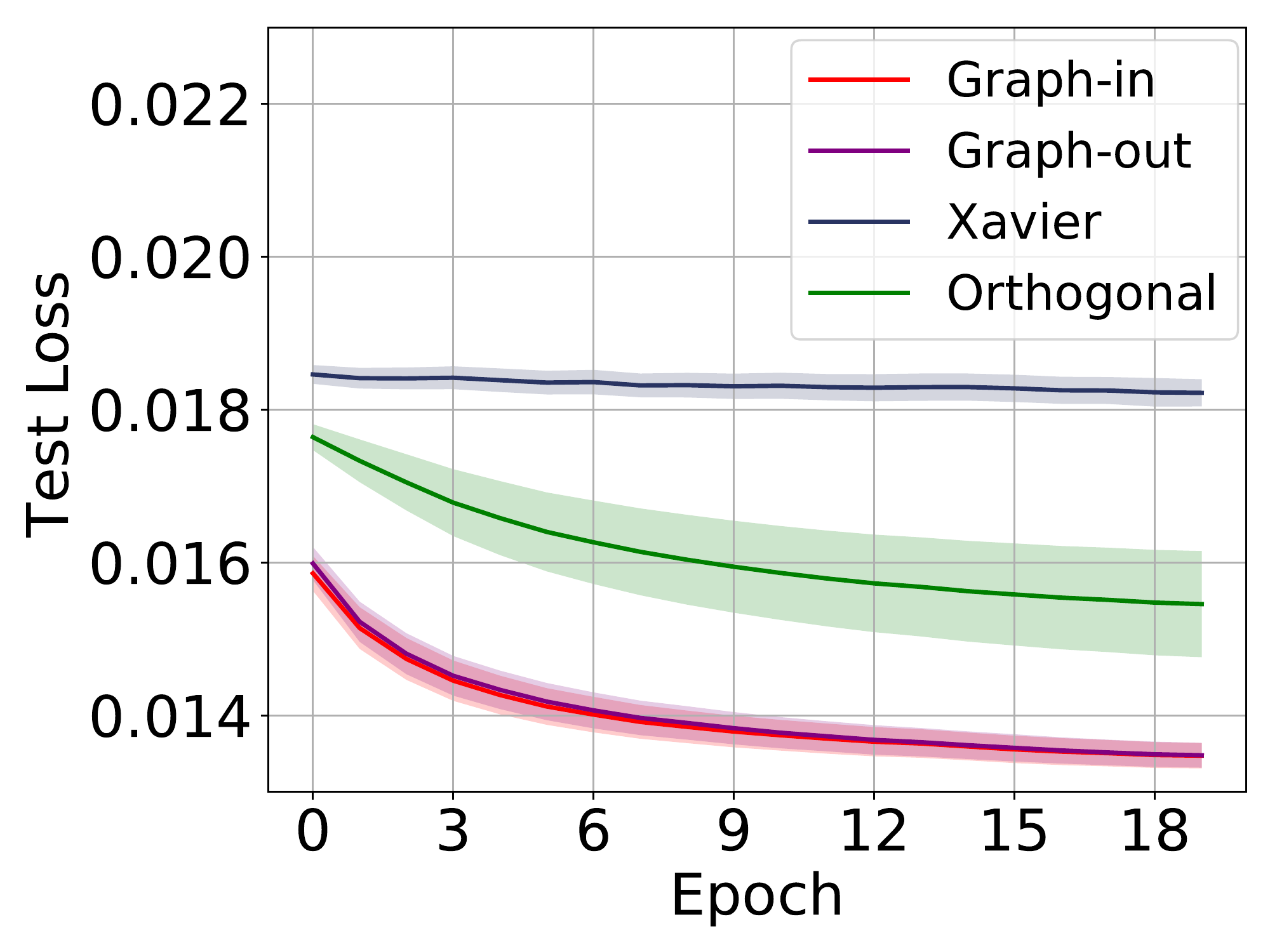}
		\label{fig:conv4trainingloss}
	}
	\subfigure[MNIST Accuracy]{
		\includegraphics[width=0.22\textwidth]{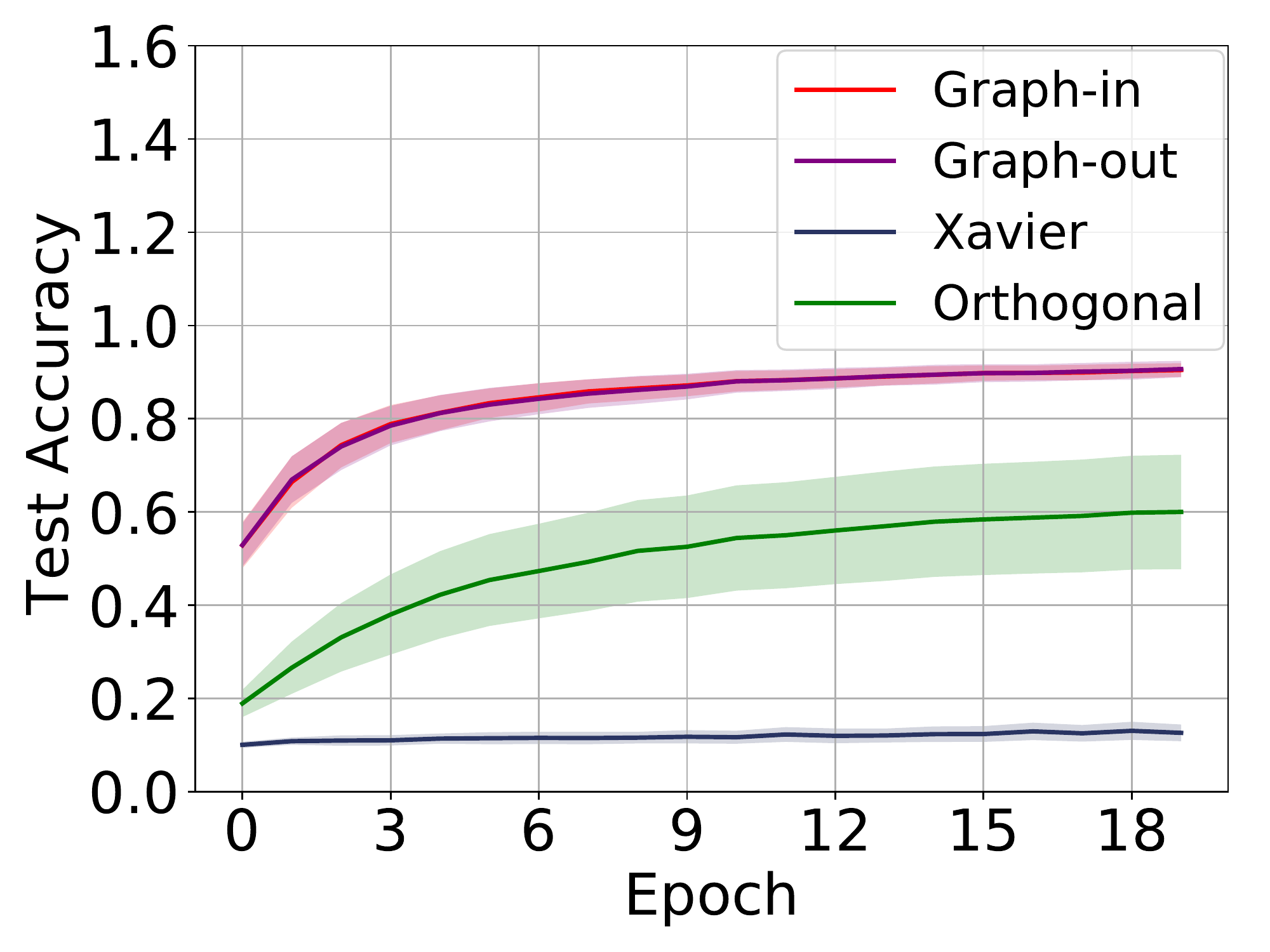}
		\label{fig:conv4trainingacc}
	}
	\caption{Results of the random tensor format experiment. \subref{fig:conv4trainingloss} and \subref{fig:conv4trainingacc} draw 150 round results on MNIST. The results show that our method works consistently well.}
	\label{fig:conv4training}
\end{figure}

\begin{figure}[t]
	\centering
	\subfigure[Cifar10 Loss]{
		\includegraphics[width=0.22\textwidth]{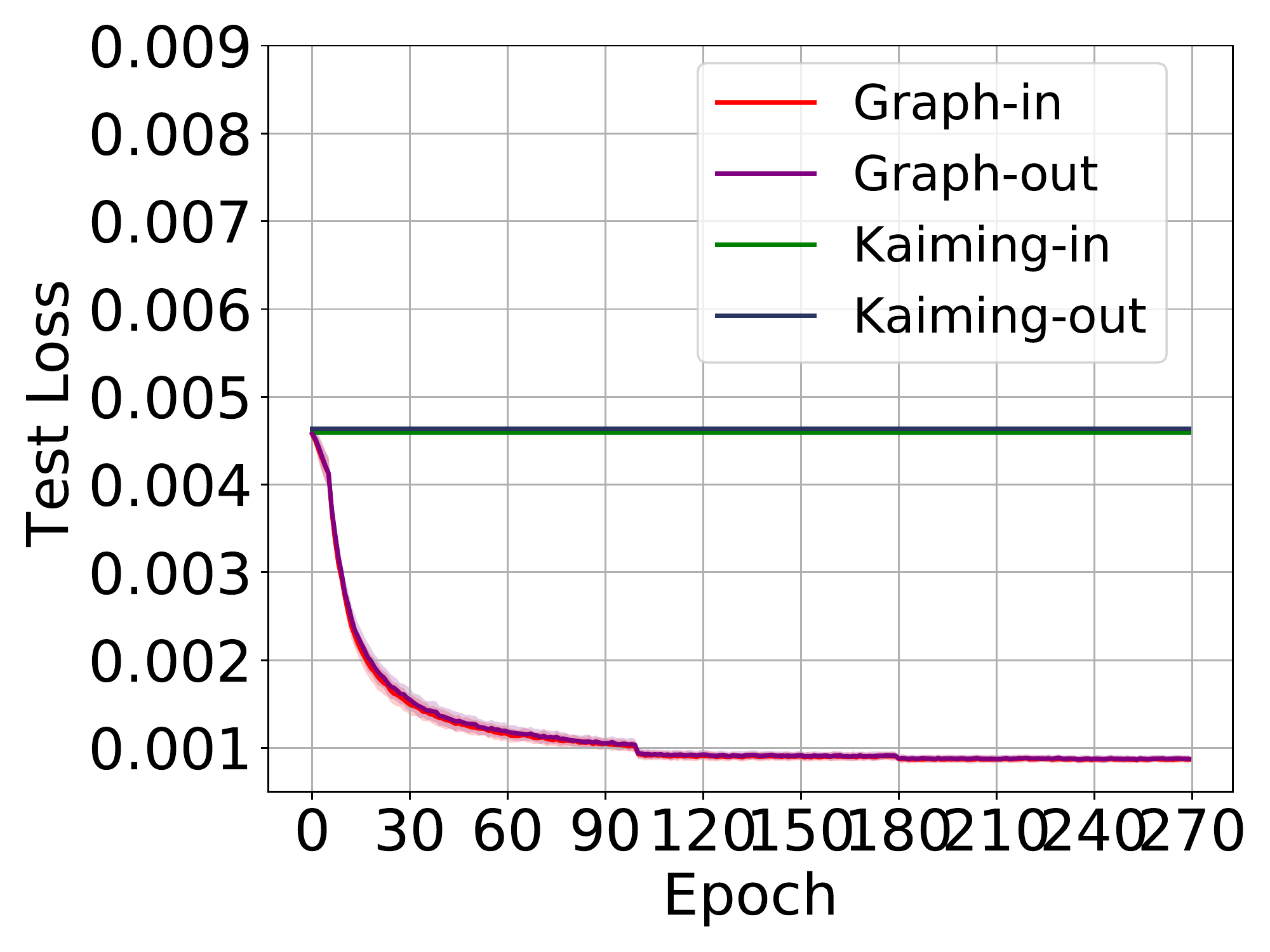}
		\label{fig:cifar10-trainingloss}
	}
	\subfigure[Cifar10 Accuracy]{
		\includegraphics[width=0.22\textwidth]{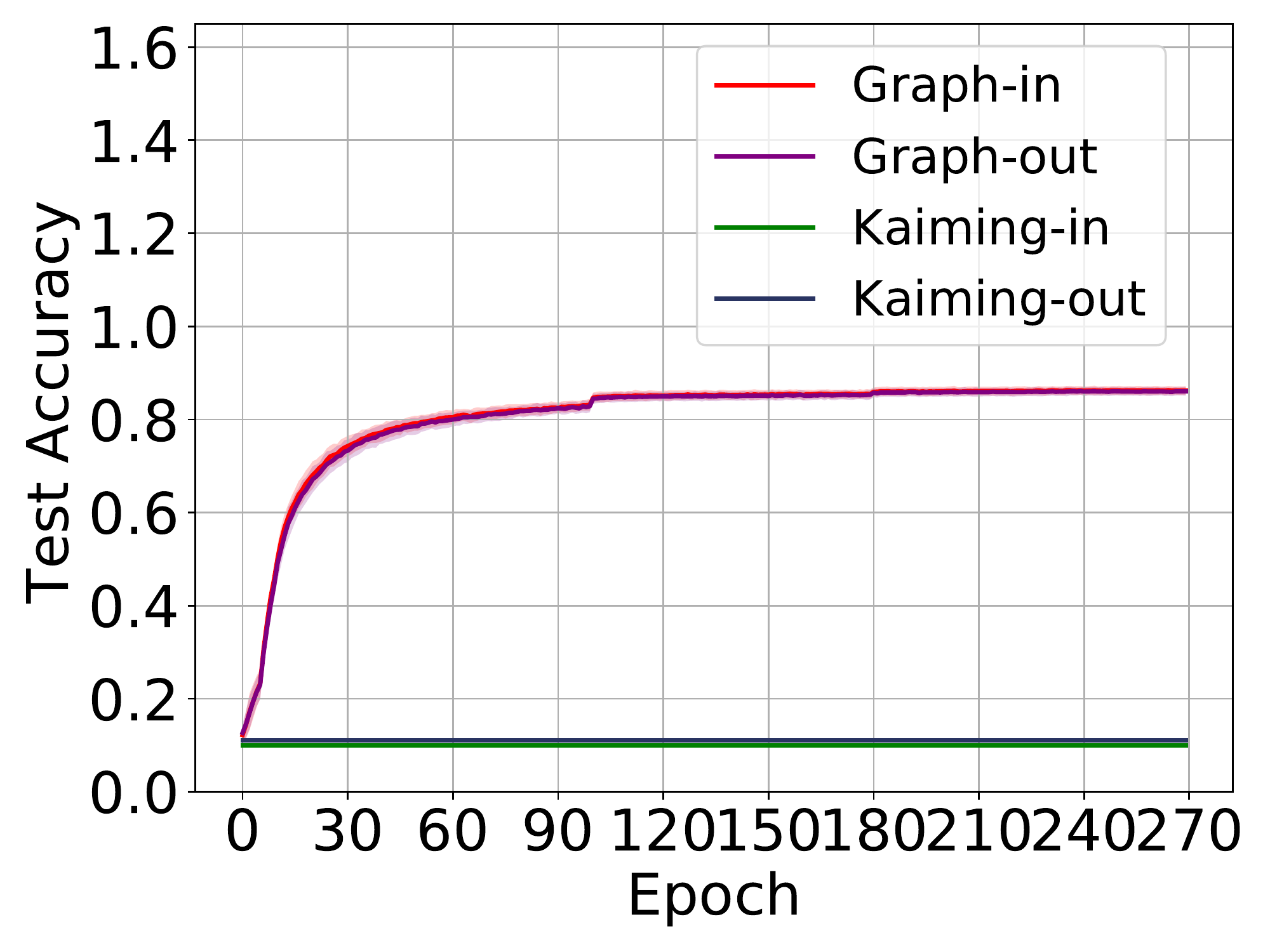}
		\label{fig:cifar10-trainingacc}
	}
	\caption{Results of the Cifar10 experiment.  \subref{fig:cifar10-trainingloss} and \subref{fig:cifar10-trainingacc} draw 30 round results on Cifar10. The results show that our method is robust enough to initialize 270 tensor formats.}
	\label{fig:allconvtraining}
\end{figure}

\subsection{Evaluation on MNIST}
\subsubsection{Activation Propagation Analysis}
\label{sec:converge}
We conduct this experiment on MNIST to show that activation amplification in propagation is the main factor to cause unstable training. In this experiment, the network consists of 5 linear layers (called Linear-5) or 5 HOdd based linear layers (called HOdd-5). 
Each layer has 500 hidden units. We use Xavier initialization for comparison since the activation function here is a hyperbolic tangent function. For completeness, we also include Orthogonal initialization from PYTORCH~\citep{DBLP:conf/nips/PaszkeGMLBCKLGA19} for comparison. The training process is optimized by Adam with the learning rate 1e-4. Xavier(NN) and Xavier(HOdd) represent the case of applying Xavier initialization to Linear-5 and HOdd-5, respectively. Graph-in and Graph-out represent the cases of applying the graphical initialization to HOdd-5.

Figure~\ref{fig:linear5distribution} shows distribution of activations when applying Graph Initialization and Xavier initialization. The activation distributions of both Graph-in and Graph-out are similar to these of Xavier(NN). Specifically, the activations of these three cases are mostly distributed in the unsaturated area (-1, 1) of the activation function tanh, which
% indicates the simpleness and success for
benefits the training. However, Xavier(HOdd) encounters explosion and the activations are distributed mostly in the saturated region around -1 and 1, which shows the limitation of Xavier when applied to HOdd-5. As shown from Figure~\ref{fig:hoddc1-loss} to Figure~\ref{fig:hoddc4-testacc}, Graph(-in/-out) initialization leads to convergence almost at the beginning of the training, while Xavier initialization costs several epochs to converge, and Orthogonal initialization fails to train the network. In addition, when the hyperedge $\varphi$ increases from 1 to 4, Xavier loses the ability to train HOdd-5, indicating that considering the hyperedge is necessary for initialization.

\begin{table*}[t]
\centering
\caption{Top-1 accuracy on Cifar10 and Tiny-ImageNet.
% Rank-Edge means edges connected only with weight vertices. 
Rank-Edge Number means the least number of edges only connected to weight vertices in layers. Random-$\ast$ denotes randomly generating models. More results are in Appendix~\ref{tbl:cifar-tiny-full}.}
\label{tbl:cifar-tiny}
\scalebox{0.9}{
\begin{tabular}{c|c|ccc|ccc}
\hline
        &                                                    & \multicolumn{3}{c|}{Cifar10}                                                        & \multicolumn{3}{c}{Tiny-ImageNet}                                                  \\ \hline
        & \begin{tabular}[c]{@{}c@{}}Rank-Edge\\ Number\end{tabular} & \begin{tabular}[c]{@{}c@{}}Kaiming\\ (-in/-out)\end{tabular} & Graph-in & Graph-out & \begin{tabular}[c]{@{}c@{}}Kaiming\\ (-in/-out)\end{tabular} & Graph-in & Graph-out \\ \hline\hline
Low-Rank & 1                                                  & 0.1                                                          & 0.8141   & 0.8163    & 0.307/0.2776                                                          & 0.3153   & 0.3076    \\
Tensor Ring      & 4                                                  & 0.1                                                          & 0.8308   & 0.8311    & 0.005                                                          & 0.2494   & 0.249    \\
% TT      & 4                                                  & 0.1                                                          & 0.8276   & 0.8341    & 0.005                                                          & -   & 0.8341    \\
% TK2     & 2                                                  & 0.1                                                          & 0.7775   & 0.7709    & 0.2245/0.005                                                          & 0.2183   & 0.2149    \\
% ODD     & 5                                                  & 0.1                                                          & 0.8512   & 0.849     & 0.005                                                          & 0.4742   & 0.4674     \\
\hline
HTK2($\varphi$=4)     & 2                                                  & 0.1                                                          & 0.8638   & 0.8705    & 0.005                                                          & 0.4014   & 0.4126    \\
HOdd($\varphi$=4)     & 14                                                  & 0.1                                                          & 0.8826   & 0.8806     & 0.005         & 0.5048   & 0.5045     \\ \hline
Random-1      & -                                                  & 0.1                                                          & 0.8538  & 0.8483     & 0.005                                                          & 0.4965   & 0.5015     \\
Random-2      & -                                                  & 0.1                                                          & 0.8801   & 0.876     & 0.005                                                          & 0.5379   & 0.5356     \\
Random-3      & -                                                  & 0.1                                                          & 0.8648   & 0.863     & 0.005                                                          & 0.5475   & 0.5403     \\
Random-4      & -                                                  & 0.1                                                          & 0.8789   & 0.8816     & 0.005                                                          & 0.5295   & 0.5306     \\ 
Random-5      & -                                                  & 0.1                                                          & 0.8622   & 0.8644     & 0.005                                                          & 0.5444   & 0.5428     \\  
% Random-6      & -                                                  & 0.1                                                          & 0.8735   & 0.8721     & 0.005                                                          & 0.5452   & 0.5446     \\  
% Random-7      & -                                                  & 0.1                                                          & 0.8601   & 0.8558     & 0.005                                                          & 0.5328   & 0.5394     \\  
% Random-8      & -                                                & 0.1                                                          & 0.8589   & 0.8561     & 0.005                                                          & 0.5269   & 0.5291     \\ 
\hline
\end{tabular}
}
\end{table*}

\begin{figure*}[t]
	\centering
	\subfigure[HRand-RN-50]{
		\includegraphics[width=0.23\textwidth]{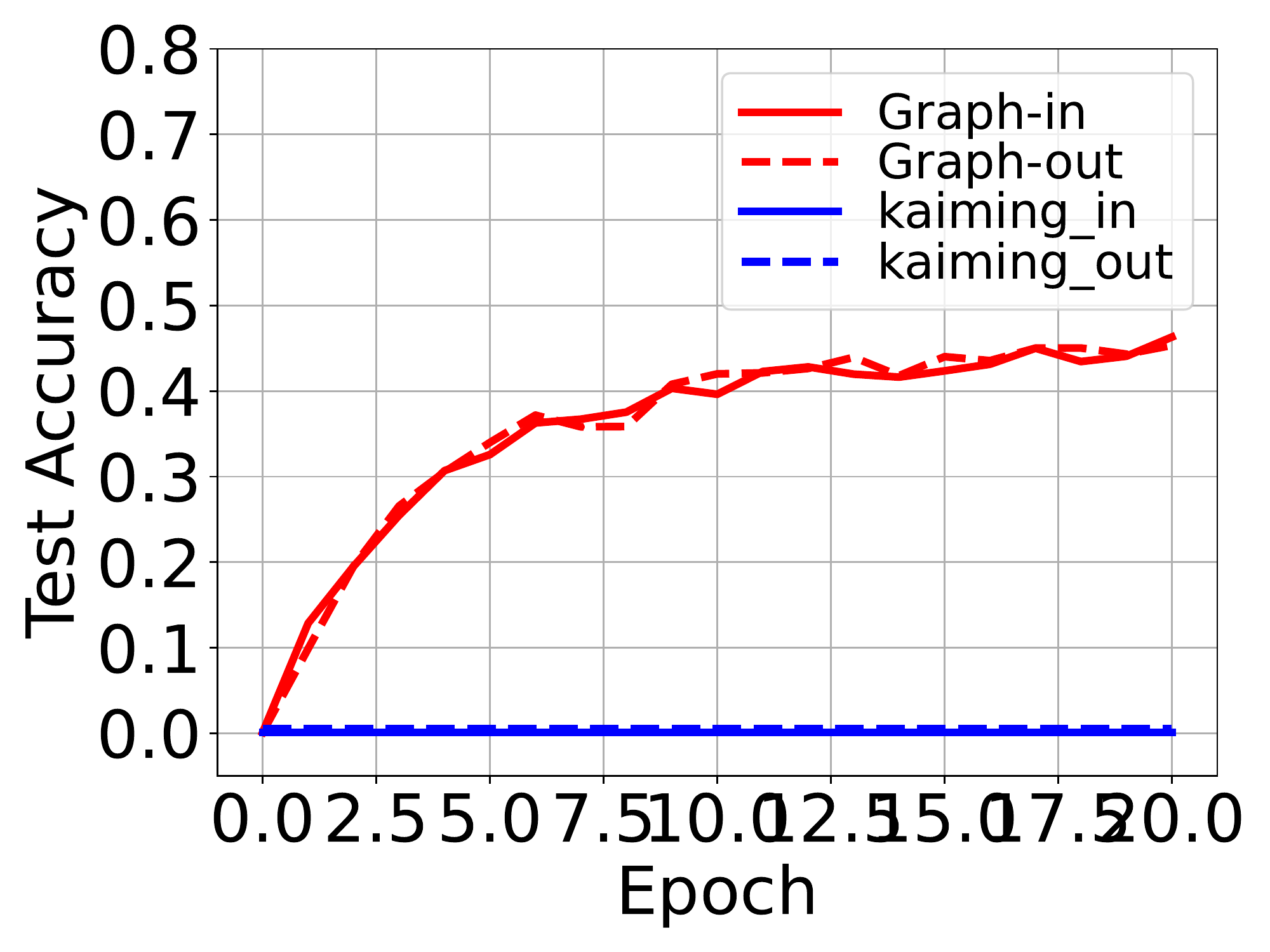}
	} 
	\subfigure[HRand-RN-101]{
		\includegraphics[width=0.23\textwidth]{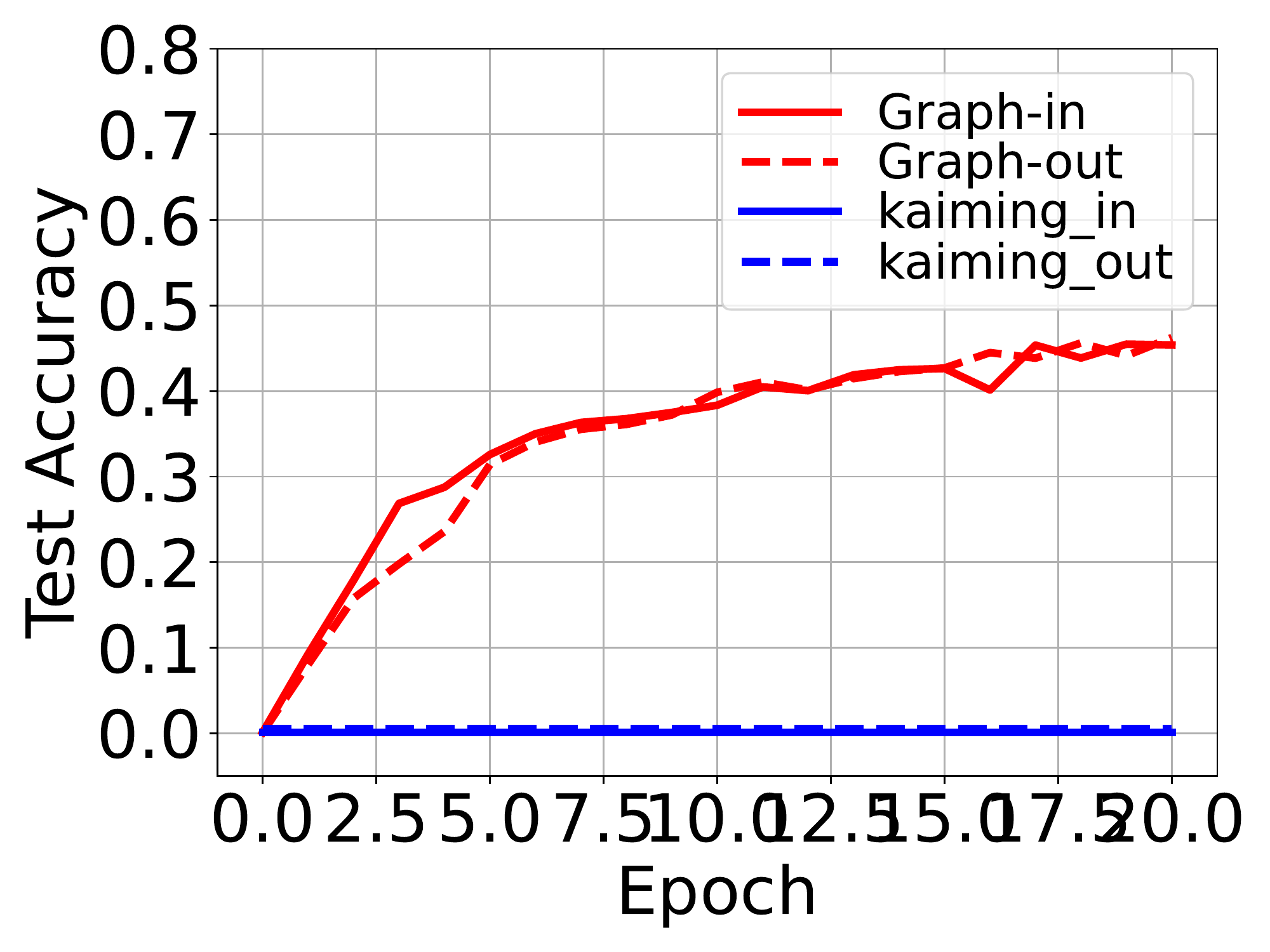}
	}
	\subfigure[HRand-gMLP-S16]{
		\includegraphics[width=0.23\textwidth]{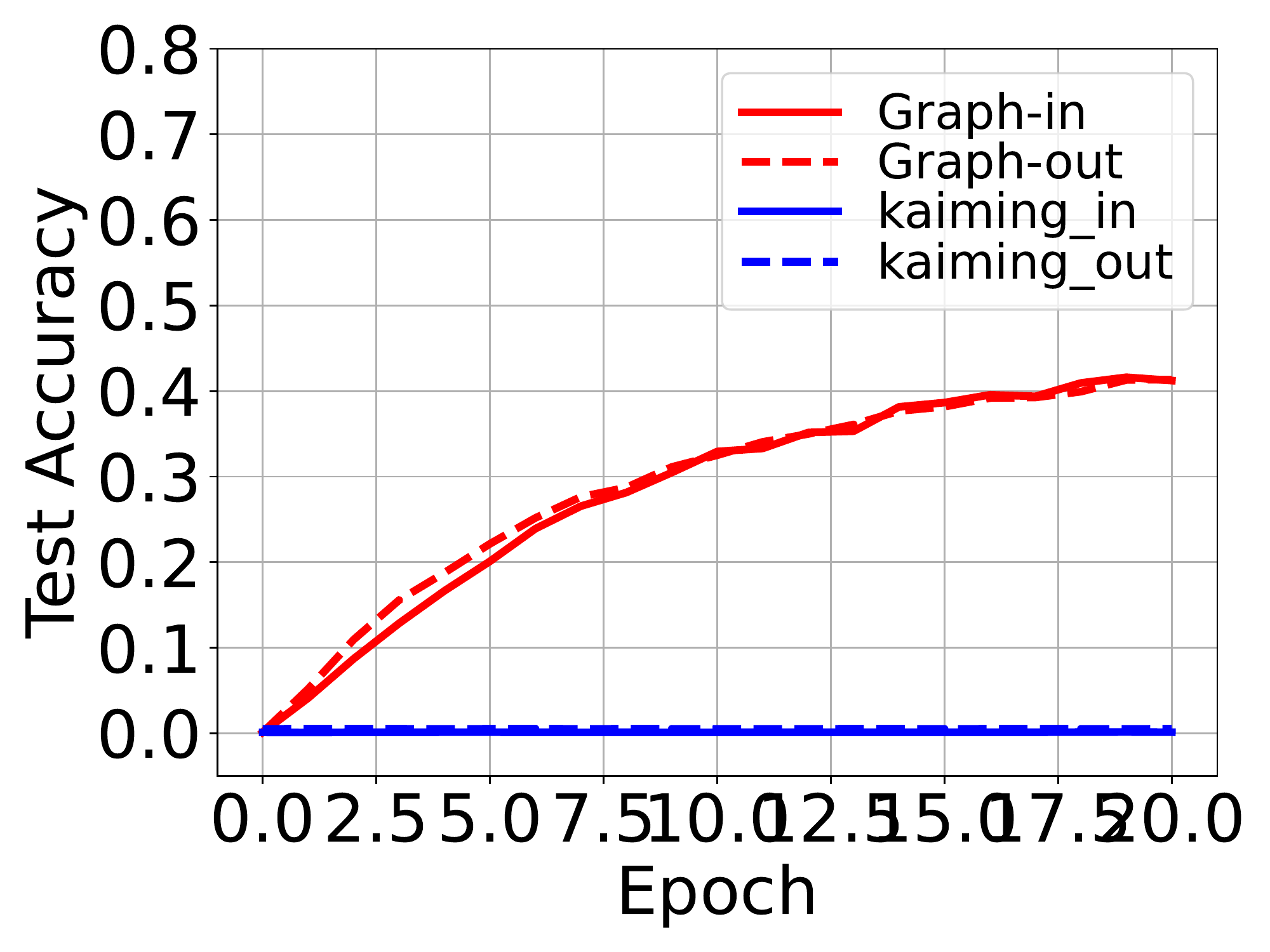}
	} 
	\subfigure[HRand-MLP-Mixer-B16]{
		\includegraphics[width=0.23\textwidth]{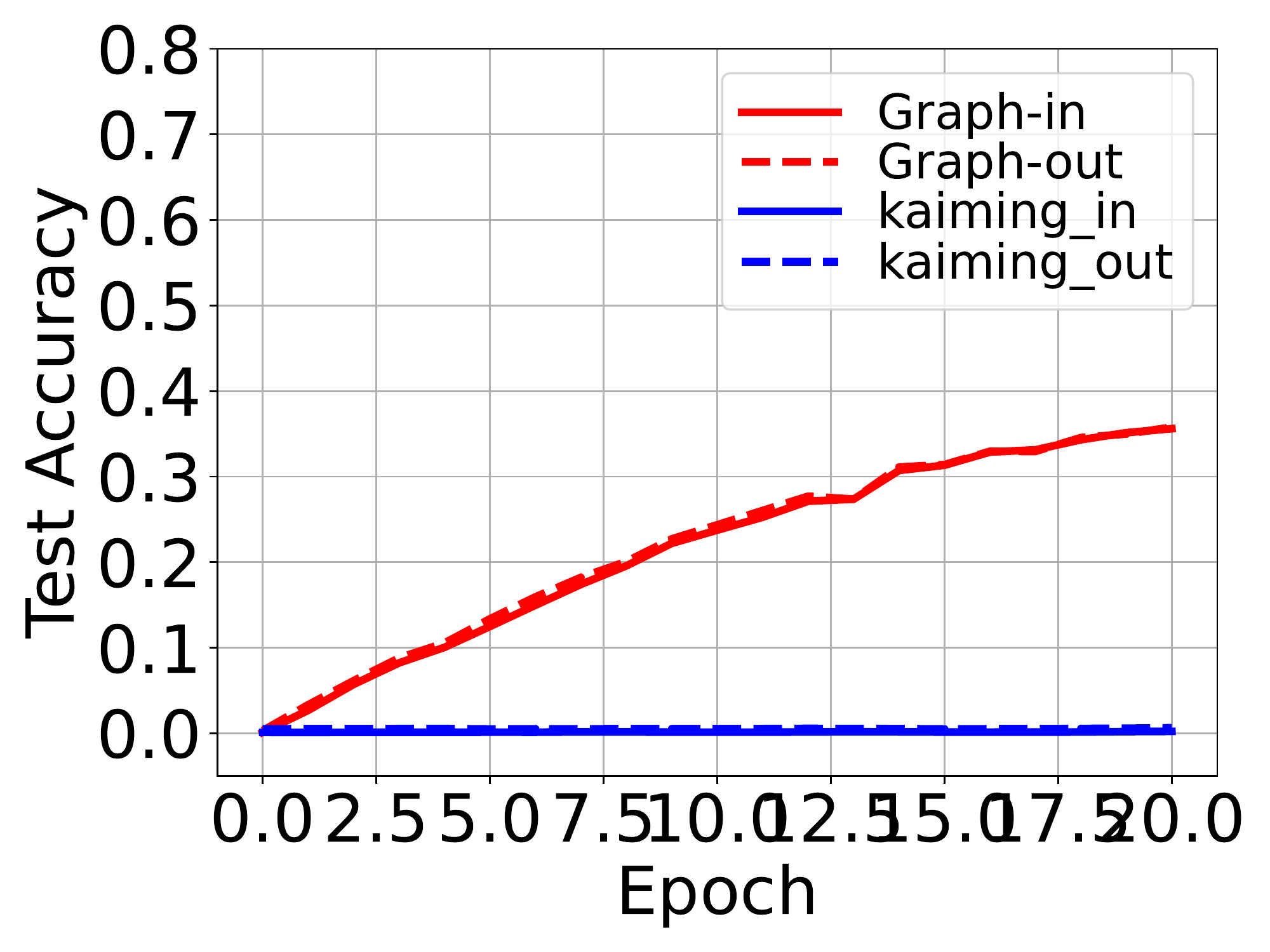}
	}
	\caption{Top-1 accuracy plots on ImageNet. RN is short for ResNet.}
	\label{fig:imagenet101}
\end{figure*}

\subsubsection{Random Tensor Formats}
\label{sec:exrandom}

As a unified paradigm, our graphical initialization has the ability to initialize a variety of TCNNs successfully.
In this section, we present the random layer experiment on MNIST to evaluate the generalization ability.
To be specific, we design a TCNN consisting of four convolutional layers (referred to as Conv-4) and structure of each layer is generated randomly. 
%Conv-4 is uncertain, which means the architecture of each layer is arbitrary.
Examples of the random layer are shown in Figure~\ref{fig:random-layer}. Obviously, these examples are quite different in vertex numbers and edge values. 
In this experiment, we conduct 150 rounds of sub-experiments repeatedly and generate random kernels in every sub-experiment (up to 600 different kernel structures if not considering the circumstance of the same shapes). The activation function here is still a hyperbolic tangent function. Thus we compare four kinds of initialization: Graph(-in/-out), Xavier and Orthogonal. The batch size is 128. The optimizer is Adam with the learning rate 1e-4.

Results are shown in Figure~\ref{fig:conv4training}. As for Xavier initialization, the network is hard to train and the test accuracy remains low. When using Orthogonal initialization, the network becomes trainable and obtains mediocre test results. However, our initialization Graph(-in/out) achieves the highest accuracy and the smallest loss with a faster speed, which shows great advantage when applied to TCNNs. Meanwhile, due to the good performance in these random-structured networks, we believe that the proposed initialization methods will be widely applicable to arbitrary TCNNs.

\subsection{Evaluation on Real-world Datasets}

\subsubsection{Evaluation on Cifar10}
In this experiment, we construct experiments on Cifar10.
Following \citet{DBLP:conf/iclr/ChangFL20}, to clearly showing the importance of weight initialization, we adopt the All Convolutional Net (All-Conv)~\citep{DBLP:journals/corr/SpringenbergDBR14} (containing 9 convolutional layers without Batch Normalization~\citep{DBLP:conf/icml/IoffeS15}) with the SGD optimizer with the learning rate 5e-3, while dropping the normalization tricks and the Adam optimization. 
We set two cases: (i) training tensorial All-Conv with random tensor formats for 30 round, totally $30\times 9=270$ tensor formats; and (ii) training tensorial All-Conv with common tensor formats (e.g., Tensor Ring and HOdd) for 1 round.
% In this experiment, following \citet{DBLP:conf/iclr/ChangFL20}, we adopt the All Convolutional Net (All-Conv)~\citep{DBLP:journals/corr/SpringenbergDBR14}, containing 9 convolutional layers without Batch Normalization~\citep{DBLP:conf/icml/IoffeS15}, and use on Cifar10 with the SGD optimizer (where lr=5e-3), 
% to clearly showing the importance of weight initialization,
% while dropping the normalization tricks and the Adam optimization. 
% We set two cases: (i) training tensorial All-Conv with random tensor formats for 30 round, totally $30\times 9=270$ tensor formats; and (ii) training tensorial All-Conv with common tensor formats (e.g., Tensor Ring) for 1 round.

Results of Case~(i) are shown in Figure~\ref{fig:allconvtraining},  Kaiming(-in/-out) is statistically hard to initialize a random tensorial All-Conv net. By contrast, Graph(-in/-out) performs a flexible adaption in a wide range of tensor formats. We show results of Case~(ii) in Table~\ref{tbl:cifar-tiny}. Graph(-in/-out) works fine from the simplest Low-Rank convolution to the most complex HOdd convolution, which indicates good applicability. However, Kaiming(-in/-out) still fails to initialize a common TCNN, even the Low-Rank convolution. Some comparisons with more common tensor formats (e.g., CP and TT) show the same results as Case~(ii) in Appendix~\ref{sec:formats}.

In this paper, we focus on a unified method that adapts arbitrary TCNNs. Therefore, a comparison with ad-hoc methods are not the point. For completeness, we put such a comparison with the Cifar10 experiment setting in Appendix~\ref{sec:adhoc}, in which our method derives comparably good results.

\subsubsection{Evaluation on Tiny-ImageNet}

Tiny-ImageNet contains a subset of ImageNet's images and 
is a challenging large-scale dataset. As Kaiming initialization is often applied to ResNet, we also evaluate our initialization for tensorial ResNet-50 on Tiny-ImageNet. The optimizer is SGD with the learning rate 1e-1.

As shown in Table~\ref{tbl:cifar-tiny}, 
% the results of the Low-Rank and Odd ResNet are quite different. 
Kaiming initialization almost fails to train
all the listed models except the low-rank one, which can be interpreted by Eq.~\eqref{eq:varequal}, i.e., variance of the output is sensitive to the number of edges.
% with more edges compared with Low-Rank convolutional models, variance of the output will be more sensitive to be influenced.
Kaiming-in seems to work for the low-rank ResNet, probably because the Low-rank convolution has much fewer edges than other TCNNs and is very close to a vanilla CNN. Also, batch normalization is used to aid the training, which is not always applicable in memory-constrained scenarios. Besides Kaiming-out still cannot work even with batch normalization.
Based on the above observations, we can see that although Kaiming initialization may work for some extremely simple TCNNs, it fails for complicated TCNNs. By contrast, our method can fit various TCNNs, including randomly generated architectures.

\subsubsection{Evaluation on ImageNet}
% Tiny ImageNet dataset contains a subset of ImageNet images and this dataset has a total of 200 classes and each class has 500 training and 50 validation 64$\times$64$\times$3 images. Although Tiny-ImageNet is smaller than ImageNet, it is still a challenging large-scale dataset.

In this section, we employ ResNet~\citep{DBLP:conf/cvpr/HeZRS16}, and two recent models gMLP~\citep{DBLP:journals/corr/abs-2105-08050} and MLP-Mixer~\citep{DBLP:journals/corr/abs-2105-01601} for validation of our initialization. Tensorial layers are all random layers (termed as HRand). We adopt these models' original settings that tensorial ResNet uses SGD, and tensorial gMLP/MLP-Mixer uses Adam.

As shown in Figure~\ref{fig:imagenet101}, results are highly consistent, namely, Graph(-in/-out) performs well in all three nets (i.e., ResNet, gMLP and MLP-Mixer), while Kaiming fails in all the situations. This phenomenon is reasonable since that data-flow variance will change exponentially when edge number increases. However, Kaiming ignores the inner production in tensor formats, leading to failure. More similar results are given in Appendix~\ref{sec:detail-imagenet}. Notably, Graph(-in/-out) always derive similar results, which is hard to be explained as claimed in ~\citet{DBLP:conf/iclr/ChangFL20}. In practice, the two modes can be used optionally.

\section{Conclusion}
\label{sec:conlusion}
In this paper, we present Reproducing Transformation to denote backward process as a convolution, and then derive the unified paradigm to help stabilize arbitrary TCNNs. Based on this variance-control principle, the proposed graphical initialization method  can avoid data-flow explosion, and have shown the ability to train diverse TCNNs successfully.
In the future, we plan to explore the application of our method to Tensorial Recurrent Neural Networks.

\section*{Acknowledgments}
This work was partially supported by the National Key Research and Development Program of China (No. 2018AAA0100204), and a key  program of fundamental research from Shenzhen Science and Technology Innovation Commission (No. JCYJ20200109113403826).


\bibliography{ref}

\begin{thebibliography}{30}
\providecommand{\natexlab}[1]{#1}
\providecommand{\url}[1]{\texttt{#1}}
\expandafter\ifx\csname urlstyle\endcsname\relax
  \providecommand{\doi}[1]{doi: #1}\else
  \providecommand{\doi}{doi: \begingroup \urlstyle{rm}\Url}\fi

\bibitem[Chang et~al.(2020)Chang, Flokas, and Lipson]{DBLP:conf/iclr/ChangFL20}
Chang, O., Flokas, L., and Lipson, H.
\newblock Principled weight initialization for hypernetworks.
\newblock In \emph{{ICLR}}. OpenReview.net, 2020.

\bibitem[Elhoushi et~al.(2019)Elhoushi, Tian, Chen, Shafiq, and
  Li]{DBLP:journals/corr/abs-1909-05675}
Elhoushi, M., Tian, Y.~H., Chen, Z., Shafiq, F., and Li, J.~Y.
\newblock Accelerating training using tensor decomposition.
\newblock \emph{CoRR}, abs/1909.05675, 2019.

\bibitem[Gao et~al.(2019)Gao, Cheng, He, Xie, Zhao, Lu, and
  Xiang]{DBLP:journals/corr/abs-1904-06194}
Gao, Z., Cheng, S., He, R., Xie, Z., Zhao, H., Lu, Z., and Xiang, T.
\newblock Compressing deep neural networks by matrix product operators.
\newblock \emph{CoRR}, abs/1904.06194, 2019.

\bibitem[Garipov et~al.(2016)Garipov, Podoprikhin, Novikov, and
  Vetrov]{DBLP:journals/corr/GaripovPNV16}
Garipov, T., Podoprikhin, D., Novikov, A., and Vetrov, D.~P.
\newblock Ultimate tensorization: compressing convolutional and {FC} layers
  alike.
\newblock \emph{CoRR}, abs/1611.03214, 2016.

\bibitem[Glorot \& Bengio(2010)Glorot and Bengio]{DBLP:journals/jmlr/GlorotB10}
Glorot, X. and Bengio, Y.
\newblock Understanding the difficulty of training deep feedforward neural
  networks.
\newblock In \emph{{AISTATS}}, volume~9 of \emph{{JMLR} Proceedings}, pp.\
  249--256. JMLR.org, 2010.

\bibitem[Hayashi et~al.(2019)Hayashi, Yamaguchi, Sugawara, and
  Maeda]{DBLP:conf/nips/HayashiYSM19}
Hayashi, K., Yamaguchi, T., Sugawara, Y., and Maeda, S.
\newblock Exploring unexplored tensor network decompositions for convolutional
  neural networks.
\newblock In \emph{NeurIPS}, pp.\  5553--5563, 2019.

\bibitem[He et~al.(2015)He, Zhang, Ren, and Sun]{DBLP:conf/iccv/HeZRS15}
He, K., Zhang, X., Ren, S., and Sun, J.
\newblock Delving deep into rectifiers: Surpassing human-level performance on
  imagenet classification.
\newblock In \emph{{ICCV}}, pp.\  1026--1034. {IEEE} Computer Society, 2015.

\bibitem[He et~al.(2016)He, Zhang, Ren, and Sun]{DBLP:conf/cvpr/HeZRS16}
He, K., Zhang, X., Ren, S., and Sun, J.
\newblock Deep residual learning for image recognition.
\newblock In \emph{{CVPR}}, pp.\  770--778. {IEEE} Computer Society, 2016.

\bibitem[Idelbayev \& Carreira{-}Perpi{\~{n}}{\'{a}}n(2020)Idelbayev and
  Carreira{-}Perpi{\~{n}}{\'{a}}n]{DBLP:conf/cvpr/IdelbayevC20}
Idelbayev, Y. and Carreira{-}Perpi{\~{n}}{\'{a}}n, M.~{\'{A}}.
\newblock Low-rank compression of neural nets: Learning the rank of each layer.
\newblock In \emph{{CVPR}}, pp.\  8046--8056. {IEEE}, 2020.

\bibitem[Ioffe \& Szegedy(2015)Ioffe and Szegedy]{DBLP:conf/icml/IoffeS15}
Ioffe, S. and Szegedy, C.
\newblock Batch normalization: Accelerating deep network training by reducing
  internal covariate shift.
\newblock In \emph{{ICML}}, volume~37 of \emph{{JMLR} Workshop and Conference
  Proceedings}, pp.\  448--456. JMLR.org, 2015.

\bibitem[Kim et~al.(2016)Kim, Park, Yoo, Choi, Yang, and
  Shin]{DBLP:journals/corr/KimPYCYS15}
Kim, Y., Park, E., Yoo, S., Choi, T., Yang, L., and Shin, D.
\newblock Compression of deep convolutional neural networks for fast and low
  power mobile applications.
\newblock In \emph{{ICLR} (Poster)}, 2016.

\bibitem[Kossaifi et~al.(2020)Kossaifi, Toisoul, Bulat, Panagakis, Hospedales,
  and Pantic]{DBLP:conf/cvpr/KossaifiTBPHP20}
Kossaifi, J., Toisoul, A., Bulat, A., Panagakis, Y., Hospedales, T.~M., and
  Pantic, M.
\newblock Factorized higher-order cnns with an application to spatio-temporal
  emotion estimation.
\newblock In \emph{{CVPR}}, pp.\  6059--6068. {IEEE}, 2020.

\bibitem[Lebedev et~al.(2015)Lebedev, Ganin, Rakhuba, Oseledets, and
  Lempitsky]{DBLP:journals/corr/LebedevGROL14}
Lebedev, V., Ganin, Y., Rakhuba, M., Oseledets, I.~V., and Lempitsky, V.~S.
\newblock Speeding-up convolutional neural networks using fine-tuned
  cp-decomposition.
\newblock In \emph{{ICLR} (Poster)}, 2015.

\bibitem[Li \& Sun(2020)Li and Sun]{DBLP:conf/icml/LiS20}
Li, C. and Sun, Z.
\newblock Evolutionary topology search for tensor network decomposition.
\newblock In \emph{{ICML}}, volume 119 of \emph{Proceedings of Machine Learning
  Research}, pp.\  5947--5957. {PMLR}, 2020.

\bibitem[Li et~al.(2021)Li, Pan, Chen, Ding, Zhao, and Xu]{li2021heuristic}
Li, N., Pan, Y., Chen, Y., Ding, Z., Zhao, D., and Xu, Z.
\newblock Heuristic rank selection with progressively searching tensor ring
  network.
\newblock \emph{Complex \& Intelligent Systems}, pp.\  1--15, 2021.

\bibitem[Liu et~al.(2021)Liu, Dai, So, and
  Le]{DBLP:journals/corr/abs-2105-08050}
Liu, H., Dai, Z., So, D.~R., and Le, Q.~V.
\newblock Pay attention to mlps.
\newblock \emph{CoRR}, abs/2105.08050, 2021.

\bibitem[Novikov et~al.(2015)Novikov, Podoprikhin, Osokin, and
  Vetrov]{DBLP:conf/nips/NovikovPOV15}
Novikov, A., Podoprikhin, D., Osokin, A., and Vetrov, D.~P.
\newblock Tensorizing neural networks.
\newblock In \emph{{NIPS}}, pp.\  442--450, 2015.

\bibitem[Pan et~al.(2019)Pan, Xu, Wang, Ye, Wang, Bai, and
  Xu]{DBLP:conf/aaai/PanXWYWBX19}
Pan, Y., Xu, J., Wang, M., Ye, J., Wang, F., Bai, K., and Xu, Z.
\newblock Compressing recurrent neural networks with tensor ring for action
  recognition.
\newblock In \emph{{AAAI}}, pp.\  4683--4690. {AAAI} Press, 2019.

\bibitem[Pan et~al.(2022)Pan, Wang, and Xu]{DBLP:journals/ijon/PanWX22}
Pan, Y., Wang, M., and Xu, Z.
\newblock Tednet: {A} pytorch toolkit for tensor decomposition networks.
\newblock \emph{Neurocomputing}, 469:\penalty0 234--238, 2022.

\bibitem[Paszke et~al.(2019)Paszke, Gross, Massa, Lerer, Bradbury, Chanan,
  Killeen, Lin, Gimelshein, Antiga, Desmaison, K{\"{o}}pf, Yang, DeVito,
  Raison, Tejani, Chilamkurthy, Steiner, Fang, Bai, and
  Chintala]{DBLP:conf/nips/PaszkeGMLBCKLGA19}
Paszke, A., Gross, S., Massa, F., Lerer, A., Bradbury, J., Chanan, G., Killeen,
  T., Lin, Z., Gimelshein, N., Antiga, L., Desmaison, A., K{\"{o}}pf, A., Yang,
  E., DeVito, Z., Raison, M., Tejani, A., Chilamkurthy, S., Steiner, B., Fang,
  L., Bai, J., and Chintala, S.
\newblock Pytorch: An imperative style, high-performance deep learning library.
\newblock In \emph{NeurIPS}, pp.\  8024--8035, 2019.

\bibitem[Rigamonti et~al.(2013)Rigamonti, Sironi, Lepetit, and
  Fua]{DBLP:conf/cvpr/RigamontiSLF13}
Rigamonti, R., Sironi, A., Lepetit, V., and Fua, P.
\newblock Learning separable filters.
\newblock In \emph{{CVPR}}, pp.\  2754--2761. {IEEE} Computer Society, 2013.

\bibitem[Springenberg et~al.(2015)Springenberg, Dosovitskiy, Brox, and
  Riedmiller]{DBLP:journals/corr/SpringenbergDBR14}
Springenberg, J.~T., Dosovitskiy, A., Brox, T., and Riedmiller, M.~A.
\newblock Striving for simplicity: The all convolutional net.
\newblock In \emph{{ICLR} (Workshop)}, 2015.

\bibitem[Szegedy et~al.(2016)Szegedy, Vanhoucke, Ioffe, Shlens, and
  Wojna]{DBLP:conf/cvpr/SzegedyVISW16}
Szegedy, C., Vanhoucke, V., Ioffe, S., Shlens, J., and Wojna, Z.
\newblock Rethinking the inception architecture for computer vision.
\newblock In \emph{{CVPR}}, pp.\  2818--2826. {IEEE} Computer Society, 2016.

\bibitem[Taki(2017)]{DBLP:journals/corr/abs-1709-02956}
Taki, M.
\newblock Deep residual networks and weight initialization.
\newblock \emph{CoRR}, abs/1709.02956, 2017.

\bibitem[Tolstikhin et~al.(2021)Tolstikhin, Houlsby, Kolesnikov, Beyer, Zhai,
  Unterthiner, Yung, Steiner, Keysers, Uszkoreit, Lucic, and
  Dosovitskiy]{DBLP:journals/corr/abs-2105-01601}
Tolstikhin, I.~O., Houlsby, N., Kolesnikov, A., Beyer, L., Zhai, X.,
  Unterthiner, T., Yung, J., Steiner, A., Keysers, D., Uszkoreit, J., Lucic,
  M., and Dosovitskiy, A.
\newblock Mlp-mixer: An all-mlp architecture for vision.
\newblock \emph{CoRR}, abs/2105.01601, 2021.

\bibitem[Wang et~al.(2020)Wang, Su, Luo, Pan, Zheng, and
  Xu]{DBLP:conf/iconip/WangSL0ZX20}
Wang, M., Su, Z., Luo, X., Pan, Y., Zheng, S., and Xu, Z.
\newblock Concatenated tensor networks for deep multi-task learning.
\newblock In \emph{{ICONIP} {(5)}}, volume 1333 of \emph{Communications in
  Computer and Information Science}, pp.\  517--525. Springer, 2020.

\bibitem[Wang et~al.(2018)Wang, Sun, Eriksson, Wang, and
  Aggarwal]{DBLP:conf/cvpr/WangSEWA18}
Wang, W., Sun, Y., Eriksson, B., Wang, W., and Aggarwal, V.
\newblock Wide compression: Tensor ring nets.
\newblock In \emph{{CVPR}}, pp.\  9329--9338. {IEEE} Computer Society, 2018.

\bibitem[Ye et~al.(2018)Ye, Wang, Li, Chen, Zhe, Chu, and
  Xu]{DBLP:conf/cvpr/YeWLCZCX18}
Ye, J., Wang, L., Li, G., Chen, D., Zhe, S., Chu, X., and Xu, Z.
\newblock Learning compact recurrent neural networks with block-term tensor
  decomposition.
\newblock In \emph{{CVPR}}, pp.\  9378--9387. {IEEE} Computer Society, 2018.

\bibitem[Ye et~al.(2020)Ye, Li, Chen, Yang, Zhe, and
  Xu]{DBLP:journals/nn/YeLCYZX20}
Ye, J., Li, G., Chen, D., Yang, H., Zhe, S., and Xu, Z.
\newblock Block-term tensor neural networks.
\newblock \emph{Neural Networks}, 130:\penalty0 11--21, 2020.

\bibitem[Yin et~al.(2021)Yin, Sui, Liao, and Yuan]{DBLP:conf/cvpr/YinSL021}
Yin, M., Sui, Y., Liao, S., and Yuan, B.
\newblock Towards efficient tensor decomposition-based {DNN} model compression
  with optimization framework.
\newblock In \emph{{CVPR}}, pp.\  10674--10683. Computer Vision Foundation /
  {IEEE}, 2021.

\end{thebibliography}
\bibliographystyle{icml2022}

%%%%%%%%%%%%%%%%%%%%%%%%%%%%%%%%%%%%%%%%%%%%%%%%%%%%%%%%%%%%%%%%%%%%%%%%%%%%%%%
%%%%%%%%%%%%%%%%%%%%%%%%%%%%%%%%%%%%%%%%%%%%%%%%%%%%%%%%%%%%%%%%%%%%%%%%%%%%%%%
% APPENDIX
%%%%%%%%%%%%%%%%%%%%%%%%%%%%%%%%%%%%%%%%%%%%%%%%%%%%%%%%%%%%%%%%%%%%%%%%%%%%%%%
%%%%%%%%%%%%%%%%%%%%%%%%%%%%%%%%%%%%%%%%%%%%%%%%%%%%%%%%%%%%%%%%%%%%%%%%%%%%%%%
\newpage
\appendix
\onecolumn

\section{Proof of Theorem~\ref{thm:backdummy}}
\label{sec:profthmdummy}

\begin{proof}
As defined in Section~\ref{sec:hypergraph}, the forward dummy $\ca{P} \in \mathbb{R}^{\alpha\times \alpha'\times \beta}$ is defined as $\ca{P}_{j, j', k} = 1$, if $j = sj' + k -p$ and otherwise $\ca{P}_{j, j', k} = 0$.
Then, $j \in \{0, 1, \dots \alpha - 1\}$, $j' \in \{0, 1, \dots, \alpha'-1\}$ and $k\in \{0, 1, \dots, \beta-1\}$. $s$ means the stride and $p$ means the padding.

The backward dummy $\ca{P}' \in \mathbb{R}^{\alpha\times \tilde{\alpha}'\times \beta}$ is expected as ${\ca{P}'}_{j, \tilde{j}', k} = 1$, if $\tilde{j}' = \tilde{s}j + \tilde{k} - \tilde{p}$ and otherwise ${\ca{P}'}_{j, \tilde{j}', k} = 0$, where $\tilde{s}=1$ , $\tilde{k} = \beta-k-1$ and  $\tilde{p}=\beta-p-1$.

The transformation matrix $\mathbf{T} \in \mathbb{R}^{\alpha' \times \tilde{\alpha}'}$ has $\tilde{\alpha}' = s\alpha'- s + 1$. The expansion of $\Delta \mathbf{y} \in \mathbb{R}^{\alpha'}$ into $ \Delta \tilde{\mathbf{y}} \in \mathbb{R}^{\alpha'}$  can be represented as $\Delta \tilde{\mathbf{y}} = \Delta \mathbf{y}\mathbf{T}$.
Here, we give an example of $s=2$
\begin{align}
\left[\begin{array}{lll}
\Delta\mathbf{y}_{0} & \Delta\mathbf{y}_{1} & \Delta\mathbf{y}_{2}
\end{array}\right] \cdot\left[\begin{array}{lllll}
1 & 0 & 0 & 0 & 0 \\
0 & 0 & 1 & 0 & 0 \\
0 & 0 & 0 & 0 & 1
\end{array}\right]=\left[\begin{array}{lllll}
\Delta\mathbf{y}_{0} & 0 & \Delta\mathbf{y}_{1} & 0 & \Delta\mathbf{y}_{2} 
\end{array}\right].
\end{align}
Obviously, $\Delta \mathbf{y}_{j} = \Delta \tilde{\mathbf{y}}_{\tilde{j}} = \Delta \tilde{\mathbf{y}}_{2j}$.

A reversal matrix $\mathbf{R} \in \mathbb{R}^{\beta \times \beta}$ is an anti-diagonal matrix, where $\mathbf{R}_{ij}=1$ when $i+j =\beta-1$ and $\mathbf{R}_{ij}=0$ for other situation. Reverse $\mathbf{b}\in \mathbb{R}^{\beta}$ into $\tilde{\mathbf{b}}=\mathbf{R}\mathbf{b} \in \mathbb{R}^{\beta}$. Here is an example when $\beta = 3$
$$
\left[\begin{array}{lll}
\mathbf{b}_{0} & \mathbf{b}_{1} & \mathbf{b}_{2}
\end{array}\right] \cdot \left[\begin{array}{lll}
0 & 0 & 1 \\
0 & 1 & 0 \\
1 & 0 & 0
\end{array}\right]=\left[\begin{array}{lll}
\mathbf{b}_{2} & \mathbf{b}_{1} & \mathbf{b}_{0}
\end{array}\right].
$$

Demonstrating $\Delta \mathbf{a}=\bm{dm}(\Delta \tilde{\mathbf{y}}, \tilde{\mathbf{b}})$ is equal to prove 
\begin{align}
    \ca{P} = \ca{P}' \times_1^1 \mathbf{T} \times_1^0 \mathbf{R},
\end{align}
namely the equvilent replacement in Figure~\ref{fig:replace}.

Contracting $\ca{P}'$ with $\mathbf{T}$
\begin{align}
    \hat{\ca{P}}' = \ca{P}' \times_1^1 \mathbf{T},
\end{align}
where $\hat{\ca{P}}' \in \mathbb{R}^{\alpha\times \beta \times \alpha'}$. According to the definition of the transformation matrix, $\tilde{j}' = sj'$, where $j'$ is the index of $\hat{\ca{P}}'$ dimension $\alpha'$.
% into a new $\tilde{j}$ and it means the product of transformation  matrix into the corresponding modes. 
Therefore, a element $\hat{\ca{P}}'_{j, k, j'}$ follows
\begin{align}
sj'=\tilde{s}j + \tilde{k} - \tilde{p}.
\end{align}

Contracting $\hat{\ca{P}}'$ with $\mathbf{R}$
\begin{align}
    \hat{\ca{P}} = \hat{\ca{P}}' \times_1^0 \mathbf{R},
\end{align}
where $\hat{\ca{P}} \in \mathbb{R}^{\alpha\times \alpha'\times \beta}$. Interacting with the reversal matrix, the index $\hat{k}$ of $\hat{\ca{P}}$ dimension $\beta$ follows $\hat{k} = \beta - \tilde{k} - 1$. Thus, $\hat{k} \equiv k$. Then, $\hat{\ca{P}}_{j, j', \hat{k}} = 1$, if $sj'=\tilde{s} j+ \beta-1-\hat{k}-\tilde{p}$ can also be formulated as $\hat{\ca{P}}_{j, j', k} = 1$, if
\begin{align}
sj'=\tilde{s}{j}-k+\beta-\tilde{p}-1.
\end{align}
By replacing $\tilde{s}$ and $\tilde{p}$ with $1$ and $\beta - p- 1$, we obtain
\begin{align}
    sj'={j}-k+p,
\end{align}
namely
\begin{align}
    {j}=sj'+k-p,
\end{align}
which indicates that $\hat{\ca{P}} \equiv \ca{P}$.
In conclusion, we demonstrate
\begin{align}
    \ca{P} = \ca{P}' \times_1^1 \mathbf{T} \times_1^0 \mathbf{R}.
\end{align}
\end{proof}

\section{Proof of Proposition~\ref{pps:tensorsum}}
\label{sec:pro-corosum}
\begin{lemma}
    \label{lemma:sum}
    Assuming that element in both vector $\mathbf{a} \in \mathbb{R}^n$ and $\mathbf{b} \in \mathbb{R}^n$ is independent to other elements in the same vector, the variance of vector $\mathbf{c} = \mathbf{a} + \mathbf{b}$ is ${\sigma^2}{(\mathbf{c})} = {\sigma^2}{(\mathbf{a})} + {\sigma^2}{(\mathbf{b})}$.
\end{lemma}

\begin{proof}
    \begin{align*}
        {\sigma^2}{(\ca{Z})} &= {\sigma^2}{({\bm{vec}}(\ca{Z}))} \\
        &= {\sigma^2}{({\bm{vec}}(\ca{X} + \ca{Y}))} \\
        &= {\sigma^2}{({\bm{vec}}(\ca{X}) + {\bm{vec}}(\ca{Y}))},\\
    \end{align*}
    where ${\bm{vec}}(\cdot)$ represents the vectorization operation. According to Lemma~\ref{lemma:sum}, we get:
    \begin{align*}
        {\sigma^2}{(\ca{Z})} &= {\sigma^2}{({\bm{vec}}(\ca{X}))} + {\sigma^2}{({\bm{vec}}(\ca{Y}))}\\
        &= {\sigma^2}{(\ca{X})} + {\sigma^2}{(\ca{Y})}.
    \end{align*}
\end{proof}

\section{Proof of Proposition~\ref{pps:tensorprod}}
\label{sec:pro-coroprod}
\begin{lemma}
    \label{lemma:prod}
    Assume that elements of a vector $\mathbf{a} \in \mathbb{R}^n$ is i.i.d.. Meanwhile, a vector $\mathbf{b} \in \mathbb{R}^n$ (also i.i.d.) follows a symmetrical distribution with  zero mean. Then the variance of $c = \mathbf{a} \odot \mathbf{b}$ is $\sigma^2{(c)} = n\sigma^2{(\mathbf{a})}\sigma^2{(\mathbf{b})}$, where $\odot$ denotes the inner-product, and ${\sigma^2}(\cdot)$ denotes the variance.
\end{lemma}

\begin{proof}
    Let $\ca{X}^{({\mathbf{x}}_{\ast})} \in \mathbb{R}^{{\mathbf{i}}_{{\mathbf{x}}_0} \times {\mathbf{i}}_{{\mathbf{x}}_1} \times \dots \times {\mathbf{i}}_{{\mathbf{x}}_{d-1}}}$ and $\ca{Y}^{({\mathbf{y}}_{\ast})} \in \mathbb{R}^{{\mathbf{j}}_{{\mathbf{y}}_0} \times {\mathbf{j}}_{{\mathbf{y}}_1} \times \dots \times {\mathbf{j}}_{{\mathbf{y}}_{d-1}}}$ denote sub-tensors where subscripts ${\mathbf{x}}_{\ast}$ and ${\mathbf{y}}_{\ast}$ indicate the fixed dimensions. An element of $\ca{Z}$ is calculated as $\ca{Z}_{{\mathbf{x}}_{\ast}, {\mathbf{y}}_{\ast}} = \ca{X}^{({\mathbf{x}}_{\ast})} \odot \ca{Y}^{({\mathbf{y}}_{\ast})}$. The change 
    of variance derivation is shown as following
    \begin{align*}
        {\sigma^2}{(\ca{Z})} &= {\sigma^2}{(\ca{Z}_{{\mathbf{x}}_{\ast}, {\mathbf{y}}_{\ast}})} \\
        &= {\sigma^2}{(\ca{X}^{({\mathbf{x}}_{\ast})} \odot \ca{Y}^{({\mathbf{y}}_{\ast})})} \\
        &= {\sigma^2}{({\bm{vec}}(\ca{X}^{({\mathbf{x}}_{\ast})}) \odot {\bm{vec}}(\ca{Y}^{({\mathbf{y}}_{\ast})}))},
    \end{align*}
    where ${\bm{vec}}(\cdot)$ represents the vectorization operation. According to Lemma~\ref{lemma:prod}, we get
    \begin{align*}
        {\sigma^2}{(\ca{Z})} &= {\sigma^2}({\bm{vec}}(\ca{X}^{({\mathbf{x}}_{\ast})}))  {\sigma^2}({\bm{vec}}(\ca{Y}^{({\mathbf{y}}_{\ast})})) \prod_{t=0}^{d-1}{\mathbf{v}_t}   \\
        &= {\sigma^2}{(\ca{X})}  {\sigma^2}{(\ca{Y})} \prod_{t=0}^{d-1}{\mathbf{v}_t}.
    \end{align*}
% $\endproof$
\end{proof}

\section{Proof of Theorem~\ref{thm:total}}
\label{sec:prototal}
\begin{proof}
First of all,
\begin{align}
    \sigma^2(\ca{Y}) &= \sigma^2(\ca{Y}_{\ast})=\sigma^2(BG(\mathbf{E})) \notag \\
    &= \sigma^2(\ca{X}\times_{\ast}^0 \ca{W}^{(0)}\times_{\ast}^{0, 1}\ca{W}^{(1)}\dots\times_{\ast}^{0, 1, \dots, n-1}\ca{W}^{(n-1)}).
\end{align}
According to Proposition~\ref{pps:tensorprod},
\begin{align}
    \sigma^2(\ca{Y}) &= \sigma^2(\ca{X}\times_{\ast}^0 \ca{W}^{(0)}\times_{\ast}^{0, 1}\ca{W}^{(1)}\dots)\sigma^2(\ca{W}^{(n-1)}) \prod_{i=0}^{n-1}\prod_{j=\tau-1}^{\tau-1}e_{ij} \notag \\
    &= {\sigma^2}(\ca{X})\prod_{k=0}^{n-1}{{\sigma^2}(\ca{W}^{(k)})}\prod_{i=0}^{n-1}\prod_{j=i+1}^{\tau-1}{e_{ij}}.
\end{align}
% $\endproofwhite$
\end{proof}

\section{Tensor Basis}
\label{sec:append-tensor-basis}
\subsection{Tenosr Notation}
In this paper, a $d$th-order tensor $\pmb{\mathcal{X}}\in\mathbb{R}^{\mathbf{i}_0\times \mathbf{i}_1 \dots \times \mathbf{i}_{d-1}} $ is represented by a boldface Euler script letter. We denote a scalar $s \in \mathbb{R}^1$ by a lowercase letter, a vector $\mathbf{v} \in \mathbb{R}^i$ by a bold lowercase letter and a matrix $\mathbf{M} \in \mathbb{R}^{\mathbf{i}_0 \times \mathbf{i}_1}$ by a bold uppercase letter.

\begin{figure*}[t]
	\centering
	\includegraphics[width=0.98\textwidth]{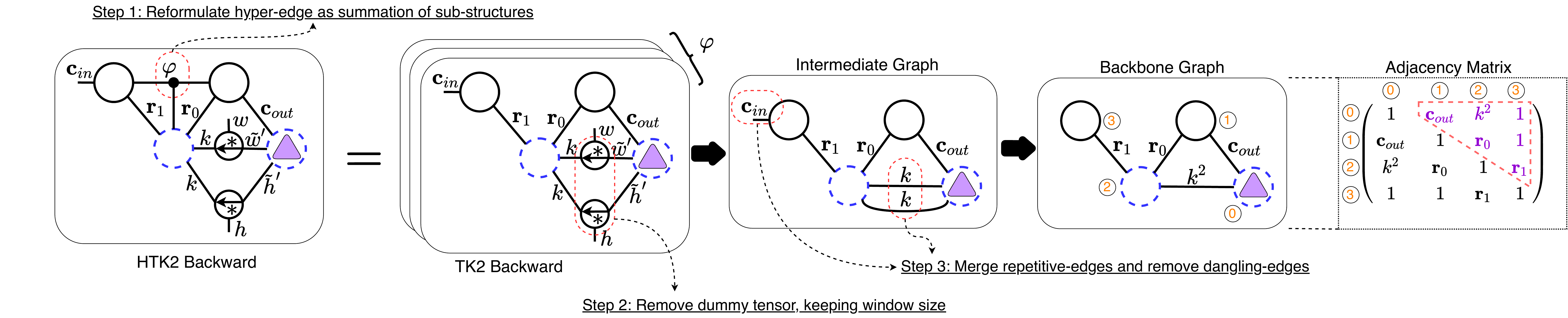}
	\caption{An example of deriving graph-out initialization for Hyper Tucker-2 (HTK2) convolution with the Backbone Graph. By unifying both the forward and backward processes, we can analyze backward propagation similarly to the forward. Then according to the adjacent matrix of Backbone Graph, the initial variance of convolutional weights can be derived as $\frac{1}{\sqrt[3]{\bm{p}_{\bm{a}}\varphi k^2 {\mathbf{c}}_{out}{\mathbf{r}}_{0}{\mathbf{r}}_{1}}}$.}
	\label{fig:backward-init}
\end{figure*}

\subsection{Dummy Tensor and Hyperedge}
\label{sec:dummy-hyper}
Traditional tensor diagram can only describe a tensor structure, which is limited to use in TCNNs. Therefore, to reinforce the representing ability of classical tensor diagram, \citet{DBLP:conf/nips/HayashiYSM19} propose a hypergraph to denote forward computation of convolutional layers by designing the dummy tensor and hyperedge.

\textbf{Dummy Tensor}~~As depicted in Figure~\ref{fig:hyper-dummy},
a vertex with a star symbol denotes a dummy tensor formulated as
\begin{equation}
\label{eq:dummy}
\mathbf{y}_{j^{'}} = \sum_{j=0}^{\alpha-1}\sum_{k=0}^{\beta-1} \ca{P}_{j, j^{'}, k}\mathbf{a}_{j} \mathbf{b}_{k},
\end{equation}
where $\mathbf{a}\in\mathbb{R}^{\alpha}$,  $\mathbf{y}\in\mathbb{R}^{\alpha '}$, $\mathbf{b}\in\mathbb{R}^{\beta}$, $\bm{dm}$ denotes the operation of dummy tensor, and $\ca{P} \in \left\{ 0, 1 \right\}^{\alpha \times {\alpha '} \times \beta}$ is a binary tensor with elements defined as $\ca{P}_{j,j^{'},k}=1$ if $j = sj' + k - p$ where $s$ and $p$ represent stride and padding operation respectively, and $0$ otherwise.

\textbf{Hyperedge}~~As illustrated in Figure~\ref{fig:hyper-edge}, an output of a special case, connecting three vectors through a hyperedge, can be calculated
\begin{equation}
\label{eq:hyper}
y = \sum_{k=0}^{\varphi-1} \mathbf{a}_{k}\mathbf{b}_{k} \mathbf{c}_{k},
\end{equation}
where $\mathbf{a}, \mathbf{b}, \mathbf{c}\in\mathbb{R}^{\varphi}$,  $y \in\mathbb{R}^{1}$, $\bm{he}$ denotes the operation of a hyperedge. Interestingly, a hyperedge is simply equal to a tensor, whose diagonal elements are 1. This tensor indicates the adding operation over several sub-structures (e.g., CNNs). For such an adding composite structure, the key point is the initialization problem of its sub-structures, which is what we focus on in this paper. \citet{DBLP:conf/nips/HayashiYSM19} show that tensor graph can represent arbitrary TCNNs by introducing dummy tensors and hyperedges.

\section{Model Extension}
\subsection{Graph-out Derivation for HTK2 Convolution}
\label{sec:backinit}
We draw Graph-out Derivation for HTK2 Convolution in Figure~\ref{fig:backward-init}.

\subsection{Example for Inter Interaction in a Tensor Format}
\label{sec:variance-example}
As the most widely used methods for CNNs,  Xavier and Kaiming initialization usually perform stable and efficient weight generation. However, they usually fail to produce appropriate weights when applied to TCNNs.
Without loss of generality, we show an instance by using Xavier to initialize the HTK2 convolution. By generating each HTK2 node $\ca{W}^{(k)}$ with the fan-in variance $\frac{1}{k^2\mathbf{c}_{in}}$, we can calculate the final output variance ${\sigma^2}(\ca{Y}_o)$ according to Eq.~\eqref{eq:equal-relation} as
\begin{align}
    \label{eq:xavier-fail}
     {\sigma^2}(\ca{Y}_o) = \bm{p}_{\bm{a}}\varphi{\sigma^2}(\ca{X})\prod_{k=0}^{n-1}{{\sigma^2}(\ca{W}^{(k)})}\prod_{i=0}^{n-1}\prod_{j=i+1}^{\tau-1}{e_{ij}}
     = \varphi k^2\mathbf{c}_{in}\mathbf{r}_0\mathbf{r}_1{\sigma^2}(\ca{X})\frac{1}{(k^2\mathbf{c}_{in})^3} = \frac{\varphi\mathbf{r}_0\mathbf{r}_1}{(k^2\mathbf{c}_{in})^2}{\sigma^2}(\ca{X}),
\end{align}
where $\mathbf{r}_0$ and $\mathbf{r}_1$ denote tensor ranks, and $\varphi$ means the value of a hyperedge. Obviously, through the output variance scale $\frac{\varphi\mathbf{r}_0\mathbf{r}_1}{(k^2\mathbf{c}_{in})^2}$, when $\mathbf{r}_{\ast}$ and $\varphi$ increase, data-flow will explode easily and cause fatal errors in the training process. However, Xavier and Kaiming initialization only consider $\mathbf{c}_{in}$ and $k$, leading to fail to train HTK2 convolution. Such problems can also be found in the fan-out mode. To resolve the dilemma, we derive a principle (i.e., Eq.~\eqref{eq:result}) that generalizes over Xavier to fit multi-node TCNNs, not limited to single-node ones (i.e., Vanilla CNNs).

\begin{figure*}[t]
	\centering
	\subfigure[Low-Rank Convolution]{
		\includegraphics[height=0.12\textheight]{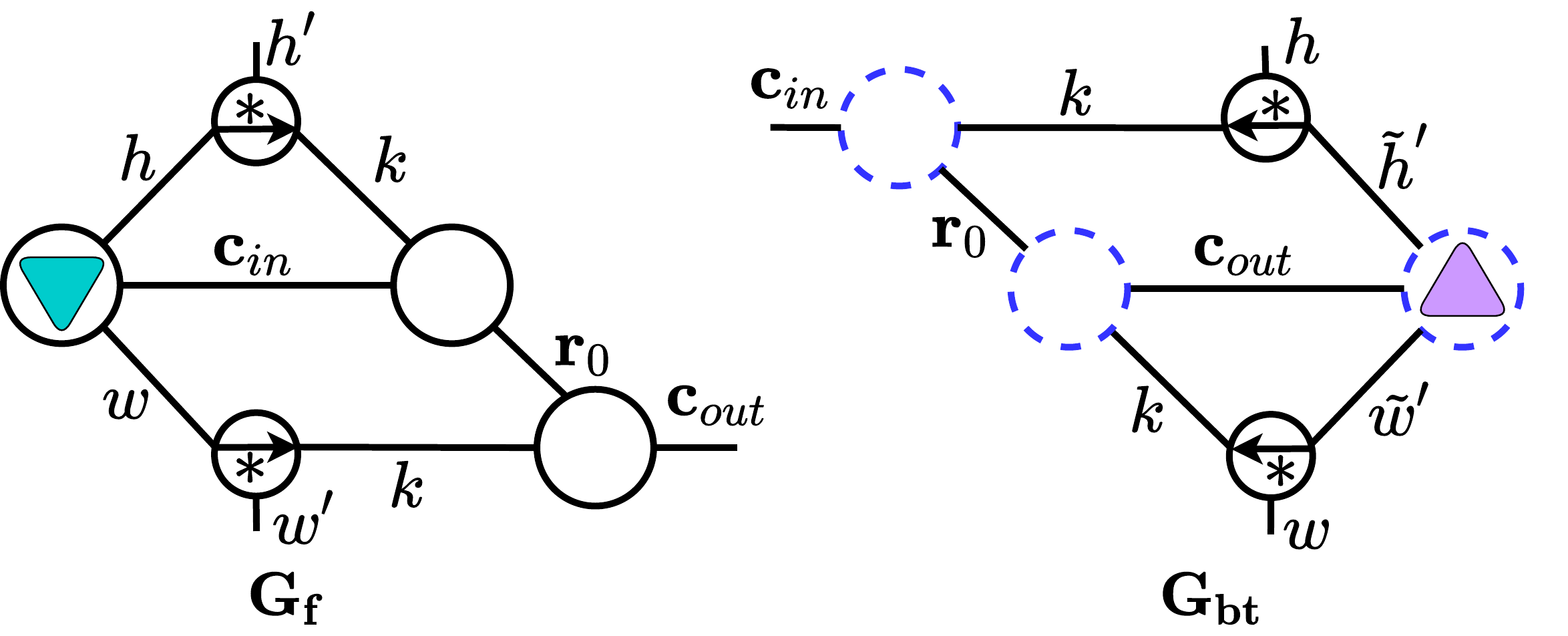}
		\label{fig:lowrankfb}
	}~~~~~~~~~~
	\subfigure[Tensor Train Convolution]{
		\includegraphics[height=0.12\textheight]{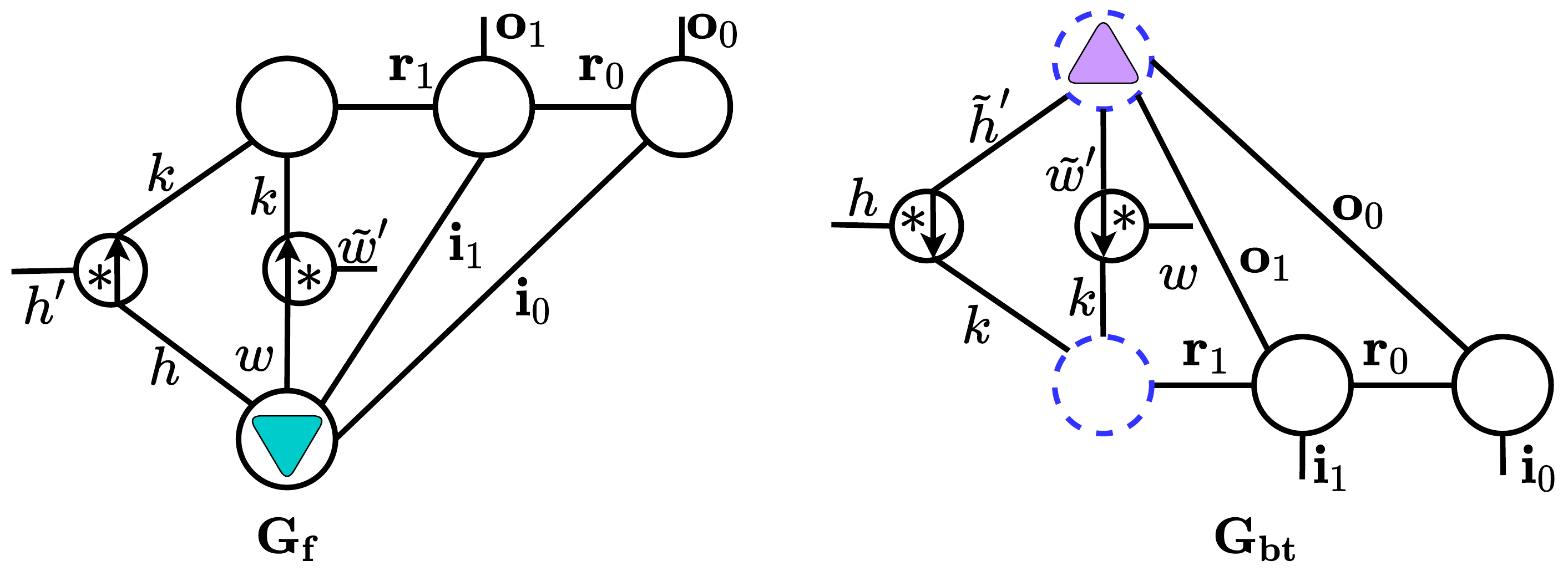}
		\label{fig:ttfb}
	}
	
	\subfigure[HyperNet Convolution]{
		\includegraphics[height=0.12\textheight]{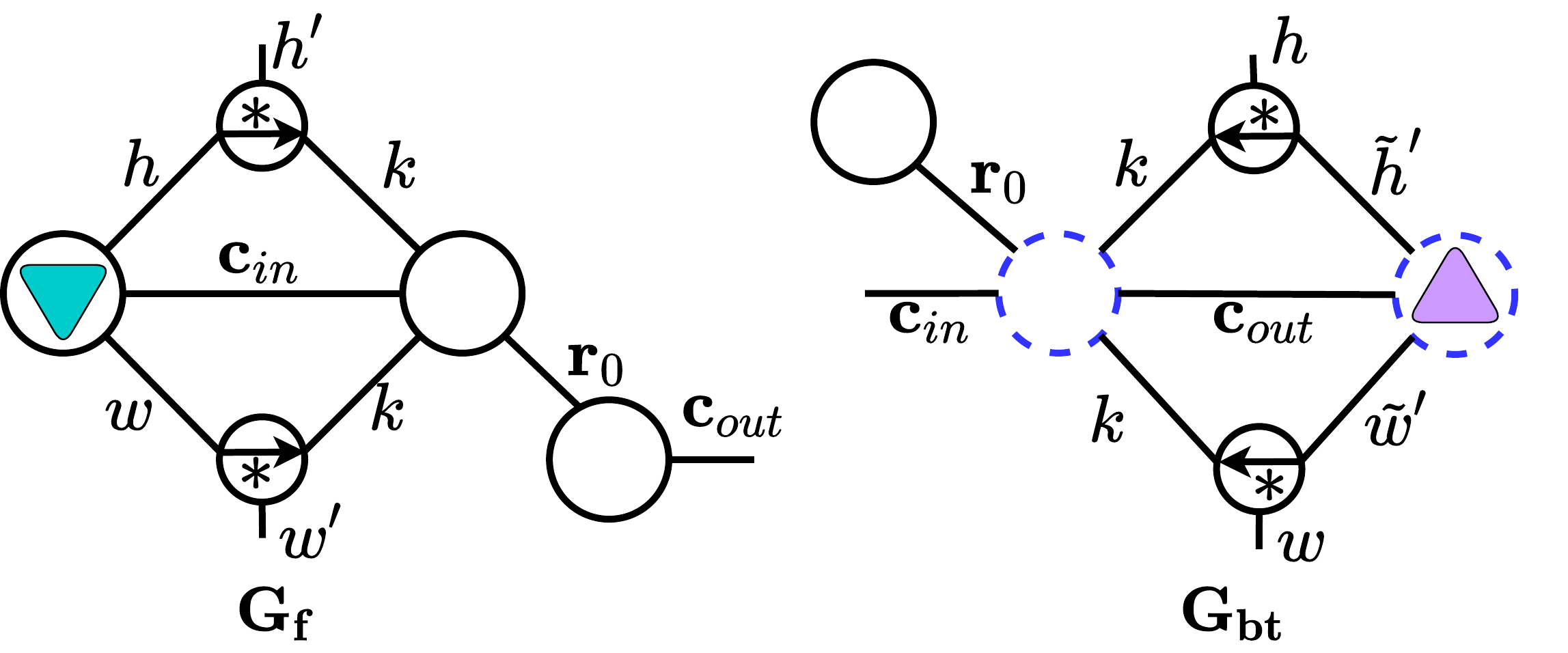}
		\label{fig:hypernetfb}
	}~~~~~~~~~~
	\subfigure[Tensor Ring Convolution]{
		\includegraphics[height=0.12\textheight]{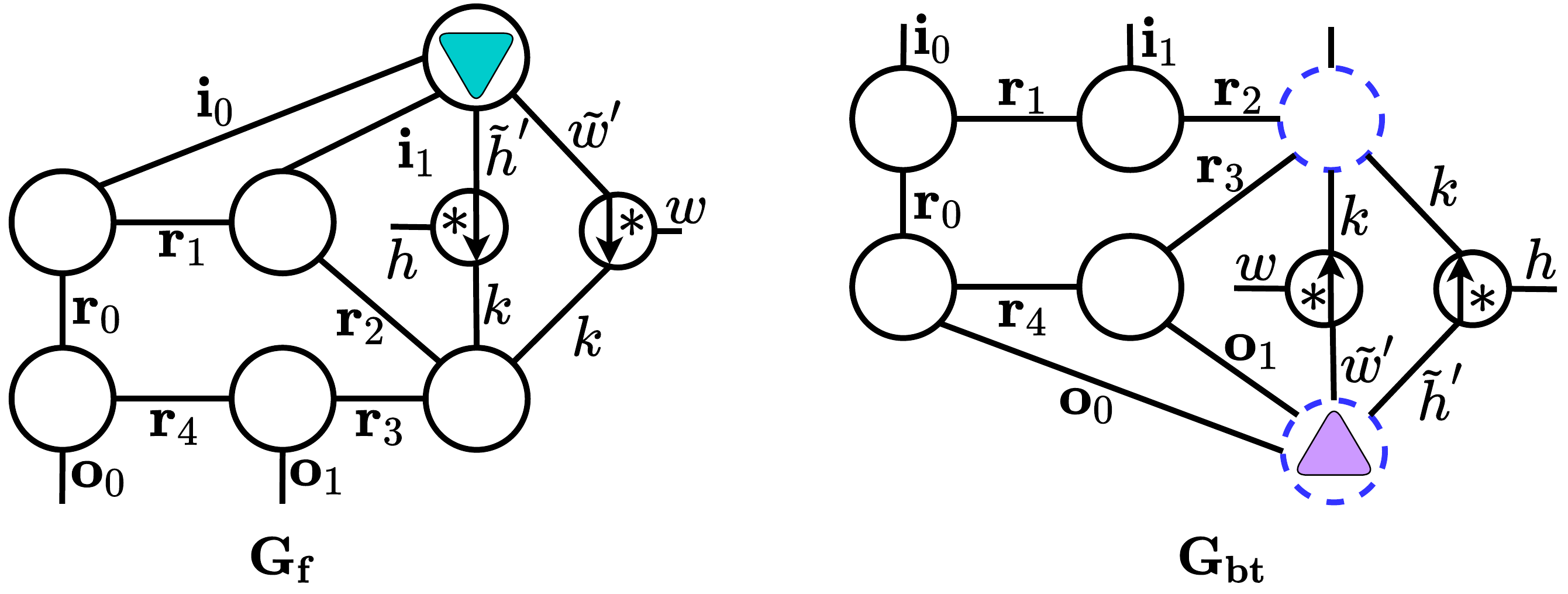}
		\label{fig:trfb}
	}
	
% 	}
	\caption{More Reproducing Transformation cases.}
	\label{fig:appfb}
\end{figure*}

\section{More Reproducing Transformation Cases}
\label{sec:more-graph}
As shown in Figure~\ref{fig:appfb}, we illustrate some additional Reproducing Transformation cases, including Low-rank (LR) convolution, HyperNet convolution, Tensor Train (TT) convolution and Tensor Ring (TR) convolution. From the similar forward and backward tensor graphs, we show that Graph-out initialization can be calculated as exactly the same as the Graph-in.

% \subsection{Experiment Setting}
% \subsubsection{Random Layer Experiment}

\section{Details of Experiments}
\label{sec:exp-detail}

All tensorial network experiments are constructed with the code\footnote{Source code at \url{https://github.com/tnbar/tednet}.} of tednet~\citep{DBLP:journals/ijon/PanWX22}.

\subsection{Applicable Validation for Proposition~\ref{pps:tensorprod}}
\label{sec:pro-validation}

This experiment is implemented to validate Proposition~\ref{pps:tensorprod} can be recursively used under loose conditions (namely non-i.i.d input) in practice. Specifically, Proposition~\ref{pps:tensorprod} requires both  $\ca{X}$ and $\ca{Y}$ i.i.d (denoted as ($\ca{X}$: i.i.d, $\ca{Y}$: i.i.d)), however, in practice, $\ca{Y}$ can easily satisfy the i.i.d condition with an initialization, while $\ca{X}$ is hard to follow the i.i.d. condition, which is denoted as ($\ca{X}$: non-i.i.d, $\ca{Y}$: i.i.d). This situation is also faced by the Xavier and Kaiming initialization methods.   However, although Xavier and Kaiming both adopt the ($\ca{X}$: non-i.i.d, $\ca{Y}$: i.i.d) pattern when initializing networks, they work well in a number of real-world tasks.

Here we conduct a experiment to demonstrate data-flow will keep stable under the ($\ca{X}$: non-i.i.d, $\ca{Y}$: i.i.d) condition,
% In this experiment, using three matrices $\mathbf{X}$, $\mathbf{W}_1$, $\mathbf{W}_2$, we can ensure that $\mathbf{W}_1$ and $\mathbf{W}_2$ are i.i.d, and $\mathbf{X}$ is non-i.i.d. This situation satisfies ($\mathbf{X}$: non-i.i.d, $\mathbf{W}_1$: i.i.d) and ($\mathbf{X}\mathbf{W}_1$: non-i.i.d, $\mathbf{W}_2$: i.i.d), which can be applicable recursively. 
% To further demonstrate this point, we design experiments to 
namely, we show that ($\mathbf{X}$: non-i.i.d, $\mathbf{Y}$: i.i.d) will maintain variance scale (i.e., $\prod_{t=0}^{d-1}v_t$ in Proposition~\ref{pps:tensorprod}). In detail, we generate $\mathbf{X} \in \mathbb{R}^{32\times 96}$ and 10 matrices, $\mathbf{W}_1 \in \mathbb{R}^{96\times 200}$, $\mathbf{W}_2 \in \mathbb{R}^{200\times 400}$, $\mathbf{W}_3 \in \mathbb{R}^{400\times 600}$, $\mathbf{W}_4 \in \mathbb{R}^{600\times 800}$, $\mathbf{W}_5 \in \mathbb{R}^{800\times 1000}$, $\mathbf{W}_6 \in \mathbb{R}^{1000\times 800}$, $\mathbf{W}_7 \in \mathbb{R}^{800\times 600}$, $\mathbf{W}_8 \in \mathbb{R}^{600\times 400}$, $\mathbf{W}_9 \in \mathbb{R}^{400\times 200}$ and $\mathbf{W}_{10} \in \mathbb{R}^{200\times 100}$. All these 10 matrices are generated by sampling from $N\sim(0, 1)$. In Table~\ref{tbl:scale}, Gaussian means $\mathbf{X}$ are generated by sampling from $N\sim(0, 1)$ and Cifar10 denotes that $\mathbf{X}$ is a Cifar10 image. The production sequence is $(\mathbf{X}\mathbf{W}_1)=\mathbf{X}\times \mathbf{W}_1, (\mathbf{X}\mathbf{W}_1\mathbf{W}_2)=(\mathbf{X}\mathbf{W}_1)\times \mathbf{W}_2, \dots, (\mathbf{X}\mathbf{W}_1\mathbf{W}_2\dots \mathbf{W}_{10})=(\mathbf{X}\mathbf{W}_1\mathbf{W}_2\dots)\times \mathbf{W}_{10}$, where $\times$ denotes a matrix production. According to Proposition~\ref{pps:tensorprod}, we calculate the variance scale (i.e., a contracting dimension of $\mathbf{W}_t$ for $t$-th production according to Proposition~\ref{pps:tensorprod}) under ($\mathbf{X}$: i.i.d, $\mathbf{Y}$: i.i.d) as the ground-truth. For example, since $\mathbf{X}\times \mathbf{W}_1$ contracts on dimension 96, the scale is calculated as 96. As for the Gaussian and Cifar10 experiments, we run 500 rounds each.

Results show that when $\mathbf{X}$ is chosen as an i.i.d distribution, its scales are almost around the Ground-Truth, which approximate the ($\mathbf{X}$: i.i.d, $\mathbf{Y}$: i.i.d). For the much more tricky case (i.e., Cifar10), since $\mathbf{X}$ is a realistic image and non-i.i.d, the variance of Scale 1 is a little large, but the mean of Scale 1 is not far away from 96. And other scales do not vary a lot and are close to the Ground-Truth. Therefore, the experiments show that ($\mathbf{X}$: non-i.i.d, $\mathbf{Y}$: i.i.d) can approximate ($\mathbf{X}$: i.i.d, $\mathbf{Y}$: i.i.d) and can be applicable recursively in practice.

\begin{table}[h]
\centering
\caption{Scale change in propagation.}
\label{tbl:scale}
\scalebox{0.68}{\begin{tabular}{l|l|l|l|l|l|l|l|l|l|l}
\hline
Data         & Scale 1       & Scale 2      & Scale 3       & Scale 4       & Scale 5       & Scale 6       & Scale 7       & Scale 8       & Scale 9       & Scale 10      \\ \hline\hline
Ground-Truth & 96            & 200          & 400           & 600           & 800           & 1000          & 800           & 600           & 400           & 200           \\ \hline
Gaussian     & 96.4 $\pm$ 1.7   & 201.5 $\pm$ 2.5 & 398.0 $\pm$ 3.8  & 603.4 $\pm$ 5.5  & 790.4 $\pm$ 6    & 1011.0 $\pm$ 8.9 & 805.3 $\pm$ 7.9  & 601.2 $\pm$ 7.0  & 385.6 $\pm$ 7    & 202.0 $\pm$ 5.2  \\ \hline
Cifar10      & 122.0 $\pm$ 54.8 & 200.2 $\pm$ 5.9 & 427.1 $\pm$ 12.4 & 588.2 $\pm$ 11.3 & 803.7 $\pm$ 12.8 & 969.3 $\pm$ 18.9 & 799.4 $\pm$ 17.5 & 567.4 $\pm$ 18.3 & 394.7 $\pm$ 19.0 & 193.1 $\pm$ 12.1 \\ \hline
\end{tabular}}
\end{table}

\subsection{Details of Activation Propagation Analysis}
\label{sec:detail-converge}
We perform this experiment on MNIST through Linear-5. The structure of Linear-5 is shown in Table~\ref{tbl:linear5}. The mini-batch size is set to 20 and the learning rate is 1e-4. The $\mathbf{r}_{\ast}$ ranks are chosen to be 5. The training process lasts 80 epoch and is optimized by Adam. We use one NVIDIA GTX 1080ti GPU for this experiment.

\subsection{Random Generator}
\label{sec:detail-random-generator}
We construct a random generator by randomly generating 4-8 vertices, 2-3 input $\mathbf{i}_{\ast}$ edges, 2-3 output $\mathbf{o}_{\ast}$ edges and uncertain number of $\mathbf{r}_{\ast}$ edges. Notably, it is guaranteed that each vertex connects to an edge.

\begin{table}[t]
\centering
\caption{Structures of Linear-5 and HOdd-5.}
\label{tbl:linear5}
\begin{tabular}{c|c|c}
\hline
Layer  & Linear-5       & HOdd--5                           \\ \hline\hline
linear1 & 784$\times$500 & (28$\times$28)$\times$(20$\times$25) \\ \hline
linear2 & 500$\times$500 & (20$\times$25)$\times$(20$\times$25) \\ \hline
linear3 & 500$\times$500 & (20$\times$25)$\times$(20$\times$25) \\ \hline
linear4 & 500$\times$500 & (20$\times$25)$\times$(20$\times$25) \\ \hline
linear5 & 500$\times$10  & (20$\times$25)$\times$(2$\times$5)   \\ \hline
\end{tabular}
\end{table}

\subsection{Details of Random Tensor Format Experiment}
\label{sec:detail-conv4}
In this experiment, each layer of Conv-4 can be constructed by a random generator in Appendix~\ref{sec:detail-random-generator}. MNIST is used for training and validation. The batch size is set to 128. The learning rate is 1e-4 and the optimizer is Adam. We train each Conv-4 for 20 epochs on a NVIDIA Tesla V100 GPU.

\subsection{Details of Experiment on Cifar10}
\label{sec:detail-cifar10-all}
\subsubsection{Details of Experiment on Cifar10 in Main Paper}
\label{sec:detail-cifar10}
We conduct this experiment based on the All Convolutional Net model~\citep{DBLP:conf/iclr/ChangFL20, DBLP:journals/corr/SpringenbergDBR14}. By replcaing standard convolution in All Convolutional Net with TCNN convolutions, including Low-rank (LR) convolution, Tensor Ring (TR) convolution, Hyper Tucker-2 (HTK2) convolution,  and Hyper Odd (HOdd) convolution, we can derive TCNN based All Convolutional Nets, as shown in Table~\ref{tbl:all-conv}. In addition, we also use random generator in Appendix~\ref{sec:detail-random-generator} to implement 30 different All Convolutional Nets for evaluation.
In this experiment, we use Cifar10 dataset for training. The mini-batch size is 128. We choose SGD as the optimizer with learning rate 5e-3, momentum 0.9 and weight decay 5e-4. All $\mathbf{r}_{\ast}$ ranks are set to 10. Each network is trained for 270 epochs on an NVIDIA Tesla V100 GPU and the learning rate will be multiplied by a fixed multiplier of 0.2 after 100, 180 and 230 epochs separately. Some results are shown in Figure~\ref{fig:cifarodd}.

\begin{table*}[h]
\centering
\caption{Architectures of the tensorial All-Conv networks. Window means the convolutional kernel window size. Channels indicate $\mathbf{c}_{in}$ and $\mathbf{c}_{out}$ of a standard convolutional kernel $\ca{C}\in \mathbb{R}^{\mathbf{c}_{in}\times \mathbf{c}_{out}\times k \times k}$. The avg pool denotes the average pooling operation.}
\label{tbl:all-conv}
\scalebox{0.8}{
\begin{tabular}{c|c|c|c|c|c|c|c}
\hline
Layer  & Window & Channels      & HTK2/Tucker              & TR                                                     & TT                                                      & Low-Rank          & HOdd                                  \\ \hline\hline
conv1   & 3$\times$3    & 3$\times$ 96  & (3)$\times$ (96)  & (3)$\times$(6$\times$ 4$\times$ 4)                     & (1$\times$ 3$\times$ 1)$\times$(6$\times$ 4$\times$ 4)  & (3)$\times$ (96)  & (1$\times$ 3)$\times$(8$\times$ 12)  \\ \hline
conv2   & 3$\times$3    & 96$\times$ 96 & (96)$\times$ (96) & (6$\times$ 4$\times$ 4)$\times$(6$\times$ 4$\times$ 4) & (6$\times$ 4$\times$ 4)$\times$(6$\times$ 4$\times$ 4)  & (96)$\times$ (96) & (8$\times$ 12)$\times$(8$\times$ 12) \\ \hline
conv3   & 3$\times$3    & 96$\times$ 96 & (96)$\times$ (96) & (6$\times$ 4$\times$ 4)$\times$(6$\times$ 4$\times$ 4) & (6$\times$ 4$\times$ 4)$\times$(6$\times$ 4$\times$ 4)  & (96)$\times$ (96) & (8$\times$ 12)$\times$(8$\times$ 12) \\ \hline
conv4   & 3$\times$3    & 96$\times$ 96 & (96)$\times$ (96) & (6$\times$ 4$\times$ 4)$\times$(6$\times$ 4$\times$ 4) & (6$\times$ 4$\times$ 4)$\times$(6$\times$ 4$\times$ 4)  & (96)$\times$ (96) & (8$\times$ 12)$\times$(8$\times$ 12) \\ \hline
conv5   & 3$\times$3    & 96$\times$ 96 & (96)$\times$ (96) & (6$\times$ 4$\times$ 4)$\times$(6$\times$ 4$\times$ 4) & (6$\times$ 4$\times$ 4)$\times$(6$\times$ 4$\times$ 4)  & (96)$\times$ (96) & (8$\times$ 12)$\times$(8$\times$ 12) \\ \hline
conv6   & 3$\times$3    & 96$\times$ 96 & (96)$\times$ (96) & (6$\times$ 4$\times$ 4)$\times$(6$\times$ 4$\times$ 4) & (6$\times$ 4$\times$ 4)$\times$(6$\times$ 4$\times$ 4)  & (96)$\times$ (96) & (8$\times$ 12)$\times$(8$\times$ 12) \\ \hline
conv7   & 3$\times$3    & 96$\times$ 96 & (96)$\times$ (96) & (6$\times$ 4$\times$ 4)$\times$(6$\times$ 4$\times$ 4) & (6$\times$ 4$\times$ 4)$\times$(6$\times$ 4$\times$ 4)  & (96)$\times$ (96) & (8$\times$ 12)$\times$(8$\times$ 12) \\ \hline
conv8   & 3$\times$3    & 96$\times$ 96 & (96)$\times$ (96) & (6$\times$ 4$\times$ 4)$\times$(6$\times$ 4$\times$ 4) & (6$\times$ 4$\times$ 4)$\times$(6$\times$ 4$\times$ 4)  & (96)$\times$ (96) & (8$\times$ 12)$\times$(8$\times$ 12) \\ \hline
conv9   & 3$\times$3    
&  \begin{tabular}{c}
96$\times$ 10 \\ 
% \specialrule{0em}{1pt}{1pt} 
avg pool
\end{tabular} 
& \begin{tabular}{c}
(96)$\times$ (10) \\ 
% \specialrule{0em}{1pt}{1pt} 
avg pool
\end{tabular} 
& \begin{tabular}{c}
(6$\times$ 4$\times$ 4)$\times$(10)  \\ 
% \specialrule{0em}{1pt}{1pt} 
avg pool
\end{tabular}                    
& \begin{tabular}{c}
(6$\times$ 4$\times$ 4)$\times$(1$\times$ 10$\times$ 1)   \\ 
% \specialrule{0em}{1pt}{1pt} 
avg pool
\end{tabular}
& \begin{tabular}{c}
(96)$\times$ (10)  \\ 
% \specialrule{0em}{1pt}{1pt} 
avg pool
\end{tabular}
&\begin{tabular}{c}
(8$\times$ 12)$\times$(1$\times$ 10) \\ 
% \specialrule{0em}{1pt}{1pt} 
avg pool
\end{tabular}\\ \hline
\end{tabular}
}
\end{table*}

\begin{figure*}[h]
	\centering
	\subfigure[HTK2-C4 Loss]{
		\includegraphics[width=0.22\textwidth]{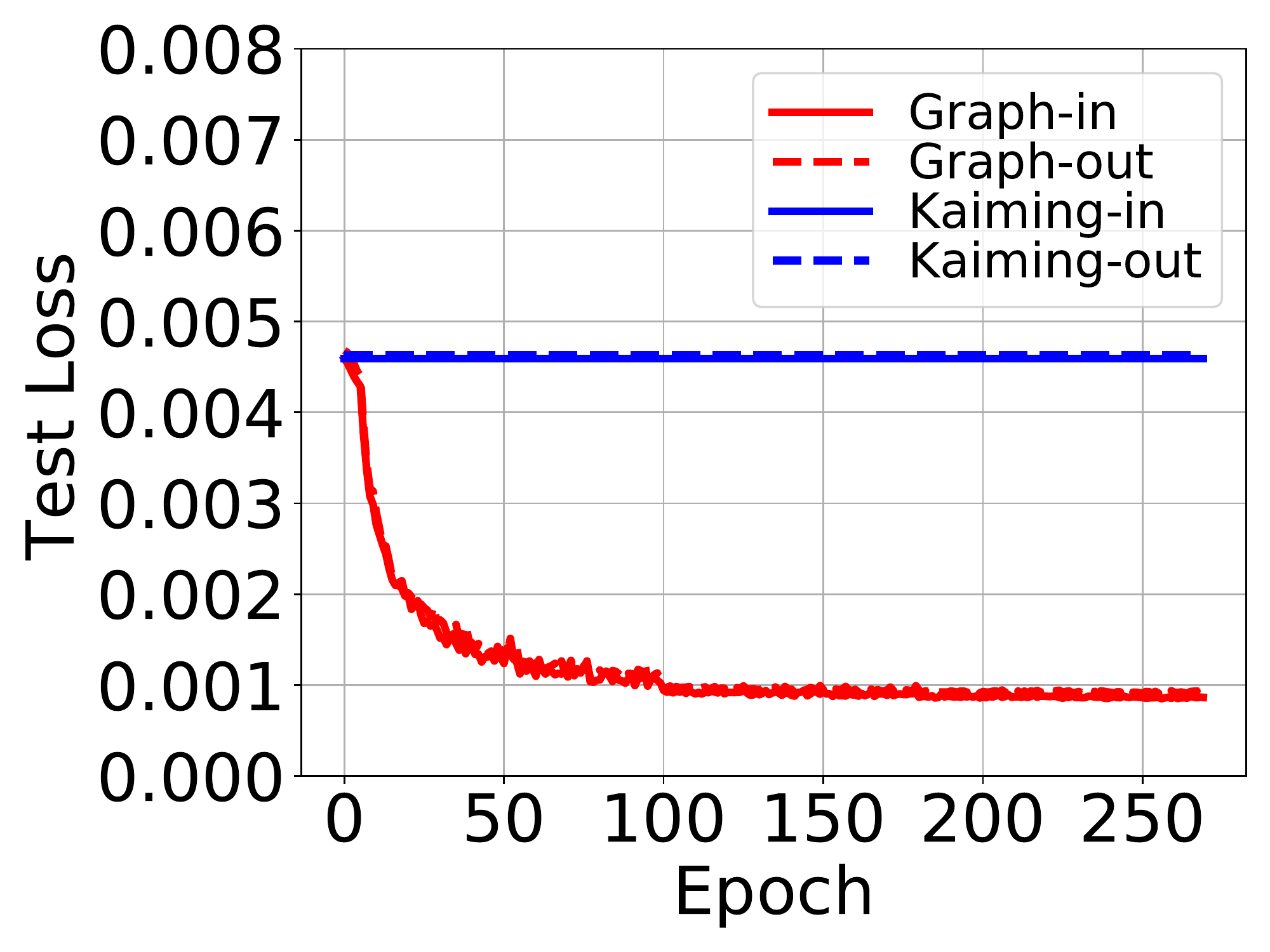}
		\label{fig:cifarhtk2c4loss}
	} 
	\subfigure[HTK2-C4 Accuracy]{
		\includegraphics[width=0.22\textwidth]{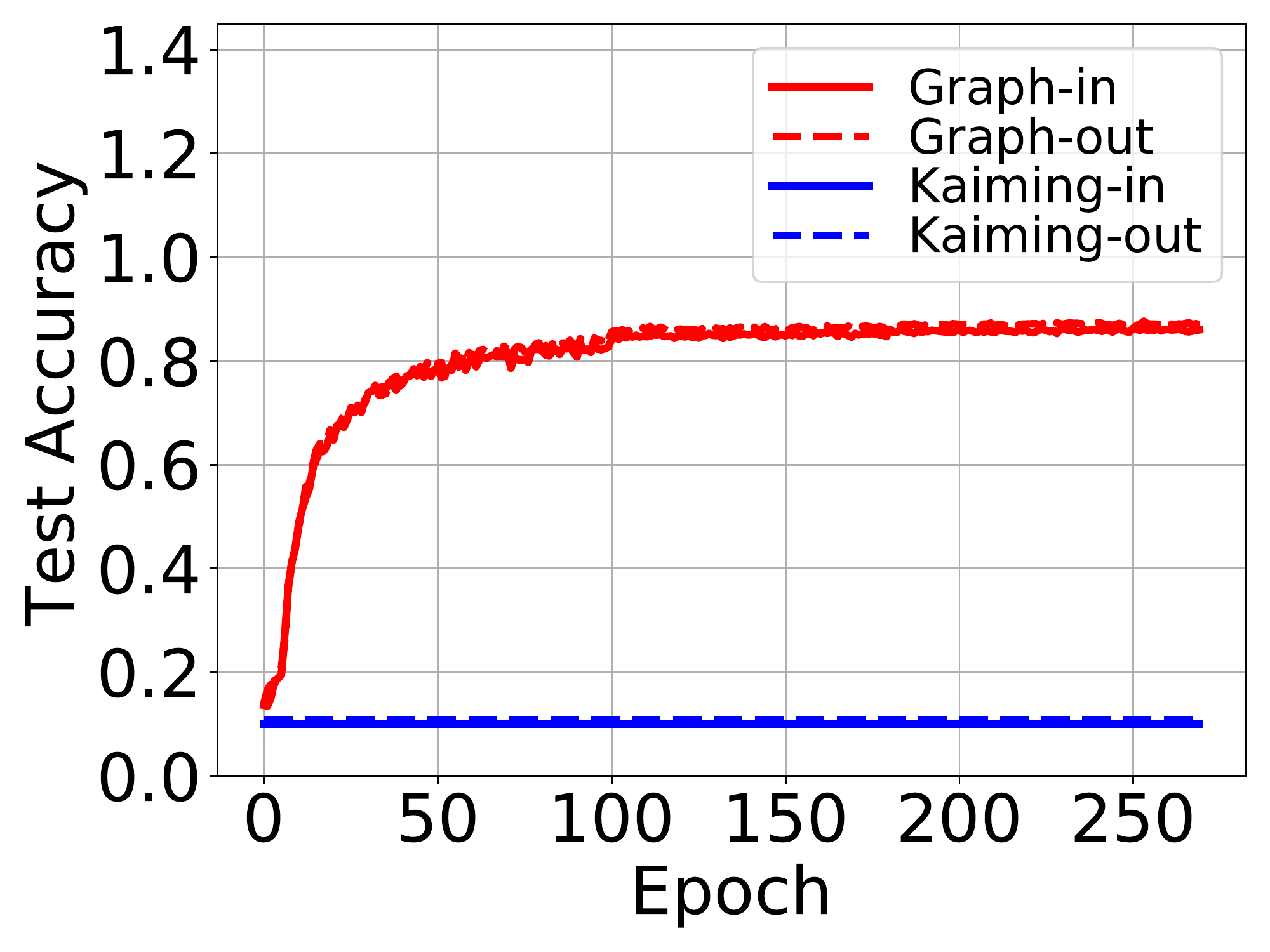}
		\label{fig:cifarhtk2c4acc}
	}
	\subfigure[HOdd-C4 Loss]{
		\includegraphics[width=0.22\textwidth]{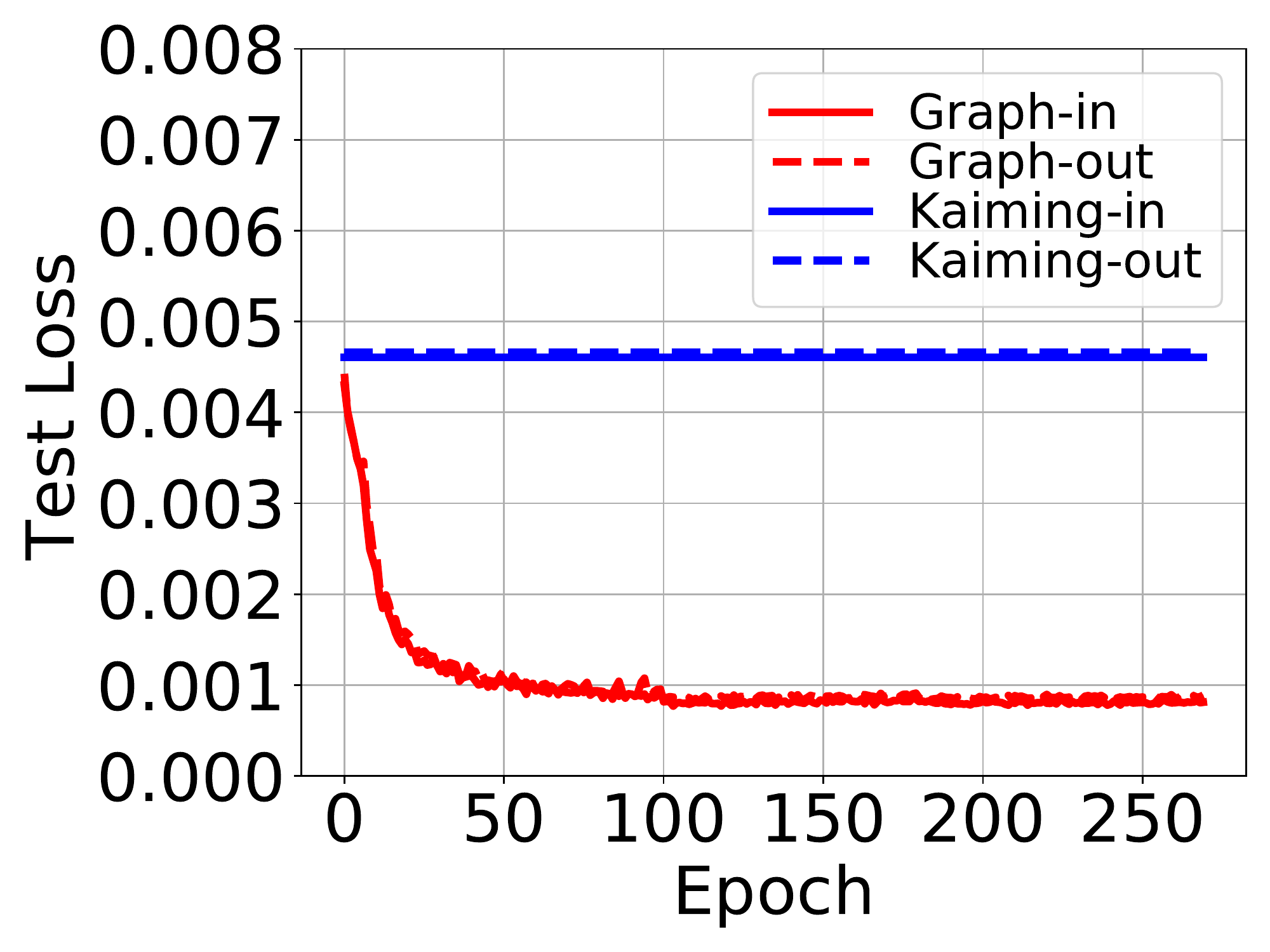}
		\label{fig:cifaroddloss}
	}
	\subfigure[HOdd-C4 Accuracy]{
		\includegraphics[width=0.22\textwidth]{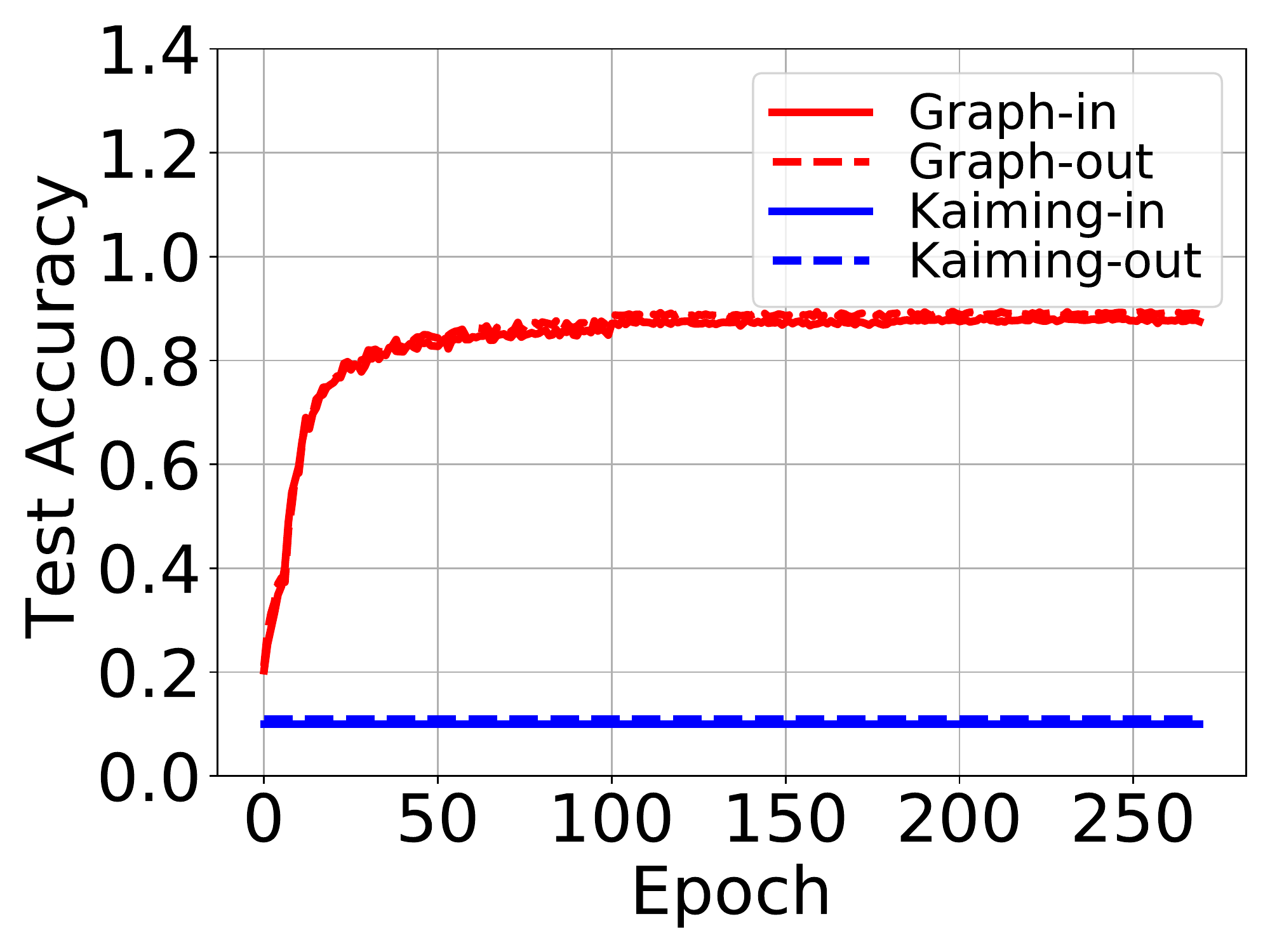}
		\label{fig:cifaroddacc}
	}
	\caption{Results on Cifar10.}
	\label{fig:cifarodd}
\end{figure*}

\subsubsection{Comparison with Ad-hoc Initialization}
\label{sec:adhoc}

To our best knowledge, our initialization is the first unified method for TCNNs, and there are no other unified initialization methods to be compared. Therefore, we would like to compare with initialization for a specific tensor format (i.e., tensor ring~\citep{DBLP:conf/cvpr/WangSEWA18} and tucker-2~\citep{DBLP:journals/corr/abs-1909-05675}), which is hard to extend to other tensor formats. In this experiment, we use Cifar10 as the validation dataset. We adopt All-Conv and ResNet-32 as the base models. Then, tensor ring and tucker-2 are adopted as tensor formats for All-Conv (termed as TR-All-Conv) and ResNet-32 (termed as TK2-ResNet-32), respectively. The experiment is conducted on an NVIDIA Tesla V100 GPU.

\textbf{TR-All-Conv Setting}~~The mini-batch size is 128. The Optimizer is SGD with learning rate 5e-3, momentum 0.9 and weight decay 5e-4. All $\mathbf{r}_{\ast}$ ranks are set to 10.

\textbf{TK2-ResNet-32 Setting}~~For method of \citet{DBLP:journals/corr/abs-1909-05675}, we use their official code\footnote{Source code at \url{https://github.com/mostafaelhoushi/tensor-decompositions}.} with default setting. For comparison with initial weights, here we decompose the original weight before training. For the proposed graphical method,  we set the batch size to 128 and optimizer as SGD with learning rate 0.1, momentum 0.9 and weight decay 5e-4. Shape of tucker-2 set as the decomposed format from method of \citet{DBLP:journals/corr/abs-1909-05675}. We train each model for 90 epochs with reducing the learning rate through multiplying it by 0.1 after 30 and 60 epochs respectively.

Results are shown in Table~\ref{tbl:adhoc}. Compared the two ad-hoc methods, the proposed graphical performs better than the two ad-hoc initialization. As specially designed methods, the two ad-hoc initialization cannot be applied to other tensor formats. By contrast, our initialization can also be applied to other formats, which indicates a practical advantage of the proposed unified initialization method. 

\begin{table}[h]
\centering
\caption{Comparison with ad-hoc initialization on Cifar-10.}
\label{tbl:adhoc}
\begin{tabular}{c|c|c|c}
\hline
Initialization (TR-All-Conv) & Accurracy & Initialization (TK2-ResNet-32) & Accurracy \\ \hline\hline
\citet{DBLP:conf/cvpr/WangSEWA18}                & 0.8307    & \citet{DBLP:journals/corr/abs-1909-05675}                  & 0.8488    \\ \hline
Graph-in           & 0.8308    & Graph-in             & 0.8554    \\ \hline
Graph-out          & \textbf{0.8311}    & Graph-out            & \textbf{0.8654}    \\ \hline
\end{tabular}
\end{table}

\subsubsection{Comparison with Common Tensor Formats}
\label{sec:formats}

In this section, we validate our initialization method by comparing with CP, Tucker, Tensor Ring (TR) convolution, Tensor Train (TT) convolution, Low-rank (LR) convolution on CIFAR10. To better estimate the performance of our method, we adopt All Convolutional Net structure~\citep{DBLP:journals/corr/SpringenbergDBR14}, which only contains convolutional layers without Batch Normalization~\citep{DBLP:conf/icml/IoffeS15} and residual connection. Optimizer is chosen to be SGD. Reducing normalization tricks and Adam optimization, the training will rely much more on weight initialization. Thus, it will be more clear how our method performs compared with baselines. All $\mathbf{r}_{\ast}$ ranks are set to 10 for convenience. The learning rate is set to 5e-3. The experiment is conducted on an NVIDIA Tesla V100 GPU.

In practice, a good initialization should generate weight suitable for diverse models. However, as shown in Table~\ref{tbl:formats}, Kaiming-in initialization fails to train all these TCNNs and so does Kaiming-out. On the contrary, our graphical initialization performs well in such a situation. As shown in the figure, our initialization shows an adaptive ability of fitting two completely different TCNN models.

\begin{table}[h]
\centering
\caption{Comparison with some common tensor formats on Cifar-10.}
\label{tbl:formats}
\begin{tabular}{c|c|c|c|c|c}
\hline
Initialization    & CP     & Tucker & Tensor Ring & Tensor Train & Low-Rank \\ \hline\hline
Kaiming(-in/-out) & 0.1    & 0.1    & 0.1         & 0.1          & 0.1      \\ \hline
Graph-in          & 0.7823 & 0.7775 & 0.8308      & 0.8276       & 0.8141   \\ \hline
Graph-out         & 0.767  & 0.7709 & 0.8311      & 0.8341       & 0.8163   \\ \hline
\end{tabular}
\end{table}

\begin{table*}[h]
	\centering
	\caption{Structures of HTK2 and HOdd based ResNet. $k$ represents convolutional kernel window. $\bullet$ n denotes n same residual blocks. When depth is set to 50, $\{\mathbf{U}_{i}\}_{i=1}^{4}$ are 2, 3, 5, 2. And $\{\mathbf{U}_{i}\}_{i=1}^{4}$ of 101-depth are set to 2, 3, 22, 2. }
	\label{tbl:resnet}
	\scalebox{0.9}{
\begin{tabular}{c|c|c}
\hline
Layer  & HTK2                                                                                                                                                                                                                                                                                                                                                                                                                                                                                                                                                                               & HOdd                                                                                                                                                                                                                                                                                                                                                                                                                                                                                                                                                                                                                                                                           \\ \hline
Pre    & $k7$, (3)$\times$(64)                                                                                                                                                                                                                                                                                                                                                                                                                                                                                                                                                              & $k7$, (1$\times$3)$\times$(8$\times$8)                                                                                                                                                                                                                                                                                                                                                                                                                                                                                                                                                                                                                                         \\ \hline
Unit 1 & \begin{tabular}{l}        $\left[        \begin{array}{cl}             k1,   & (64) \times (64) \\             k3,   & (64) \times (64) \\             k1,   & (64) \times (256)        \end{array}        \right]\bullet 1        $ \\        $\left[        \begin{array}{cl}             k1,   & (256) \times (64) \\             k3,   & (64) \times (64) \\             k1,   & (64) \times (256)        \end{array}        \right]\bullet {\mathbf{U}_1}        $      \end{tabular}                                                                                         & \begin{tabular}{l}        $\left[        \begin{array}{cl}             k1,   & (8\times 8) \times (8\times 8) \\             k3,   & (8\times 8) \times (8\times 8) \\             k1,   & (8\times 8) \times (32\times 8)        \end{array}        \right]\bullet 1        $ \\        $\left[        \begin{array}{cl}             k1,   & (32\times 8) \times (8\times 8) \\             k3,   & (8\times 8) \times (8\times 8) \\             k1,   & (8\times 8) \times (32\times 8)        \end{array}        \right]\bullet {\mathbf{U}_1}        $      \end{tabular}                                                                                                 \\ \hline
Unit 2 & \begin{tabular}{l}            $\left[        \begin{array}{cl}             k1,   & (256) \times (128) \\             k3,   & (128) \times (128) \\             k1,   & (128) \times (512)        \end{array}        \right]\bullet 1        $ \\        $\left[        \begin{array}{cl}             k1,   & (512) \times (128) \\             k3,   & (128) \times (128) \\             k1,   & (128) \times (512)        \end{array}        \right]\bullet {\mathbf{U}_2}        $      \end{tabular}                                                                            & \begin{tabular}{l}                $\left[        \begin{array}{cl}             k1,   & (32\times 8) \times (8\times 16) \\             k3,   & (8\times 16) \times (8\times 16) \\             k1,   & (8\times 16) \times (32\times 16)        \end{array}        \right]\bullet 1        $ \\        $\left[        \begin{array}{cl}             k1,   & (32\times 16) \times (8\times 16) \\             k3,   & (8\times 16) \times (8\times 16) \\             k1,   & (8\times 16) \times (32\times 16)        \end{array}        \right]\bullet {\mathbf{U}_2}        $          \end{tabular}                                                                         \\ \hline
Unit 3 & \begin{tabular}{l}        $\left[        \begin{array}{cl}             k1,   & (512) \times (256) \\             k3,   & (256) \times (256) \\             k1,   & (256) \times (1024)        \end{array}        \right]\bullet 1        $ \\        $\left[        \begin{array}{cl}             k1,   & (1024) \times (256) \\             k3,   & (256) \times (256) \\             k1,   & (256) \times (1024)        \end{array}        \right]\bullet {\mathbf{U}_3}        $          \end{tabular}                                                                         & \begin{tabular}{l}        $\left[        \begin{array}{cl}             k1,   & (32\times 16) \times (16\times 16) \\             k3,   & (16\times 16) \times (16\times 16) \\             k1,   & (16\times 16) \times (64\times 16)        \end{array}        \right]\bullet 1        $ \\        $\left[        \begin{array}{cl}             k1,   & (64\times 16) \times (16\times 16) \\             k3,   & (16\times 16) \times (16\times 16) \\             k1,   & (16\times 16) \times (64\times 16)        \end{array}        \right]\bullet {\mathbf{U}_3}        $          \end{tabular}                                                                        \\ \hline
Unit 4 & \begin{tabular}{l}        $\left[        \begin{array}{cl}             k1,   & (1024) \times (512) \\             k3,   & (512) \times (512) \\             k1,   & (512) \times (2048)        \end{array}        \right]\bullet 1        $ \\        $\left[        \begin{array}{cl}             k1,   & (2048) \times (512) \\             k3,   & (512) \times (512) \\             k1,   & (512) \times (2048)        \end{array}        \right]\bullet {\mathbf{U}_4}        $ \\         \specialrule{0em}{5pt}{1pt}        \multicolumn{1}{c}{avg pool}      \end{tabular} & \begin{tabular}{l}        $\left[        \begin{array}{cl}             k1,   & (64\times 16) \times (16\times 32) \\             k3,   & (16\times 32) \times (16\times 32) \\             k1,   & (16\times 32) \times (64\times 32)        \end{array}        \right]\bullet 1        $ \\        $\left[        \begin{array}{cl}             k1,   & (64\times 32) \times (16\times 32) \\             k3,   & (16\times 32) \times (16\times 32) \\             k1,   & (16\times 32) \times (64\times 32)        \end{array}        \right]\bullet {\mathbf{U}_4}        $ \\         \specialrule{0em}{5pt}{1pt}        \multicolumn{1}{c}{avg pool}      \end{tabular} \\ \hline
FC     & \begin{tabular}{c}            $(2048) \times (200)$ \\      \end{tabular}                                                                                                                                                                                                                                                                                                                                                                                                                                                                                                          & \begin{tabular}{c}            $(64\times 32) \times (10\times 20)$\\       \end{tabular}                                                                                                                                                                                                                                                                                                                                                                                                                                                                                                                                                                                       \\ \hline
\end{tabular}
}
\end{table*}

\begin{table*}[h]
\centering
\caption{Top-1 accuracy on Cifar10 and Tiny-ImageNet.
% Rank-Edge means edges connected only with weight vertices. 
Rank-Edge Number means the least number of edges only connected to weight vertices in layers. Random-$\ast$ denotes randomly generating models.}
\label{tbl:cifar-tiny-full}
\scalebox{0.9}{
\begin{tabular}{c|c|ccc|ccc}
\hline
        &                                                    & \multicolumn{3}{c|}{Cifar10}                                                        & \multicolumn{3}{c}{Tiny-ImageNet}                                                  \\ \hline
        & \begin{tabular}[c]{@{}c@{}}Rank-Edge\\ Number\end{tabular} & \begin{tabular}[c]{@{}c@{}}Kaiming\\ (-in/-out)\end{tabular} & Graph-in & Graph-out & \begin{tabular}[c]{@{}c@{}}Kaiming\\ (-in/-out)\end{tabular} & Graph-in & Graph-out \\ \hline\hline
Low-Rank & 1                                                  & 0.1                                                          & 0.8141   & 0.8163    & 0.307/0.2776                                                          & 0.3153   & 0.3076    \\
Tensor Ring      & 4                                                  & 0.1                                                          & 0.8308   & 0.8311    & 0.005                                                          & 0.2494   & 0.249    \\
% TT      & 4                                                  & 0.1                                                          & 0.8276   & 0.8341    & 0.005                                                          & -   & 0.8341    \\
% TK2     & 2                                                  & 0.1                                                          & 0.7775   & 0.7709    & 0.2245/0.005                                                          & 0.2183   & 0.2149    \\
% ODD     & 5                                                  & 0.1                                                          & 0.8512   & 0.849     & 0.005                                                          & 0.4742   & 0.4674     \\
\hline
HTK2 ($\varphi$=4)     & 2                                                  & 0.1                                                          & 0.8638   & 0.8705    & 0.005                                                          & 0.4014   & 0.4126    \\
HOdd ($\varphi$=4)     & 14                                                  & 0.1                                                          & 0.8826   & 0.8806     & 0.005         & 0.5048   & 0.5045     \\ \hline
Random-1      & -                                                  & 0.1                                                          & 0.8538  & 0.8483     & 0.005                                                          & 0.4965   & 0.5015     \\
Random-2      & -                                                  & 0.1                                                          & 0.8801   & 0.876     & 0.005                                                          & 0.5379   & 0.5356     \\
Random-3      & -                                                  & 0.1                                                          & 0.8648   & 0.863     & 0.005                                                          & 0.5475   & 0.5403     \\
Random-4      & -                                                  & 0.1                                                          & 0.8789   & 0.8816     & 0.005                                                          & 0.5295   & 0.5306     \\ 
Random-5      & -                                                  & 0.1                                                          & 0.8622   & 0.8644     & 0.005                                                          & 0.5444   & 0.5428     \\  
Random-6      & -                                                  & 0.1                                                          & 0.8735   & 0.8721     & 0.005                                                          & 0.5452   & 0.5446     \\  
Random-7      & -                                                  & 0.1                                                          & 0.8601   & 0.8558     & 0.005                                                          & 0.5328   & 0.5394     \\  
Random-8      & -                                                & 0.1                                                          & 0.8589   & 0.8561     & 0.005                                                          & 0.5269   & 0.5291     \\ 
\hline
\end{tabular}
}
\end{table*}

\begin{figure*}[h]
	\centering
	\subfigure[HTK2 ($\varphi$=4) Loss]{
		\includegraphics[width=0.22\textwidth]{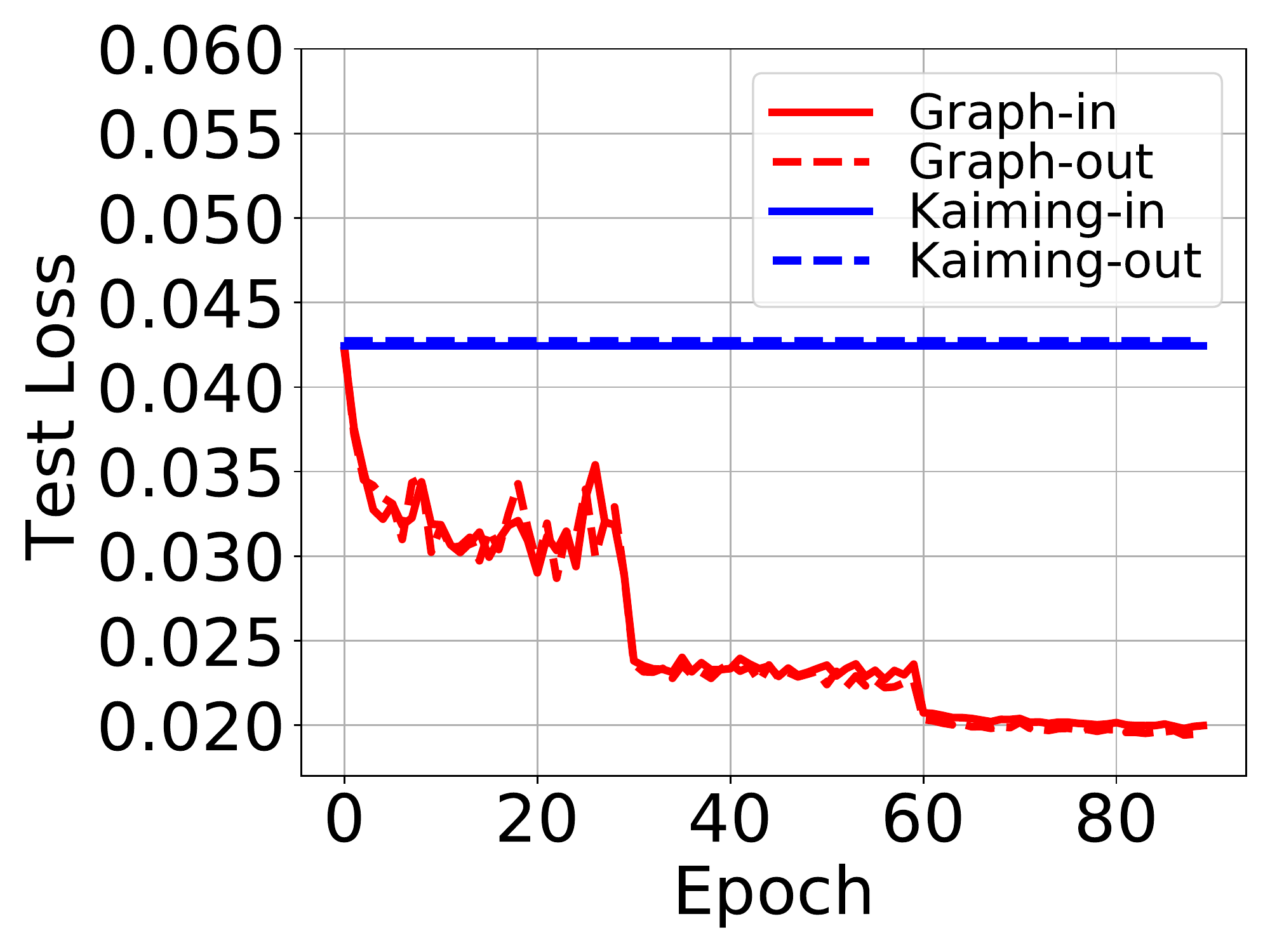}
		\label{fig:tinyhtk2c4loss}
	} 
	\subfigure[HTK2 ($\varphi$=4) Accuracy]{
		\includegraphics[width=0.22\textwidth]{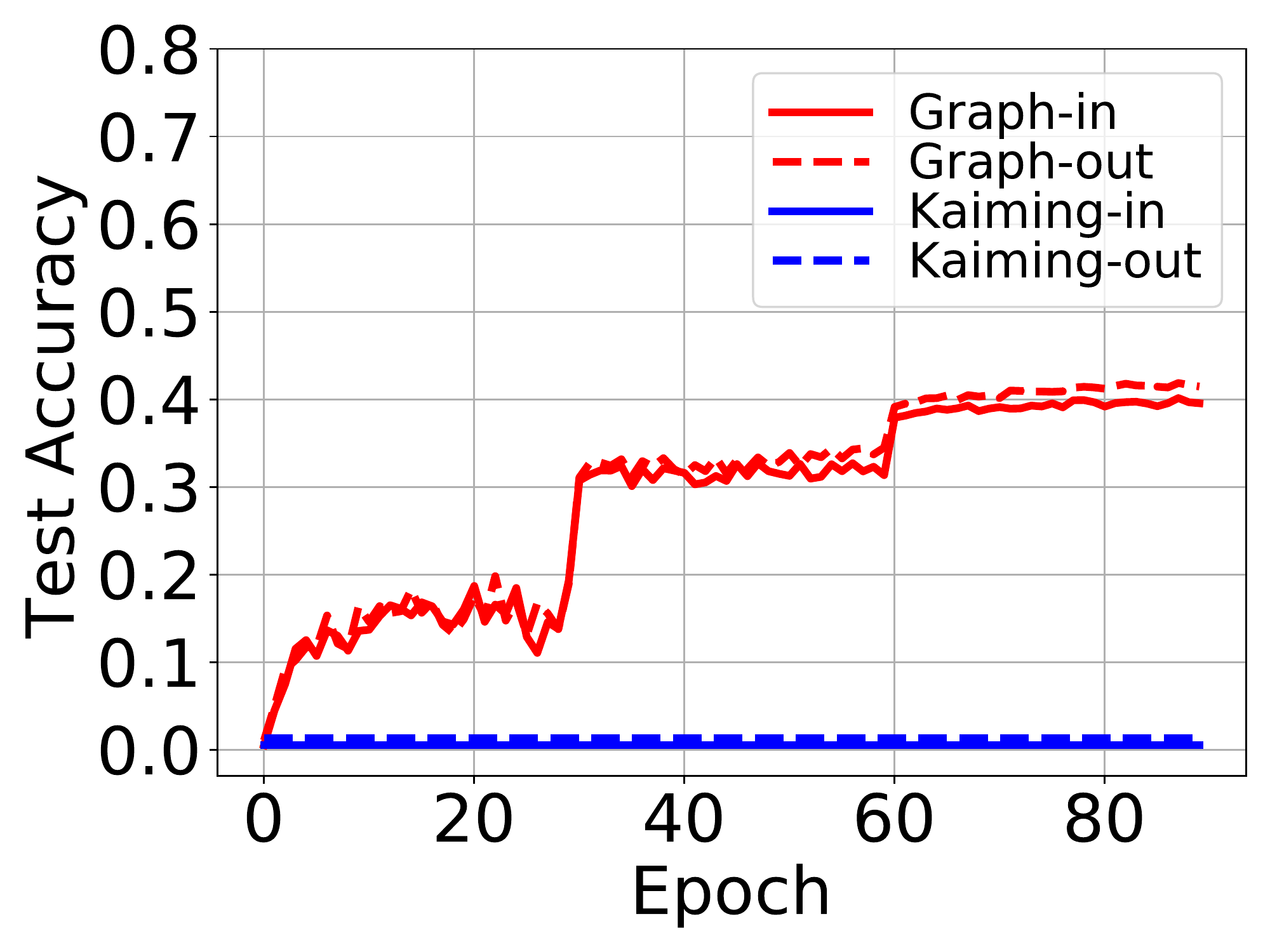}
		\label{fig:tinyhtk2c4acc}
	}
	\subfigure[HOdd ($\varphi$=4) Loss]{
		\includegraphics[width=0.22\textwidth]{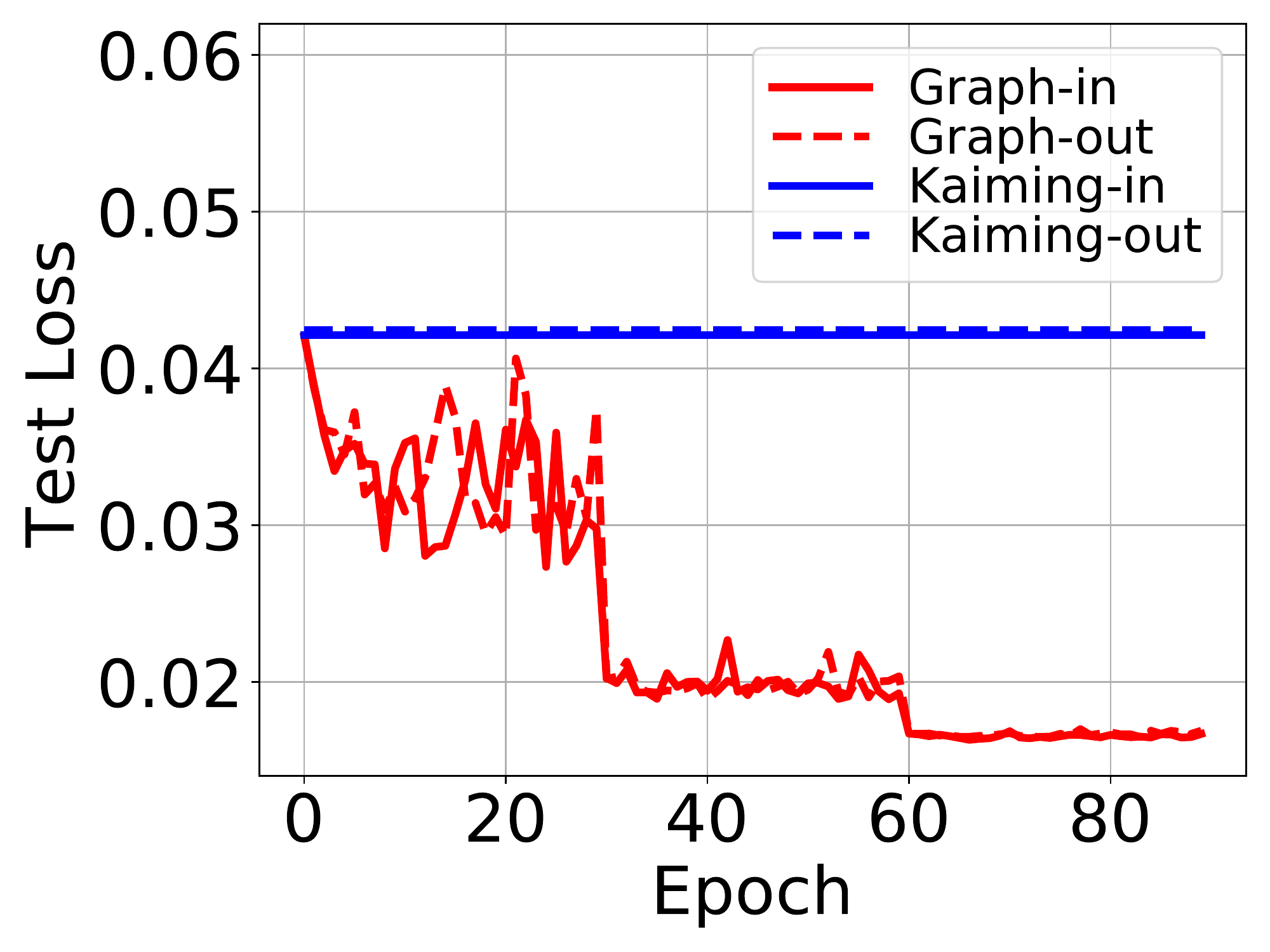}
		\label{fig:tinyoddloss}
	}
	\subfigure[HOdd ($\varphi$=4) Accuracy]{
		\includegraphics[width=0.22\textwidth]{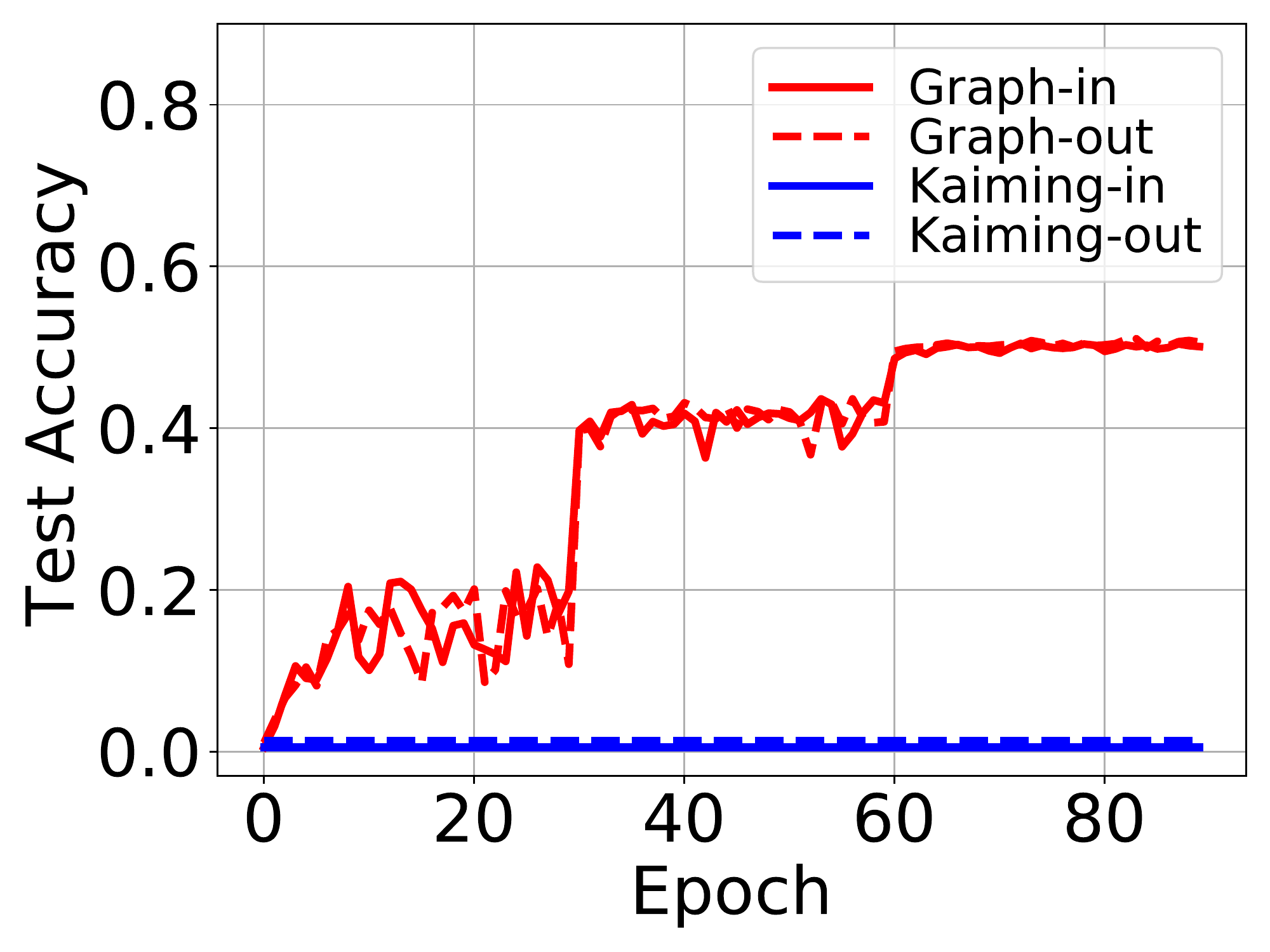}
		\label{fig:tinyoddacc}
	}
	\caption{Results on Tiny-ImageNet.}
	\label{fig:tinyodd}
\end{figure*}

\subsection{Details of Experiment on Tiny-ImageNet}
\label{sec:detail-tiny-imagenet}
In this experiment, we employ tensorial ResNet~\citep{DBLP:journals/corr/abs-1709-02956} to validate the performance of our initialization on Tiny-ImageNet. Details of HTK2 and HOdd based ResNet-50 are shown in Table~\ref{tbl:resnet}. Values of HTK2 and HOdd hyperedges  set to 4. And we also generate some random architectures to demonstrate the ability of our model. We set the batch size to 128 and optimizer as SGD with learning rate 1e-1, momentum 0.9 and weight decay 5e-4. All $\mathbf{r}_{\ast}$ ranks are still 10 except HOdd which set to 5 for faster training. We train each model for 90 epochs with reducing the learning rate through multiplying it by 0.1 after 30 and 60 epochs respectively. We use one NVIDIA Tesla V100 GPU for training. Some results are shown in Figure~\ref{fig:tinyodd}.

% \begin{table}[h]
% 	\centering
% 	\caption{Structure of the Odd-LeNet-5.}
% 	\label{tbl:lenet5}
% 	\begin{tabular}{c|c|c}
% \hline
% Layer   & LeNet-5                       & Odd-LeNet-5                                                   \\ \hline\hline
% conv1    & 5$\times$5$\times$1$\times$20 & 5$\times$5$\times($1$\times$1$)\times$(4$\times$5)           \\ \hline
% maxpool1 & 2$\times$2                    & 2$\times$2                                                   \\ \hline
% conv2    & 5$\times$5$\times$20$\times$5 & 5$\times$5$\times$(4$\times$5)$\times$(5$\times$10)          \\ \hline
% maxpool2 & 2$\times$2                    & 2$\times$2                                                   \\ \hline
% fc1      & 1250$\times$320               & (25$\times$50)$\times$(16$\times$20) \\ \hline
% fc2      & 320$\times$10                 & (16$\times$20)$\times$10                              \\ \hline
% \end{tabular}
% \end{table}

\subsection{Details of Experiment on ImageNet}
\label{sec:detail-imagenet}
Lastly, to be more convincing, we validate our initialization on ImageNet.

\textbf{Tensorial ResNet Setting}~~We construct tensorial ResNet by replacing all convolutional layers of ResNet with tensor layers. In this experiment, we train HOdd (hyperedge set to 4 and $\mathbf{r}_{\ast}$ is 5) and randomly generating ResNet for 20 epochs. Structure of HOdd ResNet is similar to Table~\ref{tbl:resnet}. Depths are set to 50 and 101. Similarly, we use SGD to optimize parameters with learning rate 0.1, momentum 0.9 and weight decay 5e-4. Mini-batch size is set to 512. We use four NVIDIA Tesla V100 GPUs for training.

\textbf{Tensorial gMLP/MLP-Mixer Setting}~~We construct tensorial gMLP/MLP-Mixer by replacing all linear layers of ResNet with tensor layers. In this experiment, we train randomly generating tensor layers for 20 epochs. Training stratey follows setting of \citet{DBLP:journals/corr/abs-2105-08050}. Data augmentation set to AutoAugment. Input resolution of ImageNet is 224$\times$224. Batch-size is 1024. We use Cutmix-Mixup with switch probability 0.5. Cutmix $\alpha$ is 1.0. Mixup $\alpha$ is 0.8. Label smoothing is 0.1. Learning rate is 1e-3 before trainig. Cosine function is adopted as learning rate decay. Optimizer is AdamW with $\epsilon=1e-6$, $\beta_1=0.9$, and $\beta_2=0.999$. Weight decay is 0.05. We use four NVIDIA Tesla A100 GPUs for training.

Results of HOdd-ResNet, HRand-ResNet, HRand-gMLP and HRand-MLP-Mixer are shown in Figure \ref{fig:imagenet-hoddres}, \ref{fig:imagenet-hranres}, \ref{fig:imagenet-hrangmlp} and \ref{fig:imagenet-hrangmlp}, respectively. In these figures, our graphical initialization is suitable for all the models while Kaiming initialization fails in all the situations, which demonstrates Graph(-in/-out) algorithm is sufficiently robust and effective. Worth to mention, test loss of tensorial ResNet explodes to NaN (not a number) from the beginning, and the test loss explosion of tensorial gMLP/MLP-Mixer is not severe like so. The difference is caused by the optimizer. ResNet is trained with SGD, and gMLP/MLP-Mixer is trained with AdamW that has the desirable property of being invariant to the scale of the gradients. However, even through AdamW is so powerful that test loss can return to an acceptable level, it still fails to train a model with the unsuitable initialization, as the test accuracy has not increased. Therefore, our graphical initialization that is adaptive to all the situations, is necessary for TCNNs as a key component of training.

\begin{figure*}[h]
	\centering
	\subfigure[HOdd-RN-50 Loss]{
		\includegraphics[width=0.22\textwidth]{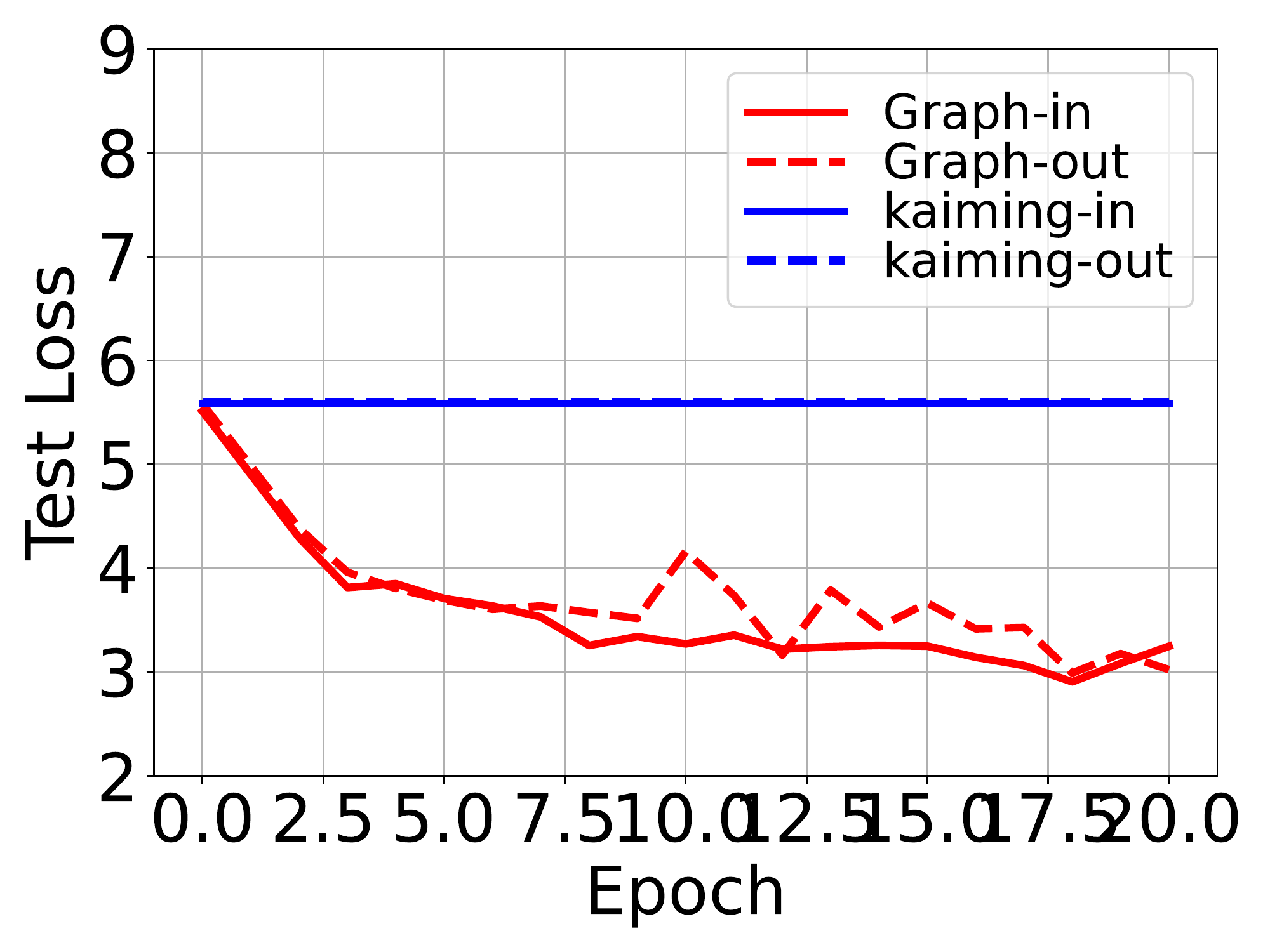}
	} 
	\subfigure[HOdd-RN-50 Accuracy]{
		\includegraphics[width=0.22\textwidth]{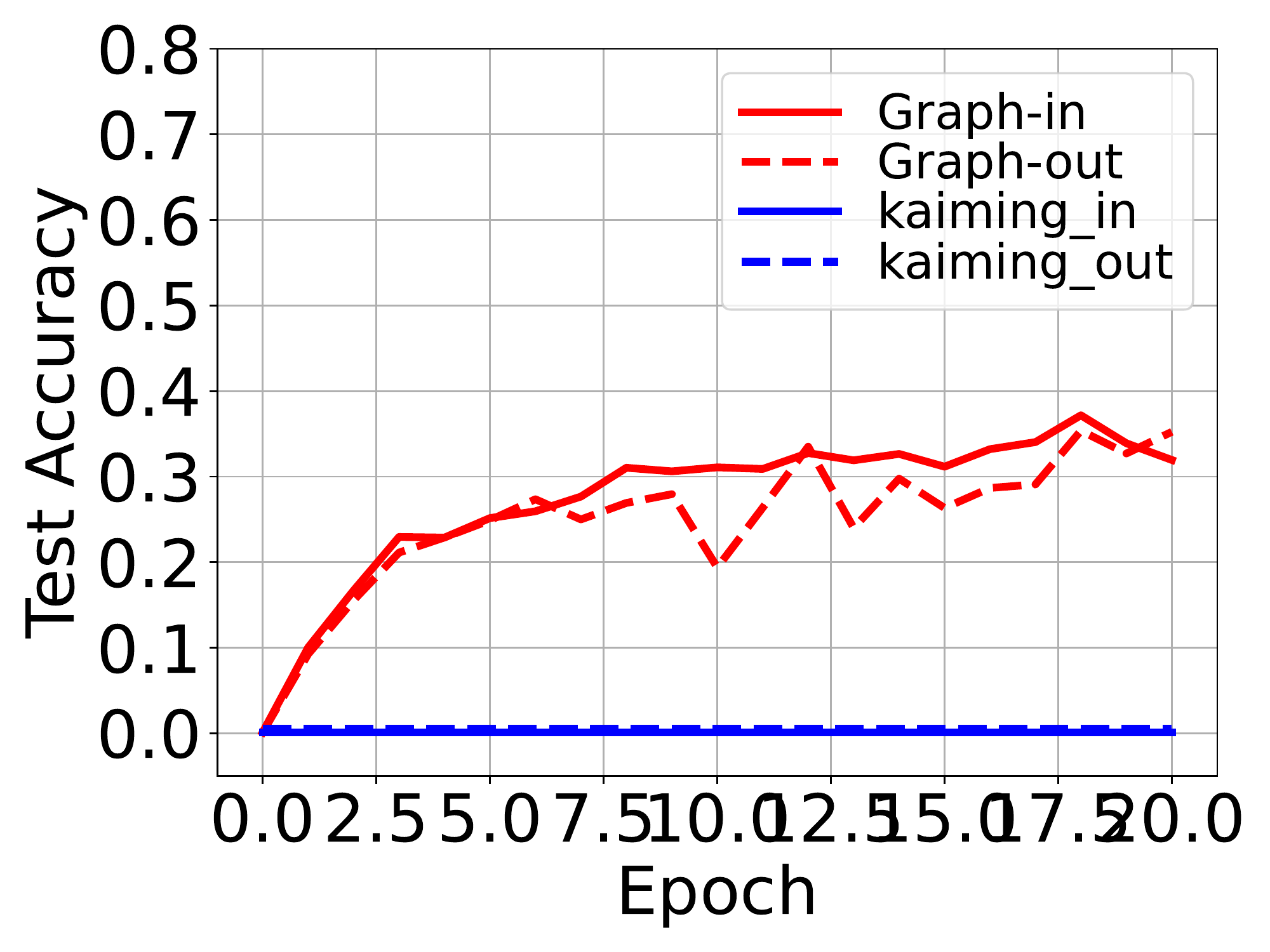}
	}
	\subfigure[HOdd-RN-101 Loss]{
		\includegraphics[width=0.22\textwidth]{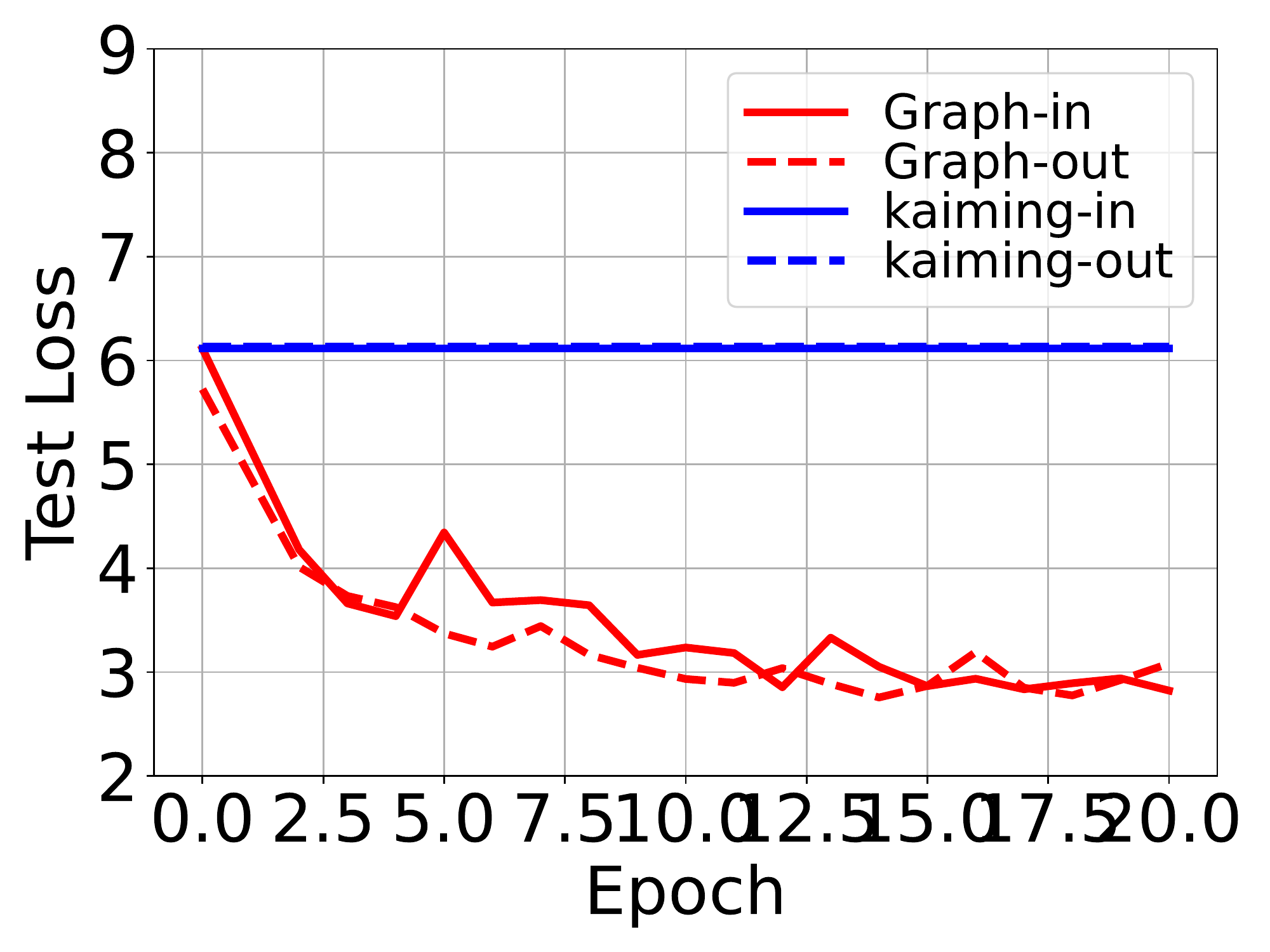}
	} 
	\subfigure[HOdd-RN-101 Accuracy]{
		\includegraphics[width=0.22\textwidth]{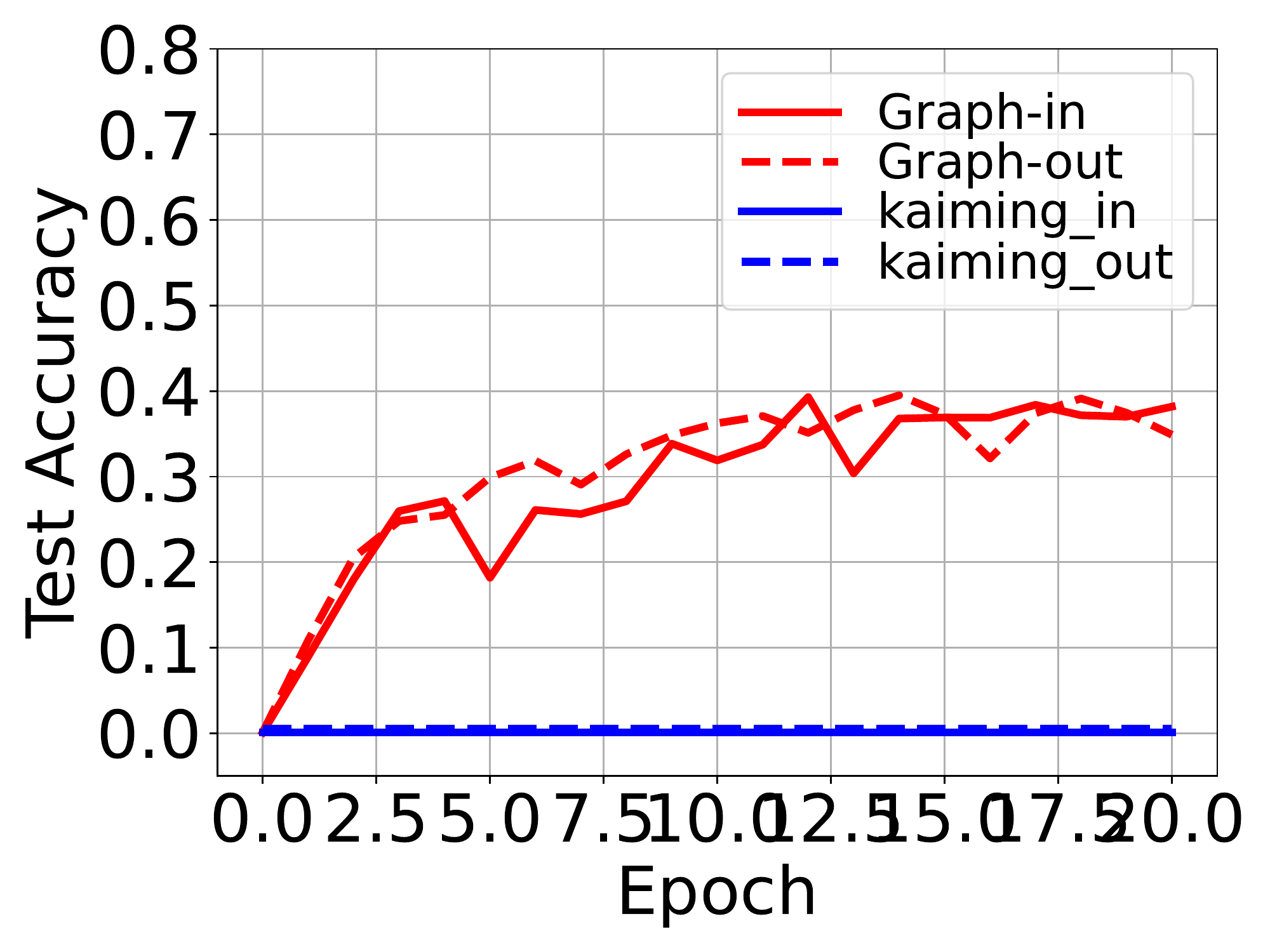}
	}
	\caption{Results of HOdd-RN on ImageNet. RN is short for ResNet.}
	\label{fig:imagenet-hoddres}
\end{figure*}
\begin{figure*}[h]
	\centering
	\subfigure[HRand-RN-50 Loss]{
		\includegraphics[width=0.22\textwidth]{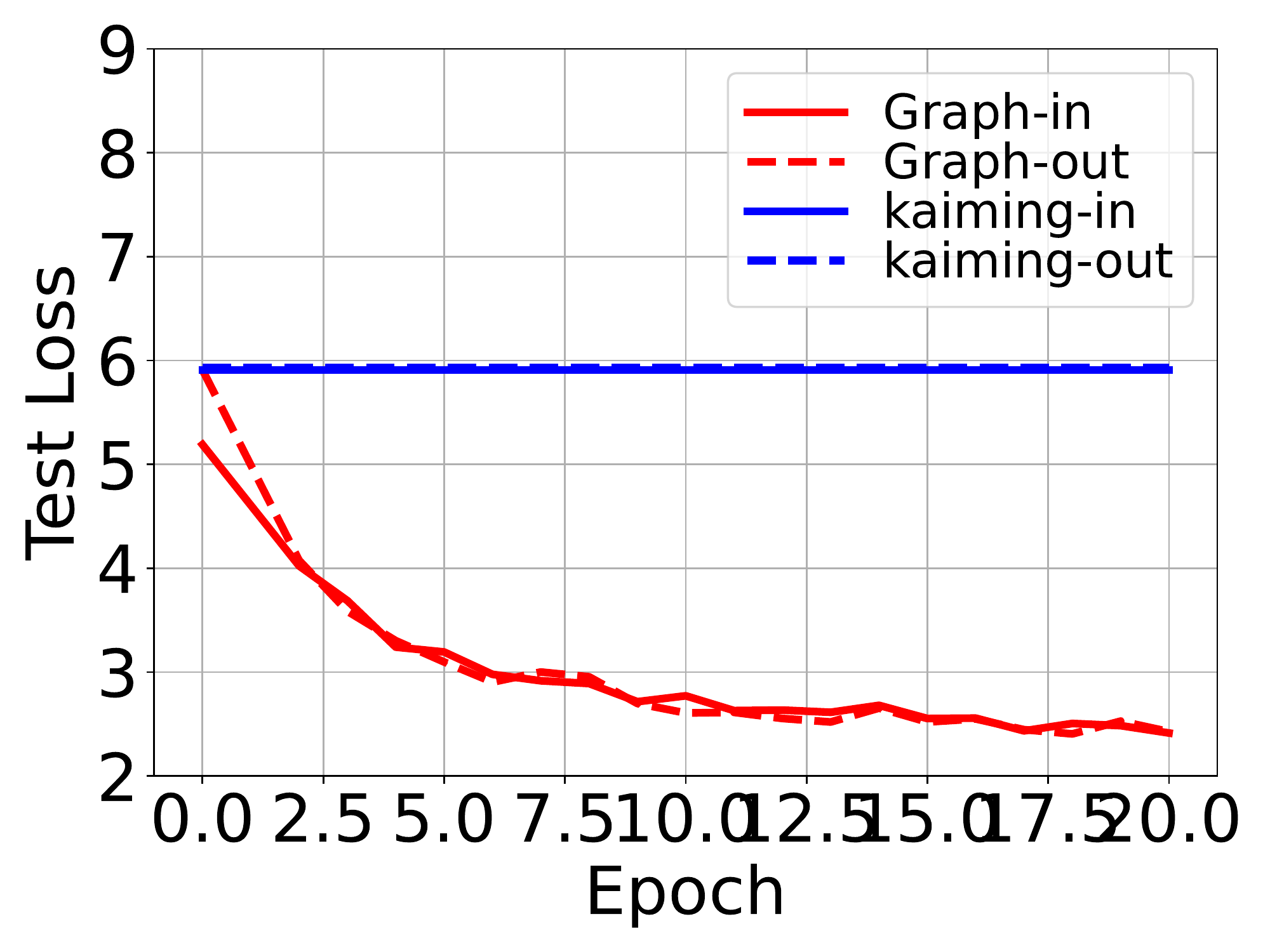}
	} 
	\subfigure[HRand-RN-50 Accuracy]{
		\includegraphics[width=0.22\textwidth]{figures/hrand50_acc.pdf}
	}
	\subfigure[HRand-RN-101 Loss]{
		\includegraphics[width=0.22\textwidth]{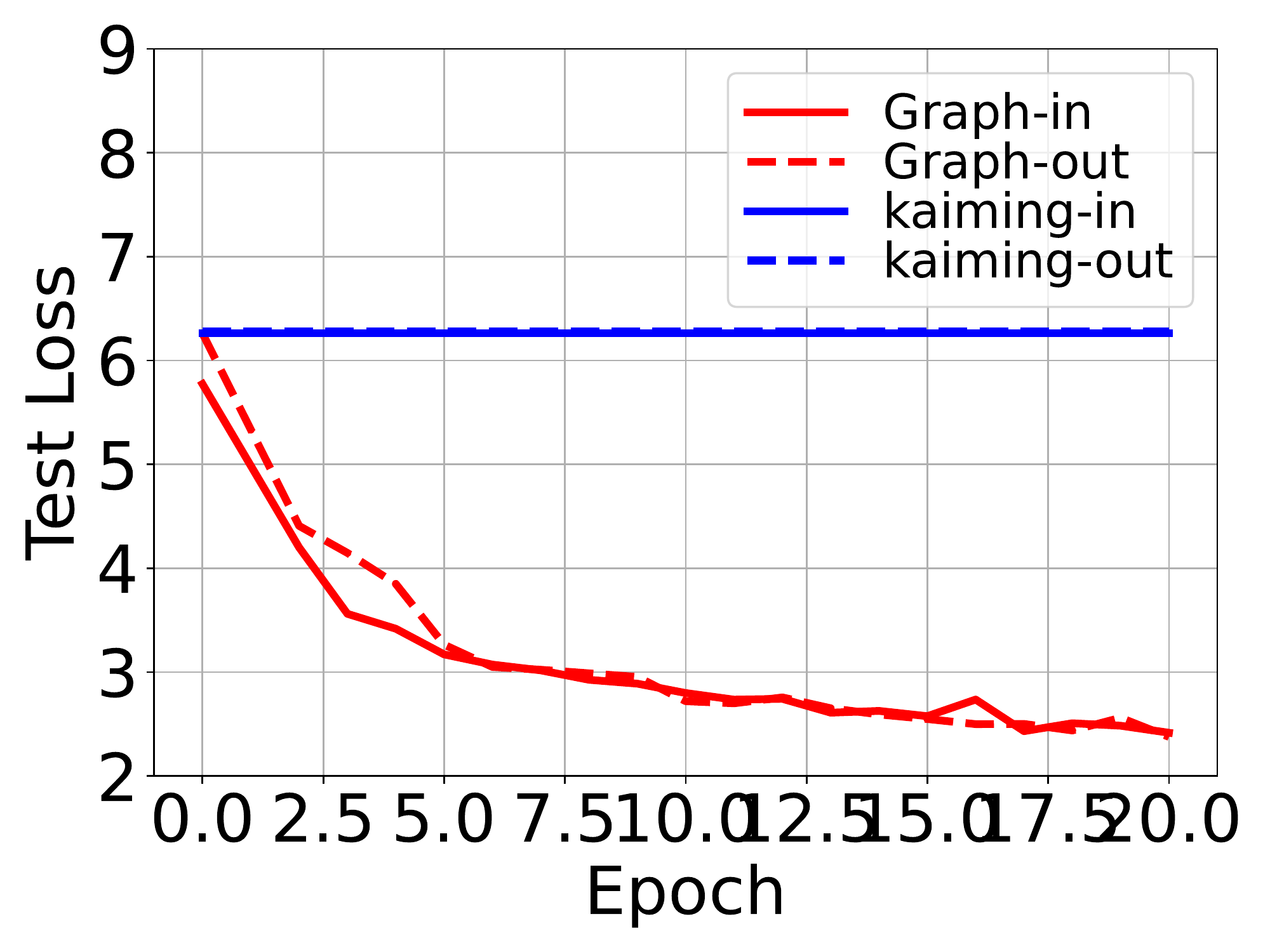}
	} 
	\subfigure[HRand-RN-101 Accuracy]{
		\includegraphics[width=0.22\textwidth]{figures/hrand101_acc.pdf}
	}
	\caption{Results of HRand-RN on ImageNet. RN is short for ResNet.}
	\label{fig:imagenet-hranres}
\end{figure*}
\begin{figure*}[h]
	\centering
	\subfigure[Hg-Ti16 Loss]{
		\includegraphics[width=0.22\textwidth]{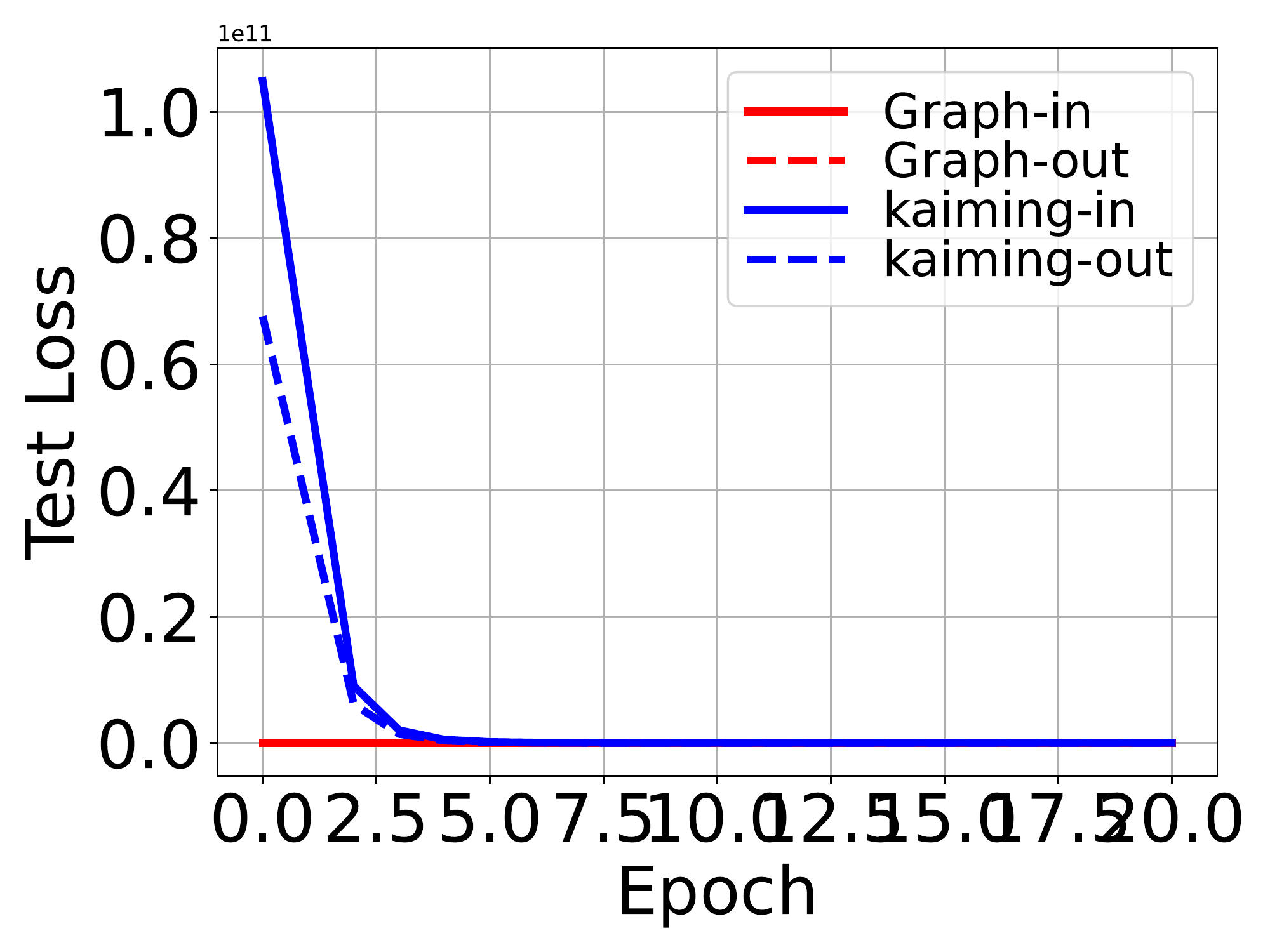}
	} 
	\subfigure[Hg-Ti16 Accuracy]{
		\includegraphics[width=0.22\textwidth]{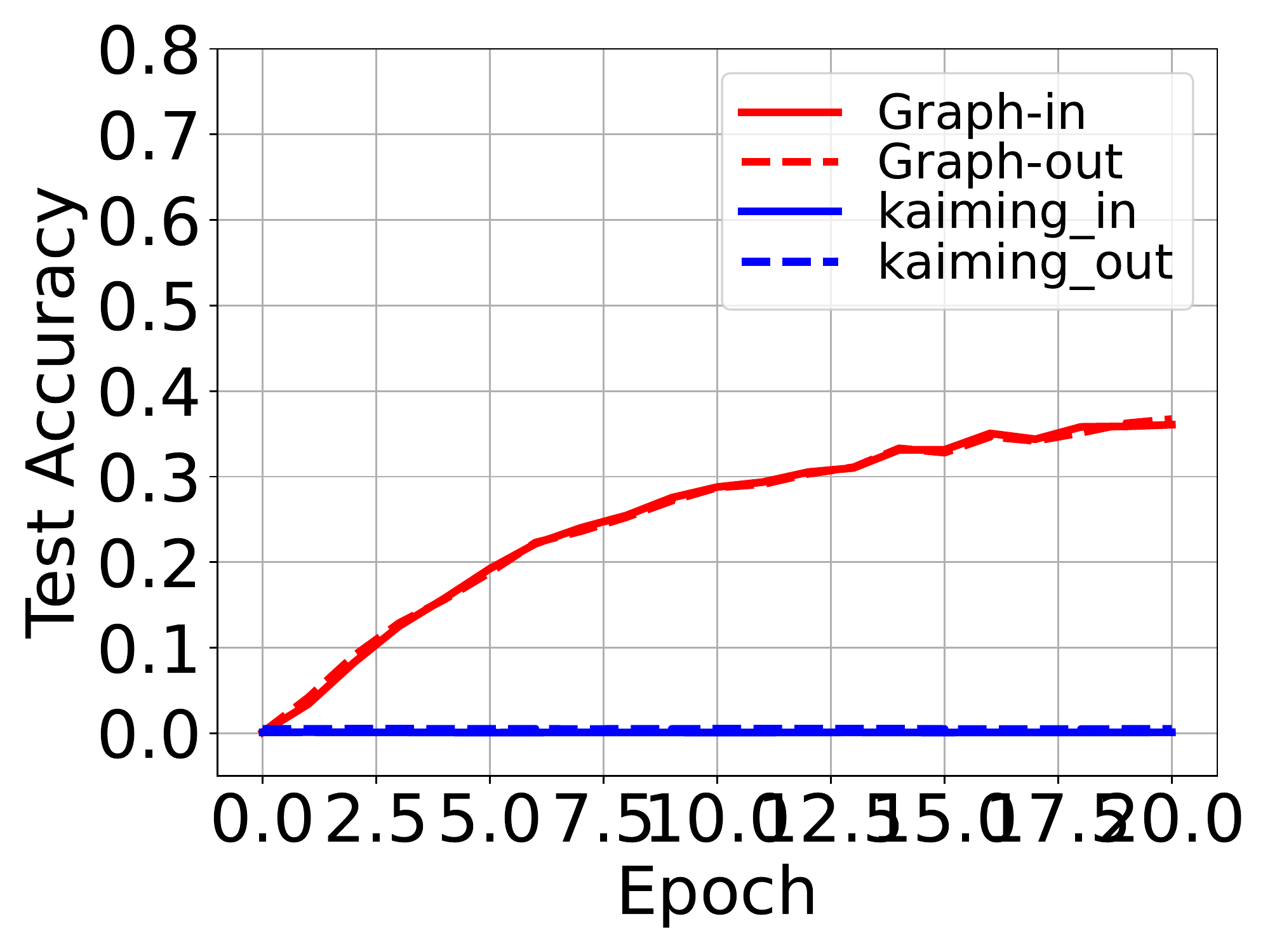}
	}
	\subfigure[Hg-S16 Loss]{
		\includegraphics[width=0.22\textwidth]{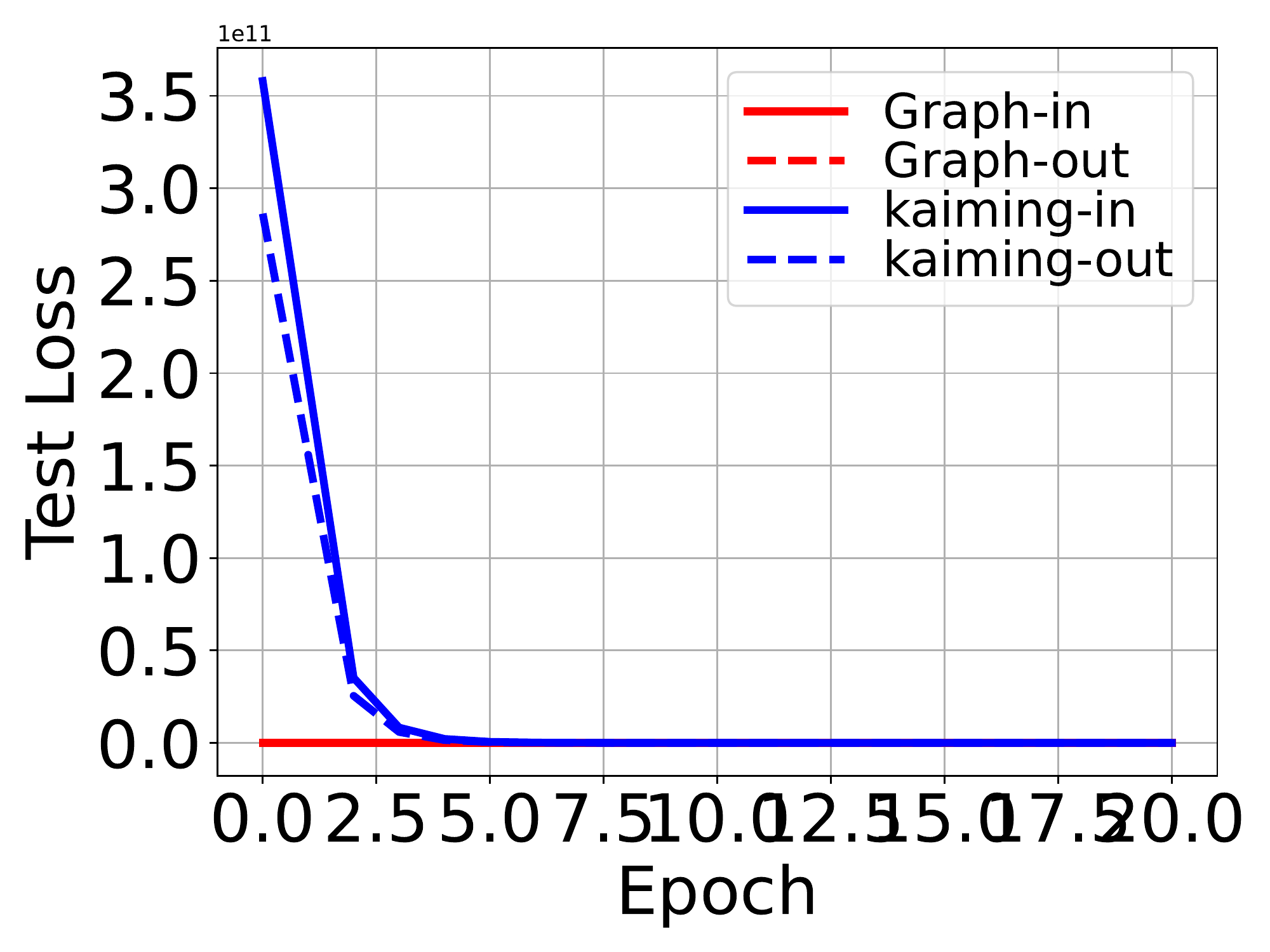}
	} 
	\subfigure[Hg-S16 Accuracy]{
		\includegraphics[width=0.22\textwidth]{figures/hrandgmlp_s16_acc.pdf}
	}
	\caption{Results of HRand-gMLP (Hg) on ImageNet.}
	\label{fig:imagenet-hrangmlp}
\end{figure*}
\begin{figure*}[h]
	\centering
	\subfigure[HMM-S16 Loss]{
		\includegraphics[width=0.22\textwidth]{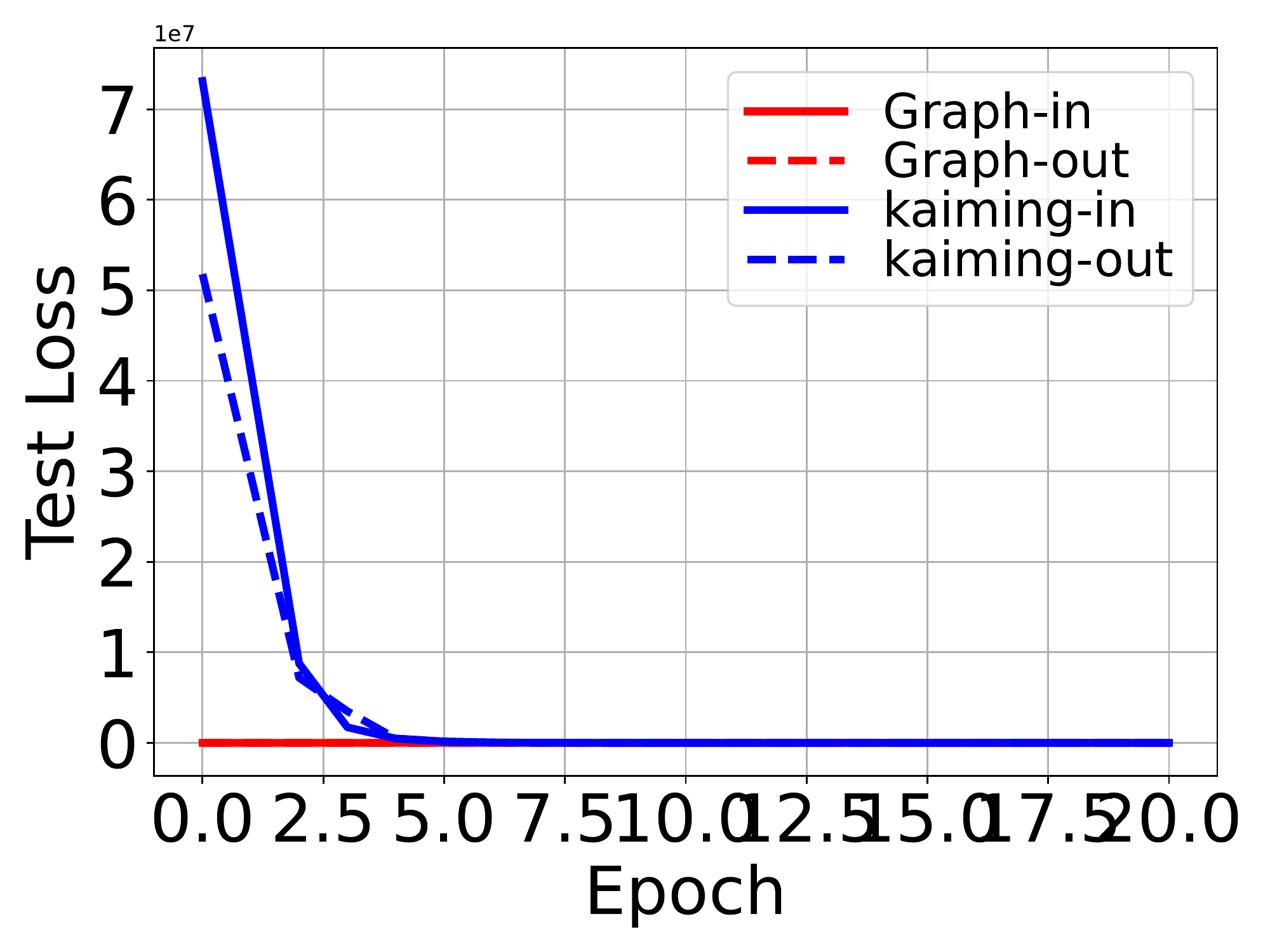}
	} 
	\subfigure[HMM-S16 Accuracy]{
		\includegraphics[width=0.22\textwidth]{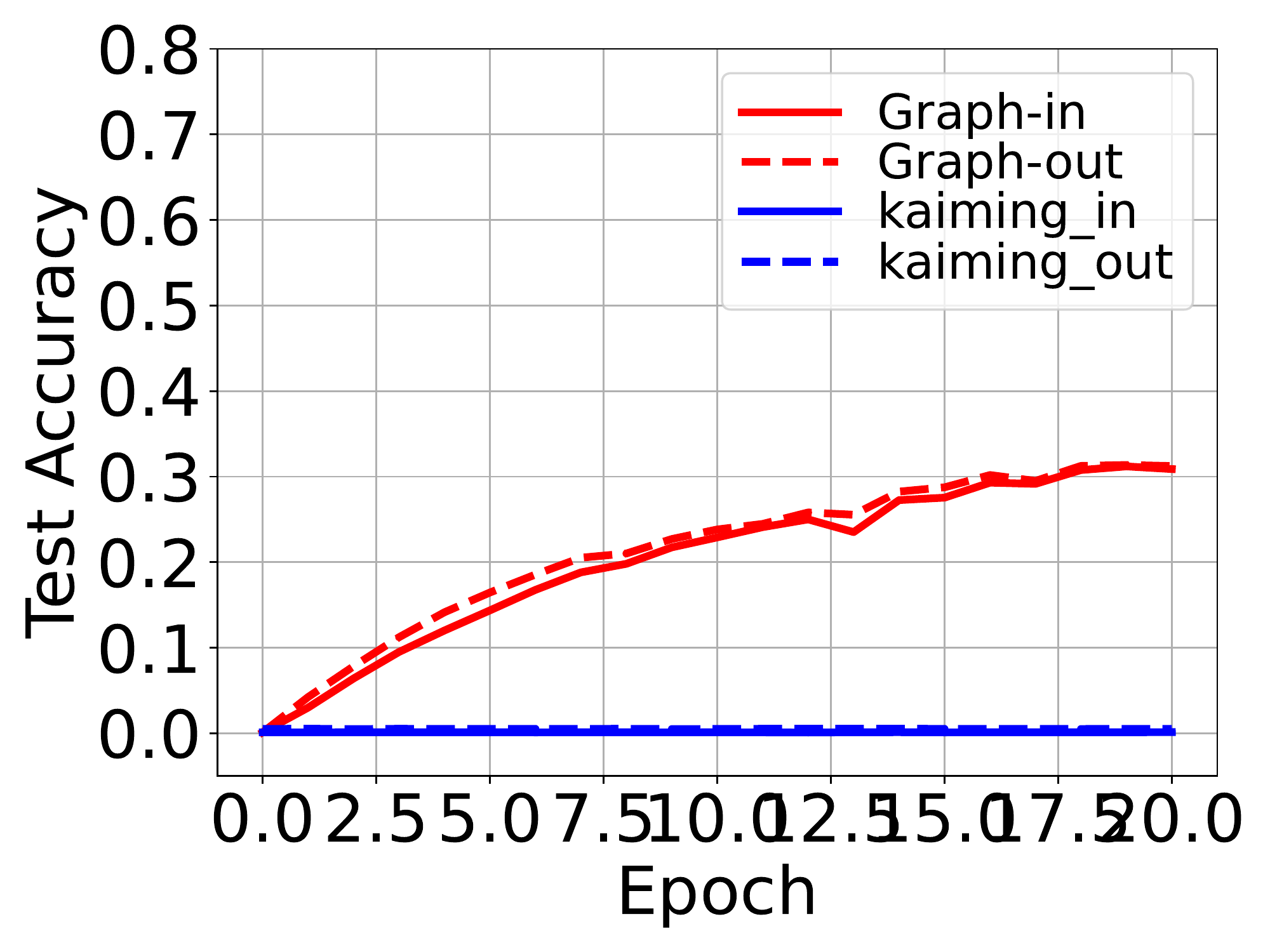}
	}
	\subfigure[HMM-B16 Loss]{
		\includegraphics[width=0.22\textwidth]{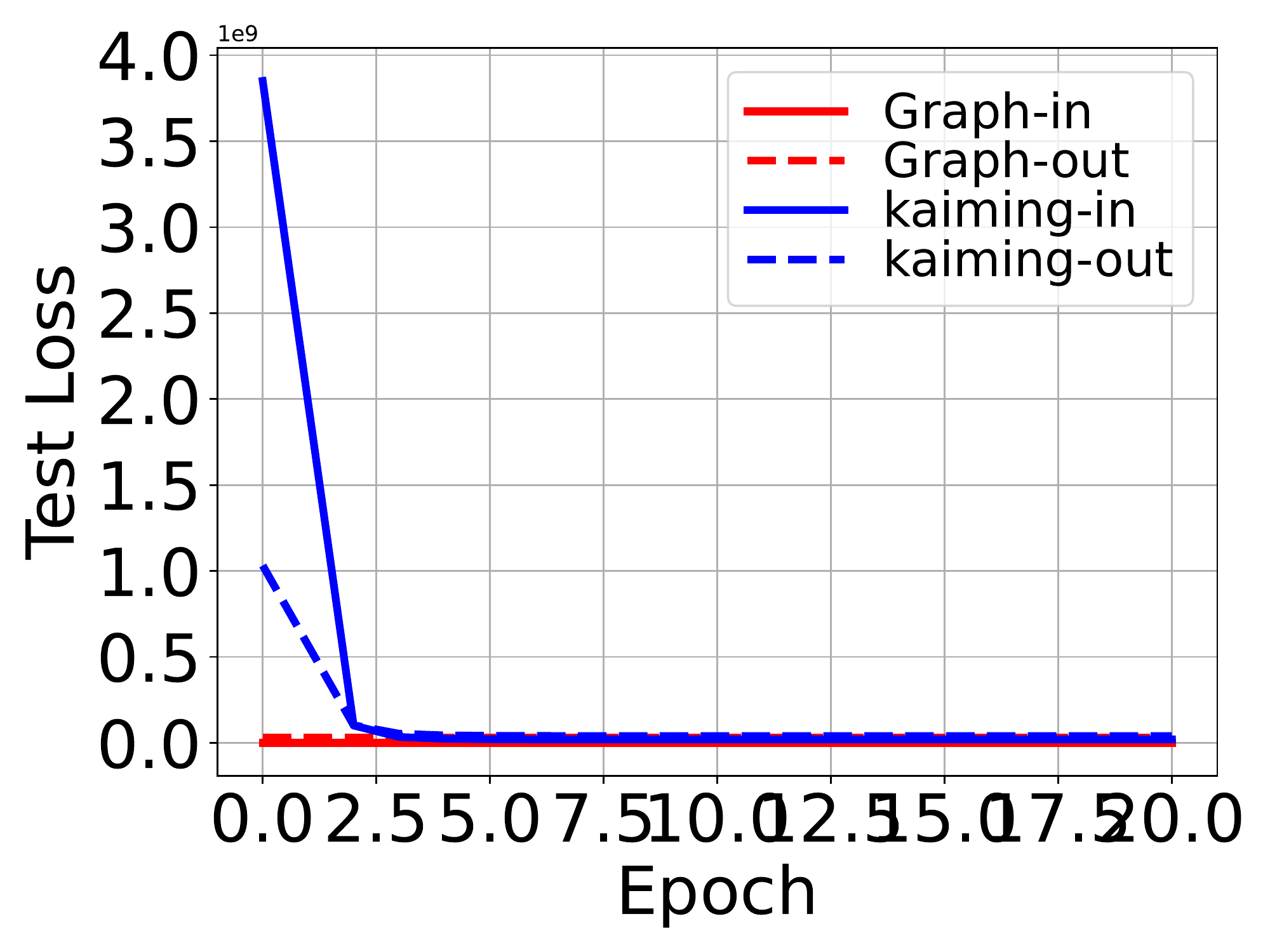}
	}
	\subfigure[HMM-B16 Accuracy]{
		\includegraphics[width=0.22\textwidth]{figures/hrandmlpmixer_b16_acc.pdf}
	}
	\caption{Results of HRand-MLP-Mixer (HMM) on ImageNet.}
	\label{fig:imagenet-hranmixer}
\end{figure*}

% As shown in Figure~\ref{fig:imagenet}, Kaiming initialization fails to train both the HOdd-RN-50 and the HOdd-RN-101 models. However, our initialization can can achieve it, which demonstrate the outstanding capacity of the graphical initialization.

% \section{Limitation and Societal Impacts}
% \label{sec:impact}
% Needing to follow the assumption of Xavier and Kaiming initialization can be regarded as a limitation of our method. Notably, the assumption is easy to be satisfied in practice. Therefore, our initialization can be flexibly applied to TCNNs, as most CNNs choose Xavier and Kaiming as the weight initialization methods.

% It may be impossible to directly cause negative societal impacts by our method, since it mainly focuses on training neural networks.


% You can have as much text here as you want. The main body must be at most $8$ pages long.
% For the final version, one more page can be added.
% If you want, you can use an appendix like this one, even using the one-column format.
%%%%%%%%%%%%%%%%%%%%%%%%%%%%%%%%%%%%%%%%%%%%%%%%%%%%%%%%%%%%%%%%%%%%%%%%%%%%%%%
%%%%%%%%%%%%%%%%%%%%%%%%%%%%%%%%%%%%%%%%%%%%%%%%%%%%%%%%%%%%%%%%%%%%%%%%%%%%%%%

\end{document}